\documentclass[acmsmall,screen,nonacm]{acmart}

\setcopyright{none}
\makeatletter
\@ifundefined{hyxmp@parse@acmart}{}{\let\hyxmp@parse@acmart\relax}
\makeatother

\usepackage{amsmath}
\usepackage{mathtools}
\usepackage{booktabs}
\usepackage{tabularx}
\usepackage{array}
\usepackage{multirow}
\usepackage{enumitem}
\usepackage{graphicx}
\usepackage{xcolor}
\usepackage{tikz}
\usepackage{listings}
\usepackage{url}
\usepackage{microtype}
\usepackage{algorithm}
\usepackage{algpseudocode}
\usepackage{placeins}
\usetikzlibrary{arrows.meta,positioning,fit,calc,shapes.geometric}

\hypersetup{
  hidelinks,
  pdftitle={Fault-Tolerant Budget Conservation in Distributed Multi-Agent Delegation},
  pdfauthor={Genliang Zhu and Chu Wang},
  pdfsubject={Fault-tolerant resource governance and effect settlement for distributed AI-agent delegation}
}

\setlist[itemize]{leftmargin=1.25em}
\setlist[enumerate]{leftmargin=1.45em}
\newcolumntype{L}[1]{>{\raggedright\arraybackslash}p{#1}}
\newcolumntype{Y}{>{\raggedright\arraybackslash}X}
\newenvironment{floatbox}{\begin{minipage}{0.96\linewidth}\centering}{\end{minipage}}

\definecolor{opBlue}{HTML}{3B82F6}
\definecolor{opBlueFill}{HTML}{EFF6FF}
\definecolor{opBlueText}{HTML}{1E40AF}
\definecolor{opGreen}{HTML}{22C55E}
\definecolor{opGreenFill}{HTML}{F0FDF4}
\definecolor{opGreenText}{HTML}{15803D}
\definecolor{opAmber}{HTML}{FCD34D}
\definecolor{opAmberFill}{HTML}{FEF3C7}
\definecolor{opAmberText}{HTML}{92400E}
\definecolor{opRed}{HTML}{EF4444}
\definecolor{opRedFill}{HTML}{FEF2F2}
\definecolor{opRedText}{HTML}{DC2626}
\definecolor{opPurple}{HTML}{8B5CF6}
\definecolor{opPurpleFill}{HTML}{F5F3FF}
\definecolor{opPurpleText}{HTML}{7C3AED}

\newtheorem{definition}{Definition}
\newtheorem{lemma}{Lemma}
\newtheorem{theorem}{Theorem}
\newtheorem{proposition}{Proposition}
\newtheorem{corollary}{Corollary}
\newtheorem{assumption}{Assumption}

\newcommand{\mcode}[1]{\ifmmode\text{\texttt{#1}}\else\texttt{#1}\fi}

\newcommand{\Resources}{\ensuremath{\mathcal K}}
\newcommand{\Credits}{\ensuremath{\mathcal C}}
\newcommand{\Operations}{\ensuremath{\mathcal O}}

\newcommand{\Graph}{\ensuremath{\mathcal G}}
\newcommand{\NatVec}{\ensuremath{\mathbb N^{|\Resources|}}}
\newcommand{\zero}{\ensuremath{\mathbf 0}}
\newcommand{\issued}{\ensuremath{\mathsf{Issued}}}
\newcommand{\available}{\ensuremath{\mathsf{Avail}}}
\newcommand{\inflight}{\ensuremath{\mathsf{Transit}}}
\newcommand{\reserved}{\ensuremath{\mathsf{Reserved}}}
\newcommand{\quarantined}{\ensuremath{\mathsf{Quarantined}}}
\newcommand{\committed}{\ensuremath{\mathsf{Committed}}}
\newcommand{\forfeited}{\ensuremath{\mathsf{Forfeited}}}
\newcommand{\Issue}{\ensuremath{\mathsf{Issue}}}
\newcommand{\Split}{\ensuremath{\mathsf{Split}}}
\newcommand{\Delegate}{\ensuremath{\mathsf{Delegate}}}
\newcommand{\Reserve}{\ensuremath{\mathsf{Reserve}}}
\newcommand{\Dispatch}{\ensuremath{\mathsf{Dispatch}}}
\newcommand{\Settle}{\ensuremath{\mathsf{Settle}}}
\newcommand{\ReturnAct}{\ensuremath{\mathsf{Return}}}
\newcommand{\Retire}{\ensuremath{\mathsf{Retire}}}
\newcommand{\Fence}{\ensuremath{\mathsf{Fence}}}
\newcommand{\JoinAct}{\ensuremath{\mathsf{Join}}}
\newcommand{\RolloverAct}{\ensuremath{\mathsf{TerminalRollover}}}

\newcommand{\Current}{\ensuremath{\mathsf{Current}}}
\newcommand{\Normalize}{\ensuremath{\mathsf{Normalize}}}
\newcommand{\Verify}{\ensuremath{\mathsf{Verify}}}
\newcommand{\Covered}{\ensuremath{\mathsf{Covered}}}
\newcommand{\restrict}{\mathbin{\sqsubseteq}}
\newcommand{\statementheading}[1]{%
  \par\smallskip\noindent\textbf{#1.}\par\nobreak\noindent}

\title[Fault-Tolerant Budget Conservation in Agent Delegation]
{Fault-Tolerant Budget Conservation in Distributed Multi-Agent Delegation}

\author{Genliang Zhu}
\affiliation{%
  \institution{Accentrust}
  \city{Vancouver}
  \country{Canada}
}
\affiliation{%
  \institution{Georgia Institute of Technology}
  \city{Atlanta}
  \state{Georgia}
  \country{USA}
}
\email{research@accentrust.com}

\author{Chu Wang}
\affiliation{%
  \institution{Accentrust}
  \city{Vancouver}
  \country{Canada}
}
\affiliation{%
  \institution{University of Illinois Urbana-Champaign}
  \city{Urbana}
  \state{Illinois}
  \country{USA}
}
\email{wally@accentrust.com}

\begin{abstract}
Resource limits are becoming an authorization boundary for AI agents that
delegate work across concurrent and failure-prone workers. Parent-child
allocation constraints, affine objects, and distributed escrow do not by
themselves prevent overspend when replies are lost, effects complete after
timeout, messages repeat, branches partition, or DAG joins alias one lineage.
We formalize fault-tolerant budget conservation for distributed multi-agent
delegation. Budgets are quantized resource vectors represented by exclusive
escrow credits that move through a delegation DAG. Before dispatch, a branch
converts credit into an operation reservation bound to lineage, epoch,
normalized effect, maximum charge, receiver, and idempotency key. It persists
a signed dispatch permit with quarantine; the gateway verifies that permit
before first acceptance. Uncertain effects remain charged until authenticated
settlement, a fenced authoritative no-effect proof, or permanent retirement.
We prove ownership partition, ledger and effect conservation, descendant
non-amplification, at-most-once settlement, late-completion safety, and
partition confinement under explicit mediation, durability, authentication,
normalization, and gateway assumptions. An indistinguishability result shows
that partition-local availability requires exclusive preallocation. Bounded
TLA+ checking, an independent JavaScript explorer, and crash-injected
two-process SQLite experiments exercise the declared scope and detect
timeout-refund and historical-certificate-validation mutants. The mechanism
preserves the issued budget bound across the evaluated crash, retry, duplicate,
partition, join, and late-completion schedules.
\end{abstract}


\begin{document}
\maketitle

\section{Introduction}
\label{sec:introduction}

AI agents increasingly act as spenders, callers, schedulers, and delegators.
A user may authorize a procurement agent to spend at most \$500, an incident
agent to issue at most 20 production writes, or a research coordinator to make
at most 100 paid API calls. The coordinator can then create workers, workers
can create specialists, and multiple branches can converge at a final
reviewer. A scalar ``remaining budget'' is no longer a local counter. It is a
consumable authorization resource moving through a distributed, asynchronous
delegation graph.

Here \emph{fault-tolerant} means that the safety bound is preserved across
declared crash and communication faults. It does not mean every isolated
branch remains available or every uncertain operation eventually resolves.

Recent work makes important parts of this problem explicit. Agent Contracts
defines multi-dimensional resource constraints and a conservation law
requiring child allocations to remain within the parent contract
\cite{ye2026agentcontracts}. Token Budgets documents 63 production budget
overruns and uses affine ownership to reject in-process aliasing,
double-spending, and use-after-delegation \cite{khan2026tokenbudgets}. AIP
binds identity, attenuated authorization, provenance, and completion records
across the Model Context Protocol (MCP), Agent2Agent (A2A), and HTTP
\cite{prakash2026aip}. Attenuating Authorization
Tokens specify offline holder-derived restriction of tool and argument
authority, depth, and lifetime while leaving transport binding
deployment-specific \cite{niyikiza2026aat}. Together, these systems establish
the allocation, affine-ownership, identity, and attenuation foundations on
which this paper builds. Our contribution is the coupling of exclusive vector
credit lineage with accepted and uncertain external-effect coverage across
crashes, partitions, retries, and multi-parent joins.

\paragraph{The settlement gap.}
Allocation safety is not execution safety. Consider a root grant of two paid
effects. The coordinator allocates one unit to each of workers \(A\) and \(B\),
so the allocation law is satisfied. Worker \(B\) first spends its own unit.
Worker \(A\) then reserves its unit and sends a payment request. The response
is lost. Its controller times out, declares the worker dead, and refunds the
unit to the parent. Worker \(B\) spends that refunded unit, after which the
original \(A\) payment completes. Every allocation
calculation was locally plausible; three effects have nevertheless consumed a
two-effect authorization.

Figure~\ref{fig:motivation} separates the unsafe trace from the required
response. The missing object in this allocation-only trace is an explicit
\emph{effect obligation}: once an external
effect may have crossed the durability boundary, its budget remains reserved
through uncertainty. Recovery may query, settle, or permanently retire the
obligation; elapsed time alone cannot make the right spendable again.

\begin{figure*}[t]
\centering
\begin{floatbox}
\resizebox{0.98\linewidth}{!}{%
\begin{tikzpicture}[
  >=Stealth,
  box/.style={draw, rounded corners=2.8mm, thick, align=center,
    minimum width=2.62cm, minimum height=0.84cm, inner sep=4.5pt, font=\small},
  edge/.style={->, thick},
  bad/.style={->, thick, draw=opRed},
  good/.style={->, thick, draw=opGreen}
]
\node[font=\bfseries, text=opRedText] (badtitle) at (0,2.1)
  {Allocation-safe appearance, unsafe settlement};
\node[box, draw=opBlue, fill=opBlueFill, text=opBlueText] (broot) at (0,0.8)
  {\textbf{Root}\\budget \(B=2\)};
\node[box, draw=opAmber, fill=opAmberFill, text=opAmberText, right=8mm of broot] (ba)
  {\textbf{Worker B}\\spend own 1};
\node[box, draw=opRed, fill=opRedFill, text=opRedText, right=8mm of ba] (btimeout)
  {\textbf{Worker A}\\reserve 1; dispatch};
\node[box, draw=opAmber, fill=opAmberFill, text=opAmberText, right=8mm of btimeout] (bb)
  {\textbf{Timeout}\\refund 1; B spends};
\node[box, draw=opRed, fill=opRedFill, text=opRedText, right=8mm of bb] (blate)
  {\textbf{Late A effect}\\total \(=3>2\)};
\draw[edge] (broot) -- (ba);
\draw[edge] (ba) -- (btimeout);
\draw[bad] (btimeout) -- node[above, font=\scriptsize, inner sep=1pt]
  {lost reply} (bb);
\draw[bad] (bb) -- (blate);

\node[font=\bfseries, text=opGreenText] (goodtitle) at (0,-1.15)
  {Fault-tolerant response: uncertainty retains the right};
\node[box, draw=opBlue, fill=opBlueFill, text=opBlueText] (groot) at (0,-2.45)
  {\textbf{Root}\\budget \(B=2\)};
\node[box, draw=opAmber, fill=opAmberFill, text=opAmberText, right=8mm of groot] (ga)
  {\textbf{Worker A}\\reserve 1; dispatch};
\node[box, draw=opAmber, fill=opAmberFill, text=opAmberText, right=8mm of ga] (gind)
  {\textbf{Indeterminate}\\keep reserved};
\node[box, draw=opPurple, fill=opPurpleFill, text=opPurpleText, right=8mm of gind] (grec)
  {\textbf{Reconcile}\\receipt / no-effect proof};
\node[box, draw=opGreen, fill=opGreenFill, text=opGreenText, right=8mm of grec] (gdone)
  {\textbf{One transition}\\commit or return};
\draw[edge] (groot) -- (ga);
\draw[edge] (ga) -- node[above, font=\scriptsize, inner sep=1pt]
  {lost reply} (gind);
\draw[good] (gind) -- (grec);
\draw[good] (grec) -- (gdone);
\end{tikzpicture}%
}
\end{floatbox}
\caption{Parent/child allocation conservation does not prevent a
timeout-refund/late-completion double spend. The safe trace keeps the
reservation charged while the external outcome is indeterminate. Color is
redundant with labels and arrow meaning.}
\Description{Two horizontal traces compare unsafe and safe handling of a lost
effect response. The unsafe red trace refunds on timeout, lets another worker
spend, and later observes the first effect, exceeding the root budget. The safe
green and amber trace keeps the right reserved and reconciles an authenticated
receipt or no-effect proof before committing or returning it.}
\label{fig:motivation}
\end{figure*}
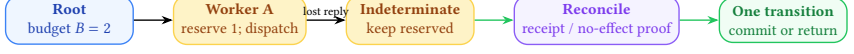

\paragraph{Why a delegation DAG is harder than a tree.}
Agent delegation is not necessarily hierarchical. Two branches may ask a
shared specialist to synthesize results; a reviewer may join multiple
principal-approved inputs; retries may reconstruct a worker under a new
process identity. A child with two incoming references does not thereby own
the sum of both ancestors. If those references alias one credit generation,
adding them creates authority. The budget-ownership relation must therefore
preserve disjoint lineage even when the task graph has joins.

\paragraph{Dual conservation.}
The paper distinguishes two properties:

\begin{align}
\textit{ledger conservation:}\quad&
\available+\inflight+\reserved+\quarantined \notag\\[-1mm]
&\qquad+\committed+\forfeited=\issued,
\label{eq:intro-ledger}\\
\textit{effect conservation:}\quad&
  \forall \iota\in\mathsf{Obs},\ \exists!\,o:\Covered(\iota,o),
  \quad \mathsf{Exposure}_g\preceq\mathbf B_g .
\label{eq:intro-effect}
\end{align}

Equation~\eqref{eq:intro-ledger} prevents rights from being created and keeps
ambiguous dispatched work visible in the quarantined bucket.
Equation~\eqref{eq:intro-effect} prevents a real effect instance from becoming
unaccounted while messages, controllers, and receipts disagree, and bounds
settled actual charge plus the maximum charge of uncertain operations. The
unique operation \(o\) is quarantined, committed, or conservatively forfeited;
injectivity prevents one reservation from covering two distinct instances.
Neither equation implies the other without the binding and maximum-charge
assumptions stated in Section~\ref{sec:threat}.

\paragraph{Approach.}
We model each root grant as a vector of quantized resources. The issuer creates
exclusive escrow credits. Delegation moves a subvector into a fresh child
generation and atomically removes it from the parent generation. A branch can
spend locally during a partition only from its exclusive escrow. Before
dispatch, the branch durably creates a reservation bound to the root, credit
lineage, current epoch, normalized effect, maximum charge, receiver, and stable
operation key. Dispatch changes the obligation state but does not refund.
Authenticated, effect-bound success receipts are applied at most once;
verified no-effect certificates return the reservation; irrecoverable
uncertainty can be permanently forfeited but not silently reused.

\paragraph{Contributions.}
This paper makes five contributions:

\begin{enumerate}
  \item \textbf{A fault-aware authorization semantics.} We define resource
        vectors, exclusive credits, a delegation DAG, operation obligations,
        authenticated terminal evidence, and a transition system that separates
        allocation, execution, settlement, and effect coverage.
  \item \textbf{A fault-tolerant protocol.} We give algorithms for idempotent
        split/delegation, local reservation, dispatch, terminal-evidence settlement,
        multi-parent join, recovery, epoch fencing, and safe reclaim, together
        with refinement obligations for MCP/A2A/HTTP adapters.
  \item \textbf{Conditional safety results.} We prove ownership partition,
        global conservation, descendant non-amplification, at-most-once
        settlement, effect-instance coverage, an external exposure bound,
        late-completion safety, and partition confinement. We state an explicit
        trace-lifting criterion that a concrete runtime must discharge before
        inheriting those abstract results.
  \item \textbf{A coordination boundary.} An indistinguishability argument
        shows why a shared, unallocated unit cannot remain spendable in every
        partition without risking amplification. Exclusive preallocation
        buys local availability; reallocation and reclaim may require
        coordination.
  \item \textbf{An evidence-disciplined reference profile.} We specify a
        provider-free artifact that connects finite model checking,
        transition-level invariant checks, mutation, fault injection, and descriptive
        local timing to explicit claims. Hosted-model quality and production
        deployment remain outside the evidence.
\end{enumerate}

\paragraph{Non-goals.}
We do not invent escrow, bounded counters, leases, Sagas, idempotency,
transactional tool use, commit-time authorization, or delegation tokens. We do
not guarantee exactly-once network delivery, roll back irreversible effects,
infer prices, optimize allocation, tolerate Byzantine infrastructure, or
remain globally available through arbitrary partitions.

\section{Background and Closest Work}
\label{sec:background}

\subsection{Resource-Bounded Agent Execution}

Agent Contracts unifies input/output specifications, skill sets,
multi-dimensional resources, temporal bounds, success criteria, and
termination in one lifecycle \cite{ye2026agentcontracts}. For delegation it
requires \(\sum_i R_i\le R_{\mathrm{parent}}\), supports recursive
subcontracts and returned capacity, and reports zero conservation violations
in a 50-topic three-agent pipeline. Its hard monitor checks between actions,
and the paper explicitly notes that a single call can exceed a budget before
usage is known. Our work treats Agent Contracts as the closest
allocation-level neighbor. Our question is what happens after allocation
when messages duplicate, branches partition, a process crashes, or an external
effect is uncertain.

Token Budgets establishes that an ordinary copyable number is the wrong
abstraction \cite{khan2026tokenbudgets}. Its affine Rust object prevents a
cost-bearing value from being cloned, double-spent, or used after delegation
within the type-system envelope. The paper also quantifies static
over-reservation, sketches distributed reservation using mechanisms such as
Sagas, consensus, and leases, and identifies distributed partial failure as a
substantial open extension. Our protocol adopts the same move-not-copy
intuition but makes that extension explicit across processes, languages,
protocols, crash recovery, and DAG joins.

CapLease is the closest durable-authorization neighbor: it records consumable
authorization state across retries, concurrent use, delegation, and crash
recovery, with an idempotent receiver contract \cite{xu2026durableauthorization}.
It therefore establishes durable agent authorization as direct prior art. The
distinct contribution here is partition-aware transfer of exclusive
multi-dimensional credit through a delegation DAG, including atomic
multi-parent joins, together with one conservation argument spanning both
credit lineage and accepted or uncertain external effects.

Among the systems compared here, EpochX is the closest credits-native
infrastructure: it locks task bounties,
delegates budgets recursively, and settles rewards on acceptance
\cite{wang2026epochx}. Its focus is a production marketplace and reusable
asset economy. Our focus is the prefix invariant and effect-coverage semantics
when transfer, settlement, retry, and recovery traces are duplicated,
reordered, lost, partitioned, or crash-interrupted.

\subsection{Delegation Identity and Scope Attenuation}

AIP supplies verifiable identity, attenuated authorization, provenance, and
completion records across MCP, A2A, and HTTP \cite{prakash2026aip}. It carries
a \mcode{budget\_usd} ceiling that attenuates along a token chain, but treats
that value as a per-token authorization ceiling rather than a running balance;
the runtime remains responsible for aggregate spend enforcement. AATs let a
holder derive a signed chain with equal-or-narrower tools, arguments, depth,
and lifetime without contacting the root issuer, while leaving transport
binding deployment-specific \cite{niyikiza2026aat}. Overlaying Governance
formalizes recursive delegation
and scope composition over existing policy \cite{ibrahim2026overlay}. These
mechanisms prevent \emph{scope amplification}. A consumable right adds a
different obligation: two children can each hold the same narrow scope while
only one may consume the remaining unit. Offline attenuation alone cannot
decide exclusive ownership.

PAuth derives service-side slices for task-implied concrete operations and
binds operand values to provenance-aware authorization envelopes
\cite{sharma2026pauth}. It strengthens operation-specific authorization, but
does not define exclusive ownership of consumable budgets, crash/partition
credit transfer, or settlement of an uncertain external effect.

Stateful least privilege makes authorization decisions depend on execution
history and represents policy state as replicated state machines
\cite{cao2024stateful}. This establishes stateful cloud authorization as a
deployed security direction. Our state is narrower but structurally
different: it tracks exclusive consumable-credit generations and their
effect obligations across an agent delegation DAG.

MCP standardizes model-to-tool connectivity and profiles OAuth authorization;
A2A standardizes agent discovery and task exchange
\cite{mcp2026authorization,a2a2026spec}. MPAC adds multi-principal sessions,
intent, operations, conflicts, Lamport-clock watermarks, optimistic
concurrency, and governance \cite{qian2026mpac}. We use protocol adapters and
causal metadata but do not claim a general coordination protocol.
The active individual Internet-Draft on Transaction Tokens for Agents
propagates subject, actor, and agentic context through multi-agent call chains
\cite{raut2026transactiontokens}; it specifies neither a running consumable
balance nor exclusive escrow ownership or effect coverage.

\subsection{Transactional Effects and Commit-Time Authority}

Atomix gives agent tool calls progress-aware transactional semantics. It uses
epochs, per-resource frontiers, delayed commit for bufferable effects, and
compensation for externalized effects \cite{mohammadi2026atomix}. Commit-time
authorization separately requires fresh, causally prior, effect-bound,
eligible authority at the durability boundary
\cite{santos2026temporary}. Context-to-Execution Integrity binds protected
fields, exact effects, and invocation authority \cite{santos2026cxi}. These
works establish that endpoint success is not enough and that external effects
need a boundary protocol. We add conserved consumable rights across
distributed delegation and keep uncertain effects charged across recovery.

Compensation is useful but weaker than prevention for irreversible payments,
messages, disclosures, or physical actions. Sagas explicitly accept
application-specific compensation rather than atomic rollback
\cite{garciamolina1987sagas}. Likewise, ``exactly once'' can conflate delivery,
state consistency, and deterministic replay \cite{trofimov2019delivery}. Our
claims are at-most-once \emph{settlement} and continuous \emph{coverage} at a
mediated gateway, not exactly-once delivery.

Two recent results sharpen this boundary. Machine-checked dual-write recovery
shows that source-side state alone cannot safely distinguish a missing sink
write from a completed one and formalizes recovery around sink acceptance
records and fencing \cite{andreakis2026dualwrite}. A cross-runtime study of
agent tool calls similarly finds that exactly-once behavior depends on the
model, harness, and tool contract, and that verification alone does not close
late-commit races \cite{li2026exactlyonceagents}. Our protocol adopts that
receiver-visible boundary and adds the invariant that every accepted or still
uncertain effect remains covered by exclusive credit.

\subsection{Escrow, Numeric Invariants, and Coordination}

The Escrow Transactional Method permits concurrent changes while reserving
enough rights to preserve numeric bounds \cite{oneil1986escrow}. Bounded
Counter conflict-free replicated data types (CRDTs) build on convergent
replicated-data-type semantics and distribute
decrement rights so disconnected replicas can act without violating a lower
bound \cite{shapiro2011crdt,balegas2015bounded}. Just-Right
Consistency combines CRDTs, transactional causal consistency, bounded
counters, and static analysis to synchronize only where application
correctness requires it \cite{shapiro2018justright}. These are foundational
prior art, not our inventions.

Invariant confluence characterizes when valid divergent states can merge
without coordination \cite{bailis2014coordination}. The recent Complete CALM
criterion states a broader history-based boundary between monotone
specifications and those requiring coordination
\cite{hellerstein2026coordination}. A consumable shared unit is non-monotone:
spending in one partition removes an option from another. Our protocol applies
the classical answer--exclusive preallocation--and then adds agent-specific
delegation lineage, authority bindings, normalized effects, and settlement
obligations.

\subsection{Leases, Fencing, and Durable Workflow}

Leases provide time-bounded authority under clock and failure assumptions
\cite{gray1989leases}; lock services use monotonically ordered ownership and
durable consensus to coordinate clients \cite{burrows2006chubby,
ongaro2014raft}. Expiry does not undo an already dispatched effect. We use an
epoch as a fencing generation, never as evidence that the previous effect did
not happen.

Recent capability-coherence work maps authorization revocation to memory
consistency and reports an execution-count bound on post-revocation operations
\cite{parakhin2026bureaucracy}. That revocation bound is complementary: our
epoch fence controls acceptance across ownership generations, while reclaim
also requires exclusive credit ownership and terminal effect evidence.

Durable distributed applications commonly decompose work into idempotent
steps, stable request identifiers, and reconciliation rather than assuming a
cross-organization transaction \cite{helland2017life}. Stable call identifiers
and duplicate suppression trace back to classic RPC semantics
\cite{birrell1984rpc}; durable systems additionally need the completion record
and protected mutation to share the relevant atomicity boundary.
Linearizability remains
the abstraction for an authoritative local transition
\cite{herlihy1990linearizability}; Lamport order supplies causal structure
without pretending synchronized clocks determine truth
\cite{lamport1978time}. Table~\ref{tab:closest} positions the resulting gap.

\begin{table*}[t]
\centering
\scriptsize
\setlength{\tabcolsep}{3pt}
\begin{tabularx}{\textwidth}{L{0.145\textwidth}YYYYYY}
\toprule
\textbf{Work} & \textbf{Consumable vector} & \textbf{Recursive / DAG delegation} &
\textbf{Crash + partition} & \textbf{External-effect settlement} &
\textbf{Primary contribution} & \textbf{Boundary relative to this paper} \\
\midrule
Agent Contracts & Yes & Recursive hierarchy & No fault semantics evaluated &
Between-action monitor & Resource-bounded contract + allocation law &
No duplicate/late-effect/fencing model \\
Token Budgets & Cost scalar/object & Fan-out ownership & In-process focus &
Estimator-based cap & Incident catalog + affine Rust case study &
No durable cross-protocol recovery \\
CapLease & Consumable authorization state & Delegation & Crash/retry recovery &
Idempotent receiver contract & Durable authorization-state protocol &
No partitioned vector-credit DAG join/effect-coverage theorem \\
EpochX & Task credits / bounties & Recursive delegation & No reported
crash/partition proof & Lock + acceptance settlement & Agent marketplace and
asset economy & No prefix ownership/effect-uncertainty invariant \\
AIP / AAT & Reusable scopes & Signed chains & Offline verification &
Invocation + completion binding & Identity and monotone scope attenuation &
No exclusive consumable ownership \\
MPAC & Not central & Multi-principal graph & Optimistic conflicts &
Operation state & General agent coordination & No resource invariant \\
Atomix & Resource frontiers & Workflow branches & Fault injection &
Commit + compensation & Transactional tool semantics &
No delegated budget conservation \\
RIFL / Beldi / ExoFlow & No budget vector & remote procedure call
(RPC)/workflow graph &
Crash recovery & Durable completion / exactly-once workflow &
Replay-safe execution substrates & No delegated-right or charge-bound theorem \\
Escrow / Bounded Counter & Numeric rights & Replica rights graph &
Partition-aware & Database update & Numeric invariant preservation &
No agent authority/effect witness \\
Commit-time authorization & Not central & Multi-agent workflow &
Controlled invalidation & Fresh effect-bound witness &
No resource transfer/settlement law \\
\textbf{This paper} & \textbf{Quantized vector} &
\textbf{Disjoint lineage in a task DAG} & \textbf{Crash, retry, partition,
late completion} & \textbf{Reservation + authenticated terminal evidence} &
\textbf{Dual ledger/effect conservation} &
\textbf{Conditional on mediated gateway and terminal-evidence soundness} \\
\bottomrule
\end{tabularx}
\caption{Closest work is complementary. The contribution is not any
individual primitive, but a fault-aware composition that keeps exclusive
delegated rights and external-effect obligations coupled.}
\label{tab:closest}
\end{table*}

\section{System, Failure, and Threat Model}
\label{sec:threat}

\subsection{Entities and Trust Boundaries}

The system contains a root issuer, delegation controllers, agent runtimes,
durable branch stores, effect gateways, a terminal-evidence authority, a reconciler, and
an auditor. The issuer authenticates a task grant and creates the initial
budget vector. Controllers split and move rights. Agent models propose work
but cannot mint, settle, or refund credit. Each branch store atomically records
local ownership, reservations, outbox state, and terminal evidence. The effect gateway
is the complete-mediation point for a protected effect. It enforces the
canonical operation key and epoch, and returns or exposes an authoritative
outcome.

In the reference profile, the issuer, controller, branch store, gateway,
terminal-evidence authority, and reconciler form the declared enforcement
trusted computing base (TCB); the controller and reconciler are trusted
orchestration but still cannot
bypass store/gateway validation. A deployment may remove those two
orchestrators from the TCB only if every request is independently revalidated
at the authoritative store or gateway. The gateway's enforcement verifier is
therefore distinct from the read-only independent audit verifier shown outside
the boundary.

Figure~\ref{fig:architecture} shows the reference boundary. The network and
agent data plane are untrusted. The issuer need not be online for every local
spend: partition-local availability derives from credit already escrowed to
that branch. Cross-zone transfer, reclaim, or expansion can require an online
coordinator.

\begin{figure*}[t]
\centering
\begin{floatbox}
\resizebox{0.98\linewidth}{!}{%
\begin{tikzpicture}[
  >=Stealth,
  box/.style={draw, rounded corners=2.8mm, thick, align=center,
    minimum width=2.65cm, minimum height=0.86cm, inner sep=5pt, font=\small},
  edge/.style={->, thick},
  trust/.style={draw=opRed, dashed, rounded corners=2.8mm, thick, inner sep=5mm}
]
\node[box, draw=opBlue, fill=opBlueFill, text=opBlueText] (issuer)
  {\textbf{Root issuer}\\grant + budget vector};
\node[box, draw=opPurple, fill=opPurpleFill, text=opPurpleText, right=8mm of issuer] (ctrl)
  {\textbf{Delegation controller}\\exclusive split / transfer};
\node[box, draw=opAmber, fill=opAmberFill, text=opAmberText, right=8mm of ctrl] (branch)
  {\textbf{Branch store}\\credit + reservation + outbox};
\node[box, draw=black!35, fill=black!2, right=15mm of branch] (agent)
  {\textbf{Untrusted agent}\\plan / delegate / retry};

\node[box, draw=opPurple, fill=opPurpleFill, text=opPurpleText, below=12mm of branch] (gateway)
  {\textbf{Effect gateway}\\idempotency + fence + effect};
\node[box, draw=opGreen, fill=opGreenFill, text=opGreenText, right=8mm of gateway] (receiver)
  {\textbf{External receiver}\\durable domain state};
\node[box, draw=opPurple, fill=opPurpleFill, text=opPurpleText, below=12mm of ctrl] (receipt)
  {\textbf{Evidence authority}\\signed success / no-effect};
\node[box, draw=opAmber, fill=opAmberFill, text=opAmberText, left=8mm of receipt] (recon)
  {\textbf{Reconciler}\\query / settle / retire};
\node[box, draw=opBlue, fill=opBlueFill, text=opBlueText, below=24mm of receipt] (audit)
  {\textbf{Independent audit verifier}\\trace + conservation totals};

\draw[edge] (issuer) -- (ctrl);
\draw[edge] (ctrl) -- (branch);
\draw[edge] ([yshift=3.5pt]agent.west) --
  node[above, font=\footnotesize, inner sep=1pt]{proposal}
  ([yshift=3.5pt]branch.east);
\draw[edge] ([yshift=-3.5pt]branch.east) --
  node[below, font=\footnotesize, inner sep=1pt]{decision}
  ([yshift=-3.5pt]agent.west);
\draw[edge] (branch) -- node[left, font=\footnotesize]{reserved dispatch} (gateway);
\draw[edge] (gateway) -- (receiver);
\draw[edge] (receiver) |- (receipt);
\draw[edge] (receipt) -- (recon);
\draw[edge] (recon) -- (branch);
\draw[edge, dashed] (receipt.south) --
  node[right, font=\footnotesize, inner sep=2pt]{read-only evidence export}
  (audit.north);

\node[trust, fit=(issuer) (ctrl) (branch) (gateway) (receipt) (recon),
  label={[font=\footnotesize, text=opRedText]above:declared trusted enforcement envelope}] {};
\end{tikzpicture}%
}
\end{floatbox}
\caption{Reference architecture and declared trust boundary. The untrusted agent
may propose and retry; only the trusted branch store can move budget state.
Protected external effects pass through a gateway whose authenticated outcome
drives idempotent settlement. The independent verifier is a read-only evidence
consumer outside the safety TCB.}
\Description{A three-tier architecture. A root issuer sends a grant and budget to
a delegation controller and branch store beside an untrusted agent. Reserved
dispatch flows to an effect gateway and external receiver. A terminal-evidence
authority and reconciler feed authoritative outcomes back to the branch store.
An independent verifier outside the boundary receives the diagrammed
read-only evidence export and verifies trace/conservation records. A dashed
boundary identifies only the trusted enforcement components.}
\label{fig:architecture}
\end{figure*}
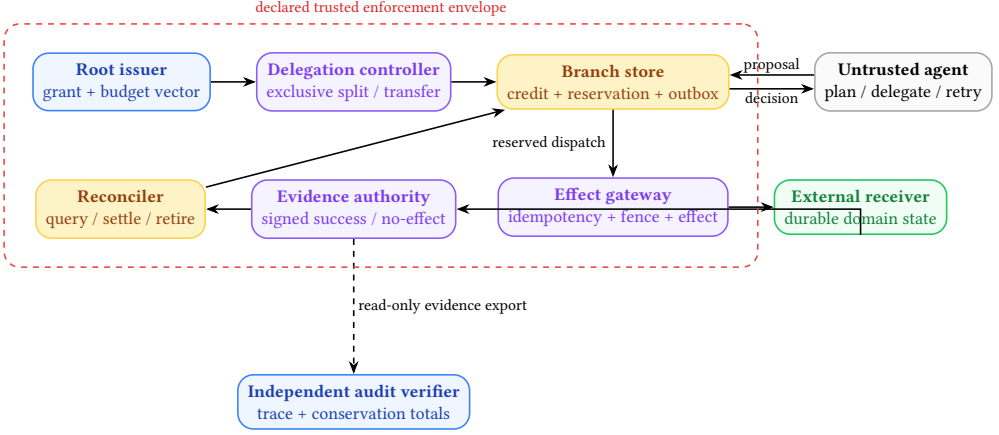

\subsection{Budget and Effect Surfaces}

The protected budget is a finite vector over dimensions
\(\Resources\): currency quanta, paid API calls, risk units, recipient-specific
disclosures, or effect counts. Each dimension has an authenticated unit,
quantization, and maximum-charge rule. Cross-dimension exchange is disallowed
unless a separately authorized conversion contract is present; otherwise a
dollar cannot silently substitute for a risk unit.

Protected effects include any mediated external state change whose authority
is consumable: a payment, message, credential issuance, infrastructure write,
record mutation, disclosure, or paid call. Pure reads may still consume an API
or disclosure budget. The normalizer maps protocol-specific requests to a
canonical effect descriptor. Effects outside the declared gateway envelope,
including covert channels and direct credentials, are outside the theorem.

\subsection{Failure Model}

The network is asynchronous. Messages can be lost, duplicated, delayed,
reordered, or delivered after an arbitrary partition. Processes can crash and
restart, losing volatile state while preserving completed durable commits. We
do not use wall-clock order to infer whether an effect happened. Clocks are
used only for policy expiry under a stated deployment profile; epochs and
authoritative terminal evidence determine settlement.

The model includes concurrent fan-out and fan-in, duplicate delegation and
dispatch, retries, crash before or after each durable write, response loss,
timeout followed by late completion, branch partition and healing, stale epoch
messages, repeated terminal evidence, and recovery scans. It excludes Byzantine
consensus and arbitrary corruption inside a trusted durable store.

\subsection{Adversary}

The planner, model, tool description, retrieved content, and downstream agent
may be malicious or confused. They can request overlapping child budgets,
replay messages, alias a reservation across effects, relabel effects, trigger
retry storms, construct joins with shared ancestry, withhold completion, and
present stale epochs. They cannot forge issuer or terminal-evidence signatures, break
collision-resistant digests, bypass the protected gateway, or directly mutate
the trusted store under the main theorem.

\subsection{Assumptions}

\begin{assumption}[Authenticated roots and credits]
\label{ass:auth}
Root grants, credit generations, controller identities, and protocol bindings
are unforgeable and domain-separated.
\end{assumption}

\begin{assumption}[Durable local linearization]
\label{ass:durable}
Each branch transition atomically updates credit, operation, outbox, and audit
records in one durable local transaction.
\end{assumption}

\begin{assumption}[Unique root issuance and stable decisions]
\label{ass:escrow}
Each root generation is issued once. Local split, join, reservation,
settlement, and successor-generation decision keys are linearized by their
authoritative store. In particular, one write-once
\((g,\chi_c,cid,\eta_c,\mcode{reissue})\) decision, whose predecessor epoch
comes from the authenticated credit record rather than a caller parameter,
selects at most one successor pair \((\eta^+,c^+)\). Each
cross-store transfer ID has one monotone decision register whose value can
change from \mcode{pending} to exactly one of \mcode{accepted} or
\mcode{cancelled}. Exclusive ownership is a protocol invariant to be proved,
not part of this assumption.
\end{assumption}

\begin{assumption}[Sound effect boundary]
\label{ass:effect}
All protected paths are normalized to the correct resource dimensions and
effect identity before dispatch.
\end{assumption}

\begin{assumption}[Fenced idempotent gateway]
\label{ass:gateway}
For each receiver-scoped key, gateway acceptance, fence installation, effect
record creation, and terminal outcome have one linearization order. The
gateway selects the authoritative epoch and fence by the authenticated
\((g,\chi_c)\) committed inside the operation key and \(\beta_o\).
The
gateway rejects a not-yet-accepted stale generation, treats a canonical
operation key as unique within the receiver namespace, rejects any changed
binding under that key, and treats exact replay idempotently. A request
accepted before a fence may complete afterward and remains queryable; the
fence prevents only a new old-generation acceptance. The gateway revalidates
the bound domain precondition and commit-time witness immediately before
acceptance.
\end{assumption}

\begin{assumption}[Maximum-charge contract]
\label{ass:charge}
For endpoint profile \(\chi\), a trusted function
\(\mathsf{RequiredCap}(e,\chi)\) returns a conservative vector. Reservation
requires
\(\mathsf{RequiredCap}(e,\chi)\preceq\mathbf u_o\), and the gateway prevents
the complete actual charge of the accepted effect instance from exceeding
\(\mathbf u_o\). If no such cap can be enforced, the hard-bound profile is
unsupported.
\end{assumption}

\begin{assumption}[Terminal-evidence completeness and non-equivocation]
\label{ass:receipt}
Authenticated terminal evidence is a tagged success receipt or no-effect
certificate. Both bind the operation ID and key, complete acceptance digest
\(\beta_o\), root, logical slot and credit lineage, epoch, effect digest,
receiver, outcome, issuer, and signature. A success receipt additionally binds the unique
instance ID, gateway acceptance sequence, and complete actual charge no
greater than the reservation. A no-effect certificate instead binds a
receiver-installed \(\mathit{fenceEpoch}>\eta_o\). The authority does not sign
conflicting terminal variants. A success receipt exists only for the unique
effect instance accepted under that exact \(\beta_o\). A no-effect certificate
may be issued only when the exact
\((g_o,\chi_o,\kappa_o,\beta_o,\eta_o)\) binding never
linearized \(\mathsf{GatewayAccept}\); it proves both
non-acceptance/non-occurrence and that its installed fence makes every future
occurrence under that binding impossible. A fence does not cancel an already
accepted request.
\end{assumption}

\begin{assumption}[Adapter refinement]
\label{ass:adapter}
MCP, A2A, HTTP, and runtime adapters preserve canonical operation and transfer
identifiers, root, logical slot, credit lineage and footprint commitments, source and
operation labels, resource and maximum-charge vectors, selector schema version,
epoch, \(\kappa\), \(\beta_o\), normalized effect and receiver, joint witness,
and every terminal-evidence binding and committed/remainder footprint
commitment used by the abstract transition.
\end{assumption}

\begin{assumption}[Complete mediation]
\label{ass:mediation}
An agent cannot reach a protected effect with credentials or a path outside
the branch store and effect gateway.
\end{assumption}

\subsection{Goals and Explicit Non-Goals}

The safety goals are global ledger conservation, descendant
non-amplification, at-most-once settlement, continuous injective effect
coverage, a root-bounded external exposure vector, stale-epoch rejection,
late-completion safety, and partition confinement. Concrete trace lifting is a
separate implementation proof obligation. Liveness is conditional: an online
branch with sufficient local escrow and a responsive gateway can progress
without contacting the root issuer. A permanently partitioned,
escrow-exhausted branch can safely refuse forever.

The paper does not optimize allocation, predict token or price use, guarantee
receiver honesty beyond Assumption~\ref{ass:receipt}, reverse completed
effects, provide exactly-once network delivery, tolerate Byzantine
infrastructure, conceal accounting data from the evidence plane, or establish
production service-level objectives (SLOs).

\begin{table*}[t]
\centering
\small
\begin{tabularx}{\textwidth}{L{0.17\textwidth}L{0.19\textwidth}YY}
\toprule
\textbf{Fault or attack} & \textbf{Naive failure} &
\textbf{Required mechanism} & \textbf{Safe consequence} \\
\midrule
Concurrent fan-out & Two children copy one remainder & Exclusive atomic split
with stable operation ID & Inputs consumed once; duplicate is a stutter \\
Multi-parent join & Shared ancestor counted twice & Disjoint lineage test and
atomic input consumption & Alias rejected; distinct inputs may merge \\
Duplicate dispatch & Effect executes twice & Canonical effect key at fenced
gateway & Replay returns same outcome \\
Response loss & Timeout interpreted as failure & Durable obligation and
\mcode{indeterminate} state & Credit stays reserved \\
Crash after effect & Ledger misses consumption & Authenticated receipt query
and idempotent settlement & Recovery moves reservation once \\
Stale worker & Old owner spends reclaimed credit & Monotone epoch and gateway
fence & Old generation rejected \\
Network partition & Both sides spend cached total & Preallocated exclusive
escrow & Each component spends only its rights \\
Unknown receiver & Fabricated ``no effect'' proof & Evidence authority or
permanent retirement & No unsafe automatic refund \\
\bottomrule
\end{tabularx}
\caption{Failure semantics. Safety may reduce availability or strand budget;
the protocol reports that cost rather than hiding it as successful execution.}
\label{tab:faults}
\end{table*}

\section{Why Simpler Mechanisms Fail}
\label{sec:failures}

\subsection{Allocation Conservation Is Not Prefix Conservation}

The static condition
\[
  \sum_{i\in\mathsf{children}(p)}b_i\le b_p
\]
checks one allocation snapshot. It says nothing about whether an old child
still owns \(b_i\) after a refund, whether the same request was applied twice,
or whether an external effect already consumed the allocation. The property
we require holds for every execution prefix and accounts for in-flight and
uncertain obligations.

\subsection{Read--Check--Decrement Replicas Double Spend}

If two partitioned branches read a global remaining value \(1\), both can
locally verify \(1\ge1\) and decrement. Last-writer-wins convergence can
produce the valid-looking final value \(0\) after two effects. CRDT convergence
is not invariant preservation; unrestricted decrement is not
invariant-confluent \cite{bailis2014coordination}. Exclusive rights make each
local decrement monotone with respect to ownership.

\subsection{Static Splitting Is Safe but Can Strand Utility}

Giving each of \(n\) possible workers \(B/n\) before execution prevents
double-spend, but idle branches strand capacity and dynamic fan-out requires
worst-case over-reservation. Token Budgets measures this cost in its setting
\cite{khan2026tokenbudgets}. Our protocol retains the safety of static escrow
while allowing coordinated, idempotent transfer and return. It does not
pretend that reallocation remains available through every partition.

\subsection{Timeout Refund Creates a Late-Completion Race}

Let an operation be dispatched before time \(t\), let no response arrive by
\(t+\Delta\), and let the controller refund. There are two executions
indistinguishable to the controller at \(t+\Delta\): one in which the gateway
never executed and one in which it executed but the response is delayed. A
refund is safe in the first and unsafe in the second. No timeout value resolves
the information gap. The system needs a fenced, authoritative no-effect proof
or must retain/retire the reservation.

\subsection{Idempotency Keys Do Not Allocate Authority}

A stable request key can stop the same receiver from applying one effect
twice. It does not stop two different operation keys from spending the same
credit, a relabeled path from creating the same normalized effect, or two
receivers from interpreting keys in different namespaces. The key must be
bound to one credit generation, effect digest, receiver, epoch, and maximum
charge, and the gateway must enforce that binding.

\subsection{Scope Attenuation Does Not Consume}

A scope token answers which operations a child may invoke. It can be copied to
two children without widening the scope set, yet both children may spend the
last dollar. Scope non-amplification and resource conservation are orthogonal.
We compose with AIP/AAT-style identity and attenuation rather than replacing
them.

\subsection{Compensation Is Not Conservation}

A Saga can schedule a refund after a duplicated purchase, but the duplicate
effect occurred and compensation can fail, cost money, or be semantically
incomplete. Compensation is a valid recovery policy for designated effects;
it is not evidence that the authorization budget was never exceeded. We record
compensation as a further effect with its own authority and accounting.

\subsection{Independent Correct Components Can Compose Incorrectly}

An allocation service, idempotent gateway, lease manager, and receipt verifier
can each be locally correct while disagreeing on operation identity or epoch.
For example, a gateway may deduplicate on an HTTP key while the ledger retries
under a new operation ID; or a lease manager may fence the branch while the
receiver accepts the old generation. The protocol therefore requires one
joint canonical binding and a concrete-to-abstract refinement contract.

\subsection{The Joint Conditions Are Individually Necessary}

The external bound is not obtained by merely placing familiar components next
to one another. It uses four cross-component conditions on the same canonical
operation:
\[
  \mathcal J=
  \mathsf{ExclusiveOwnership}\land
  \mathsf{InstanceInjectivity}\land
  \mathsf{CompleteCharge}\land
  \mathsf{FenceOrdering}.
\]
Each conjunct has a one-unit counterexample if omitted while the others are
retained:

\begin{enumerate}
  \item without exclusive ownership, two branches reserve the same atom under
        distinct valid keys and each produces one correctly receipted effect;
  \item without instance injectivity, one valid reservation and key produces
        two distinct receiver instances while only one operation remains in
        the ledger;
  \item without complete charge, the gateway enforces a cap of one but a
        receipt reports zero for a one-unit effect, so internal settlement is
        conserved while actual cost is not; and
  \item without fence ordering, an old request can be accepted after its atom
        is reissued, then both the old and successor operations complete.
\end{enumerate}

Consequently, bounded counters alone do not imply effect conservation, and
durable completion records alone do not imply delegated-right conservation.
The contribution proved below is the composition: \(\mathcal J\), together
with persist-before-exposure and complete mediation, turns exact ownership
partition into the external exposure bound in
Equation~\eqref{eq:external-bound}.

\section{Formal Budget and Delegation Model}
\label{sec:formal}

\subsection{Resource Vectors and Labels}

Let \(\Resources\) be a finite set of authenticated resource dimensions. A
budget is a vector \(\mathbf b\in\NatVec\). We use componentwise order,
addition, and partial subtraction:

\[
  \mathbf x\preceq\mathbf y
  \iff \forall k\in\Resources:\ x_k\le y_k,
  \qquad
  \mathbf y\mathbin{\dot-}\mathbf x
  \text{ is defined iff }\mathbf x\preceq\mathbf y.
\]

Quantization is explicit. A currency dimension may use cents; an effect count
uses integer events. The abstract results extend to cancellative commutative
resource monoids with decidable subresource, but the executable profile uses
finite natural vectors.

A resource label
\[
  \lambda=(g,\tau,p,\mathit{purpose},\mathit{audience},
           \mathit{effectClass},t_{\exp},v)
\]
binds root grant \(g\), task atom \(\tau\), principal \(p\), purpose,
recipient/audience, effect class, expiry, and policy version. The relation
\(\lambda'\restrict\lambda\) means that every behavior denoted by
\(\lambda'\) is permitted by \(\lambda\). It is a partial order. Delegation
requires both quantitative and semantic attenuation:

\begin{equation}
  \mathbf b_{\mathrm{child}}\preceq\mathbf b_{\mathrm{parent}},
  \qquad
  \lambda_{\mathrm{child}}\restrict\lambda_{\mathrm{parent}}.
\label{eq:dual-attenuation}
\end{equation}

An authenticated conversion contract is required to exchange dimensions. In
its absence all conservation equations are componentwise and no conversion is
defined.

Expiry and policy-version currency are evaluated as external predicates in
the commit-time witness. The budget-ownership machine does not infer trusted
time or independently prove policy freshness.

\subsection{Root Rights and Compressed Footprints}

\begin{definition}[Root right universe]
\label{def:universe}
For root grant \(g\) with issued vector \(\mathbf B_g\), define
\[
  U_g=\biguplus_{k\in\Resources}
      \{(g,k,j)\mid 1\le j\le B_{g,k}\}.
\]
Each element is a logical consumable right in one dimension. It need not be
materialized; an implementation may represent contiguous ranges, rights
matrices, or balance certificates.
\end{definition}

For \(X\subseteq U_g\), let
\[
  \mathsf{vec}(X)_k
  = |\{(g,k,j)\in X\}|.
\]
Sets make aliasing explicit. A split partitions one footprint; a join unions
only pairwise-disjoint footprints. Numeric balances are a compressed
implementation of this ownership semantics.

\subsection{Credits and the Delegation DAG}

\begin{definition}[Credit generation]
\label{def:credit}
A credit generation is
\[
  c=\langle cid,g,\chi_c,\mathit{owner},P_c,X_c,\lambda_c,\eta_c,\sigma_c\rangle
\]
where \(cid\) is the generation identifier, \(\chi_c\) is the stable logical
slot (reissue-chain) identifier, \(P_c\) is a finite set of parent generations,
\(X_c\subseteq U_g\) is its currently available footprint,
\(\lambda_c\restrict\lambda_g\), \(\eta_c\) is a monotonically ordered epoch,
and
\(\sigma_c\in\{\mcode{active},\mcode{closing},\mcode{closed}\}\) is its
lifecycle status. Its available vector is
\(\mathbf a_c=\mathsf{vec}(X_c)\). A reissue preserves \(\chi_c\) while
creating a fresh \(cid\) and higher epoch; a split or join creates fresh
logical slots for its output credits. Consequently, disjoint split siblings
may be current concurrently.
\end{definition}

For state \(s\), let \(\mathsf{ActiveGeneration}_s(g,\chi)\) denote the
credit authority's current-generation register for logical slot \((g,\chi)\).
The current-generation predicate is
\begin{equation}
\begin{split}
\Current_s(c,\eta)\iff{}&
  \sigma_c=\mcode{active}\land\eta=\eta_c\\
 &\land\eta=\mathsf{AuthoritativeEpoch}_s(g,\chi_c)\\
 &\land\mathsf{ActiveGeneration}_s(g,\chi_c)=cid .
\end{split}
\label{eq:current-generation}
\end{equation}
Uniqueness is therefore per logical reissue chain, not per root, owner, or
authority. It prevents two generations of the same slot from being spendable
without forbidding independently owned, pairwise-disjoint sibling slots.
This predicate never substitutes for ownership: every ownership-moving action
also requires its referenced input footprint to have exactly the stated
singleton raw source tag. For dispatch that source is \(\mathsf R(o)\), not
the credit's remaining available footprint.

The task graph \(\Graph_T=(V_T,E_T)\) may be a DAG. The credit-lineage graph
\(\Graph_C\) is also acyclic, but its edges have resource semantics. A
multi-parent credit can be created only by \(\JoinAct\), which consumes
pairwise-disjoint inputs atomically. Thus a task node may depend on two
branches without owning both branches' budgets; dependency is not ownership.

\begin{definition}[Lineage-disjoint join]
\label{def:join}
For current inputs \(c_1,\ldots,c_m\), a join is admissible only if
\[
  \forall i\ne j:\ X_{c_i}\cap X_{c_j}=\varnothing
  \quad\land\quad
  \bigwedge_{i=1}^{m}\Current_s(c_i,\eta_{c_i})
  \quad\land\quad
  \exists g^\star\ \forall i:\ g_{c_i}=g^\star,
\]
and the branch store can consume all inputs in one authoritative transition.
The output is a same-root credit \(c'\) with \(g_{c'}=g^\star\) and footprint
\(\bigcup_i X_{c_i}\); its label must be an authenticated common lower bound,
\(\forall i:\lambda'\restrict\lambda_{c_i}\). We do not assume that the
label partial order has a greatest lower bound.
\end{definition}

The set representation prevents two aliases of one ancestor from being added.
In a compressed implementation, each transfer certificate carries a lineage
commitment and a disjointness witness checked against durable spent/transfer
records.

\subsection{Transfers}

A cross-store delegation cannot atomically debit one database and credit
another without a distributed transaction. We represent a durable handoff
\[
  x=\langle xid,c_s,\mathit{src},\mathit{dst},X_x,
       \lambda_x,\eta_s,\sigma_x\rangle .
\]
The source atomically removes \(X_x\) from \(X_{c_s}\) and records
\(\sigma_x=\mcode{sealed}\). Those atoms are \emph{in transit}: neither source
nor destination may spend them.

Every \(xid\) has one monotone decision register
\begin{equation}
\delta[xid]\in
\{\mcode{pending},\mcode{accepted}(dst,cid),
  \mcode{cancelled}(src)\}.
\label{eq:transfer-decision}
\end{equation}
The compare-and-set from pending to accepted is the logical ownership switch
to the uniquely named destination child; a crash after that decision is
recovered by materializing the same child row, which is an abstract stutter.
Source cancellation requires a destination-exclusion proof and a
compare-and-set from pending to cancelled; that decision logically restores
the source, and later row finalization only materializes it. Thus accept and
cancel cannot both become terminal, and stale source or destination rows never
override the decision register. If no shared decision authority or exclusion
protocol is available, cancellation is disabled and the rights remain in
\(\mathsf T\). Timeout is never a decision.

\subsection{Operations, Effects, and Terminal Evidence}

Throughout, \(H(\mcode{tag}\parallel\cdots)\) denotes a
collision-resistant, domain-separated hash over one length-prefixed canonical
encoding. Maps and sets have deterministic key/element order, integers have a
single representation, and strings have a fixed Unicode normalization.
\(H(P_c)\) and \(H(X_o)\) may be verified lineage and rights commitments
rather than materialized sets, but their opening/verification contract must
bind exactly the modeled values.

\begin{definition}[Normalized effect]
\label{def:effect}
A normalized effect is
\[
  e=\langle \mathit{class},\mathit{resource},\mathit{subject},
  \mathit{receiver},\mathit{payloadDigest},
  \mathit{domainPrecondition}\rangle .
\]
The function \(\Normalize\) maps a protocol-specific proposal and current
domain state to \(e\). Effect-equivalent MCP, A2A, HTTP, shell, or browser
paths share the same digest under the deployment's effect contract.
\end{definition}

\begin{definition}[Effect instance]
\label{def:effect-instance}
An effect descriptor is a template; an occurrence is a distinct instance
\[
 \iota=\langle iid,\mathit{gateway},\mathit{receiver},\kappa,\beta,h_e,
                   \mathit{acceptSeq}\rangle .
\]
The acceptance sequence is assigned by the gateway's per-key linearization
domain, and \(\beta\) is the complete accepted-operation binding defined
below.  The one-effect profile permits at most one \(iid\) for a canonical
key. If an endpoint commits a bundle atomically, the entire bundle is one
normalized instance and its receipt reports the bundle's complete charge.
For a terminal instance, \(\mathsf{ActualCharge}(\iota)\) is the unique
resource vector assigned to that complete instance by the trusted charge
contract; Assumption~\ref{ass:receipt} requires the committed receipt to
report exactly this vector. It is not inferred from an agent's estimate or a
partial stream observation.
\end{definition}

\begin{definition}[Operation reservation]
\label{def:operation}
An operation is
\[
  o=\langle oid,g,\chi_c,cid,\eta_c,\kappa,e,h_e,X_o,
 \mathbf u,\lambda_o,\Phi_o,\nu_{\mathsf{sel}},\sigma_o\rangle ,
\]
where \(oid\) is a stable operation ID. The source credit is \(c_o\); we write
\(g_o=g_{c_o}\), \(\chi_o=\chi_{c_o}\), \(cid_o=cid_{c_o}\), and
\(\eta_o=\eta_{c_o}\), while the unsubscripted source fields displayed in the
tuple denote those same immutable values. Moreover,
\[
 \kappa=H(\mcode{op}\parallel issuer\parallel g\parallel cid\parallel oid
              \parallel receiver\parallel h_e)
\]
is its receiver-scoped canonical idempotency/query key.  The collision-
resistant, domain-separated derivation makes a key name one operation binding
within the threat model; \(\kappa\) is not an agent-chosen alias.  Here \(e\)
and \(h_e\) are the normalized effect and digest, \(X_o\) is an exclusively
reserved footprint, \(\mathbf u=\mathsf{vec}(X_o)\) is the maximum charge,
\(\lambda_o\restrict\lambda_c\), \(\Phi_o\) is the joint authorization
witness, and \(\sigma_o\) is the operation phase.
\(\nu_{\mathsf{sel}}\) names the immutable deterministic footprint-selector
schema used if a successful receipt reports a partial charge.
\end{definition}

The gateway does not accept a loose collection of fields.  It linearizes the
domain-separated digest
\begin{equation}
\begin{split}
\beta_o=H(&\mcode{accept}\parallel oid\parallel g\parallel\chi_c
  \parallel cid\parallel\eta_c\parallel\kappa\parallel H(P_c)\parallel H(X_o)\\
 &\parallel h_e\parallel receiver\parallel\mathbf u
  \parallel H(\lambda_c)\parallel H(\lambda_o)\parallel H(\Phi_o)
  \parallel\nu_{\mathsf{sel}}).
\end{split}
\label{eq:accept-binding}
\end{equation}
Thus the root, logical slot, source lineage, generation, reserved rights,
semantic ceiling, effect, receiver, maximum charge, selector schema, and
commit witness compared by \(\Covered\) are fields of one authenticated
acceptance binding rather than implicit attributes of \(\iota\).

The operation phases are:

\[
\begin{array}{ll}
\mcode{reserved}:&
  \text{durable obligation and outbox exist; dispatch not known to begin},\\
\mcode{indeterminate}:&
  \text{dispatch may have crossed the effect boundary},\\
\mcode{committed}:&
  \text{authoritative success receipt applied by one terminal decision},\\
\mcode{verified-aborted}:&
  \text{durable local non-dispatch proof or authoritative no-effect
  certificate},\\
\mcode{forfeited}:&
  \text{rights permanently unavailable; uncertainty is not refunded}.
\end{array}
\]

Dispatch moves \mcode{reserved} to \mcode{indeterminate} before or atomically
with exposing work to the gateway. A synchronous success can proceed directly
to settlement, but a crash at any cut leaves a durable reservation or
indeterminate obligation.

\begin{definition}[Authenticated terminal evidence]
\label{def:receipt}
Externally authenticated terminal evidence is the tagged sum
\[
\epsilon ::= r^+ \mid \pi_{\neg e},
\]
where
\begin{align*}
r^+={}&\langle rid,oid,iid,\mathit{acceptSeq},\kappa,\beta_o,g,\chi_c,\\[-1mm]
 &\quad cid,\eta_c,h_e,\mathit{receiver},\mathbf d,\mcode{committed},
 t,\mathit{issuer},\mathit{signature}\rangle,\\
\pi_{\neg e}={}&\langle pid,oid,\kappa,\beta_o,g,\chi_c,cid,\eta_c,\\[-1mm]
 &\quad h_e,\mathit{receiver},\mathit{fenceEpoch},
 \mcode{verified-no-effect},t,\mathit{issuer},\mathit{signature}\rangle .
\end{align*}
\(\Verify_+(r^+,o)\) holds only when every binding agrees, the issuer is
trusted for the receiver, the signature is valid,
\(\zero\preceq\mathbf d\preceq\mathbf u_o\), and
\(\mathbf d=\mathsf{ActualCharge}(\iota)\) for the identified gateway
instance. \(\Verify_{\neg e}(\pi_{\neg e},o)\) additionally requires exact
\((g_o,\chi_o,\kappa_o,\beta_o,\eta_o)\) agreement, a trusted proof authority,
and \(\mathit{fenceEpoch}>\eta_o\) installed for the same
\((g_o,\chi_o)\) at the receiver, and establishes the
permanent non-acceptance/non-occurrence condition in
Assumption~\ref{ass:receipt}. Let
\(\Verify_{\mathrm{term}}(\epsilon,o)\) select the appropriate variant
verifier. Neither variant may substitute for the other.

A fused local recovery transaction additionally admits a distinct durable
non-dispatch witness
\[
\begin{aligned}
\pi_{\mathrm{local}}=\langle&
 lid,oid,\kappa,\beta_o,g,\chi_c,cid,\eta_c,h_e,\mathit{outboxSeq},\\
 &H(\mathit{outboxRow}),\mcode{never\mbox{-}dispatchable},
 t,\mathit{issuer},\mathit{authenticator}\rangle .
\end{aligned}
\]
\(\mathsf{LocalNoDispatchPre}_s(\pi_{\mathrm{local}},o)\) holds only when all
immutable operation and acceptance-binding fields agree; the issuer and
authenticator identify the authoritative durable operation store; and the
named outbox row's authenticated monotone history has never reached
\mcode{dispatchable}. The protected send path accepts only rows carrying an
earlier durable \mcode{dispatchable} marker. Define
\[
\begin{split}
\Verify_{\mathrm{local}}(\pi_{\mathrm{local}},o,s,s')\iff{}&
 \mathsf{LocalNoDispatchPre}_s(\pi_{\mathrm{local}},o)\\
 &\land\mathsf{Outbox}_{s'}[oid]=
   \mcode{local\mbox{-}abort}\\
 &\land\mathsf{PersistedLocal}_{s'}(\pi_{\mathrm{local}},o).
\end{split}
\]
\(\mathsf{PersistedLocal}_{s'}\) means that the authoritative store persisted
and authenticated the witness in that same serialized commit. The absorbing
post-state permanently prevents later dispatch; this witness is therefore not
a textual worker assertion. Define the full terminal-outcome witness type and
the return-only subtype
\[
  \omega ::= r^+\mid\pi_{\neg e}\mid\pi_{\mathrm{local}},
  \qquad
  \zeta ::= \pi_{\neg e}\mid\pi_{\mathrm{local}}.
\]
Thus \(\epsilon\) remains external signed evidence, whereas \(\omega\) is the
closed tagged sum used by a fused terminal-and-successor transaction.
\end{definition}

For every supported selector schema \(\nu_{\mathsf{sel}}\), fix a public total
order \(\prec_{g,\nu_{\mathsf{sel}}}\) on the atoms of each root and
dimension. For fixed \(g\), \(X\subseteq U_g\), and
\(\mathbf d\preceq\mathsf{vec}(X)\), define
\[
  \mathsf{Take}_{\nu_{\mathsf{sel}}}(X,\mathbf d)
  =\bigcup_{k\in\Resources}
    \mathsf{First}_{d_k}
      \bigl(X\cap\{(g,k,j)\mid j\ge 1\},
      \prec_{g,\nu_{\mathsf{sel}}}\bigr).
\]
Here \(\mathsf{First}_n(Y,\prec)\) is the set of the \(n\) least elements of
finite \(Y\); its existence in every dimension follows from
\(\mathbf d\preceq\mathsf{vec}(X)\).
Thus \(D_o=\mathsf{Take}_{\nu_{\mathsf{sel}}}
(X_o,\mathbf d_{r^+})\) is not re-chosen after a crash or software upgrade.
The terminal decision atomically persists commitments to both \(D_o\) and
\(X_o\setminus D_o\). A conforming implementation uses this version-bound
partition; persisting the charge vector alone is insufficient for atom-level
lineage refinement.

Write \(c\preceq_{\mathsf{lin}}c'\) when \(c=c'\) or an authenticated
directed path in \(\Graph_C\) leads from \(c\) to \(c'\). For a source credit
\(c_o\) and candidate terminal ownership transition \(s\rightarrow s'\), let
\(\mathsf{Succ}_{s\rightarrow s'}(c_o)\) contain (i) \(c_o\), or its
authenticated same-lineage successor, only when that target is current in
both states, and (ii) a fresh same-authority child that is absent in \(s\)
and created with \(P_{c_r}=\{c_o\}\), a fresh logical slot \(\chi_{c_r}\),
and the unique current generation for that slot in \(s'\) by this same
transition. It may also contain (iii) the fresh same-slot successor
\(c_r=c^+\) created by \(\RolloverAct\), but only when that same atomic action
satisfies Equations~\eqref{eq:rollover-ready}--\eqref{eq:rollover-action},
closes \(c_o\), and installs \(c^+\) as current for \((g,\chi_{c_o})\).
Every member satisfies
\(c_o\preceq_{\mathsf{lin}}c_r\); the set never contains an ancestor,
unrelated credit, or
accepted-transfer destination. Define the transition predicate
\begin{equation}
\begin{split}
\mathsf{ReturnTarget}(o,c_r,s,s')\iff{}&
 g_{c_r}=g_o
 \land c_r\in\mathsf{Succ}_{s\rightarrow s'}(c_o)
 \land c_o\preceq_{\mathsf{lin}}c_r\\
&\land\Current_{s'}(c_r,\eta_{c_r})\\
&\land\mathsf{Authority}(c_r)=\mathsf{SettlementAuthority}(o)
 \land\lambda_{c_r}\restrict\lambda_{c_o}.
\end{split}
\label{eq:return-target}
\end{equation}
The terminal decision row binds \(c_r\) before the concrete ownership update;
both become visible in one atomic commit, so the abstract machine exposes no
intermediate state. Returning capacity to a broader ancestor or another store
is a subsequent sealed \(\Delegate\) transition, not an implicit side effect
of settlement.

A governance receipt saying ``policy allowed'' is not an effect receipt. A
model- or agent-reported token count is not authoritative unless the resource
contract designates that reporter. The guarantee tiers are:

\begin{enumerate}
  \item \emph{fenced, idempotent, queryable receiver}: ledger and mediated
        effect coverage;
  \item \emph{receipt-producing but non-idempotent receiver}: no unsafe refund
        after ambiguity, possibly permanent forfeiture; and
  \item \emph{agent self-report only}: internal ledger integrity, not actual
        cost/effect conservation.
\end{enumerate}

\subsection{Ownership State}

The ownership state is \emph{not} typed as a total location function. It
contains six raw relations that could, in an invalid state, overlap or omit a
right:
\[
\begin{array}{lll}
\mathcal A_s\subseteq U_g\times\Credits,&
\mathcal T_s\subseteq U_g\times\mathsf{Transfers},&
\mathcal R_s\subseteq U_g\times\Operations,\\
\mathcal Q_s\subseteq U_g\times\Operations,&
\mathcal C_s\subseteq U_g\times\Operations,&
\mathcal F_s\subseteq U_g\times\Operations.
\end{array}
\]
The letters denote available credit, sealed transfer, reserved operation,
indeterminate/quarantined operation, committed consumption, and forfeited
capacity created by the \(\Retire\) action. For a tagged relation, let
\[
 \mathsf{supp}(Z_s)=\{u\mid\exists z:(u,z)\in Z_s\}
\]
and define the candidate-location multiset
\[
\begin{split}
\mathsf{Own}_s(u)={}&
\{\mathsf A(c):(u,c)\in\mathcal A_s\}
\uplus\{\mathsf T(x):(u,x)\in\mathcal T_s\}\\
&\uplus\{\mathsf R(o):(u,o)\in\mathcal R_s\}
\uplus\{\mathsf Q(o):(u,o)\in\mathcal Q_s\}\\
&\uplus\{\mathsf C(o):(u,o)\in\mathcal C_s\}
\uplus\{\mathsf F(o):(u,o)\in\mathcal F_s\}.
\end{split}
\]
A conflicting distinct tag or cross-relation overlap therefore produces two
candidates rather than being hidden by a function type. Byte-identical
physical replicas of one logical row abstract to one relation member and are
governed by the store's uniqueness key. Only after ownership partition is
proved do we write \(\mathsf{loc}_s(u)\) for the unique element of
\(\mathsf{Own}_s(u)\).

Bucket vectors are computed from the raw supports:
\[
\mathbf A_g(s)=\mathsf{vec}(\mathsf{supp}(\mathcal A_s)),\quad\ldots,\quad
\mathbf F_g(s)=\mathsf{vec}(\mathsf{supp}(\mathcal F_s)).
\]
Within one relation, the protocol invariant also requires a right to have at
most one tag. The support deliberately measures resource mass rather than
tag multiplicity; all conservation equalities use it only after
\(\mathsf{Own}_s(u)\) has established the ownership partition. Competing
tags are detected by \(\mathsf{Own}\), not by \(\mathsf{supp}\).

The target invariant is:

\begin{equation}
 \mathbf A_g+\mathbf T_g+\mathbf R_g+\mathbf Q_g+
 \mathbf C_g+\mathbf F_g=\mathbf B_g.
\label{eq:conservation}
\end{equation}

Tracking \(\mathbf F_g\) matters. If crash or distrust permanently strands a
right, writing only ``consumed \(\le B\)'' hides the availability loss.
Retirement keeps the equation exact and makes the safety--utility tradeoff
auditable.

\subsection{Transition System}

The abstract machine starts at
\[
  \mathcal A_{s_0}=U_g\times\{c_g\},\quad
  \mathbf A_g=\mathbf B_g,\quad
  \mathbf T_g=\mathbf R_g=\mathbf Q_g=
  \mathbf C_g=\mathbf F_g=\zero.
\]
For every ownership-moving action, the guard requires
\(\mathsf{Own}_s(u)=\{\ell_{\mathrm{src}}\}\) for every input atom. Its
authoritative commit removes each source pair and inserts exactly one
destination pair in the same transition. A replay can return the prior
decision but cannot execute the relation update twice. Each transition is
atomic at its stated local authority:

\begin{description}[leftmargin=2em,style=nextline]
  \item[\(\Issue(g,\mathbf B_g,\lambda_g)\).]
  Create the root credit with a fresh logical slot, status \mcode{active}, and
  initial epoch \(\eta_c\); atomically set
  \(\mathsf{AuthoritativeEpoch}(g,\chi_c)=\eta_c\) and
  \(\mathsf{ActiveGeneration}(g,\chi_c)=cid\), and assign every atom in
  \(U_g\) to it.

  \item[\(\Split(c,X_1,\ldots,X_m)\).]
  Require that \(c\) is the current owner generation at its authoritative
  epoch, pairwise-disjoint \(X_i\subseteq X_c\), and label attenuation. Remove
  \(\bigcup_i X_i\) from the parent and create children with pairwise-distinct
  fresh logical slots and generation identifiers. In the same commit, install
  each child as \mcode{active} with matching authoritative-epoch and
  active-generation registers. If the parent retains a nonempty available
  footprint or any outstanding operation/transfer it remains current; only if
  the action consumes its complete responsibility does it close the parent and
  clear that slot's active-generation register. A stable request ID makes replay
  a stutter.

  \item[\(\Delegate(c,\mathit{dst},X,\lambda',xid)\).]
  Atomically move \(X\subseteq X_c\) from \(\mathsf A(c)\) to
  \(\mathsf T(x)\) and initialize \(\delta[xid]=\mcode{pending}\).
  \(\mathsf{AcceptTransfer}\) wins the decision compare-and-swap (CAS) and
  logically moves the
  same atoms to a uniquely named child credit with a fresh logical slot;
  \(\mathsf{CancelTransfer}\) needs an exclusion proof and the competing CAS
  to logically restore them to the source. Acceptance atomically installs the
  child as \mcode{active} with matching authoritative-epoch and
  active-generation registers. Later database-row materialization stutters.
  The source cannot spend sealed atoms.

  \item[\(\Reserve(c,o,X,e)\).]
  Require \(X\subseteq X_c\), current epoch, joint witness, and effect/label
  compatibility. Atomically move \(X\) from \(\mathsf A(c)\) to
  \(\mathsf R(o)\) while persisting operation, obligation, outbox, and audit.

  \item[\(\Dispatch(o)\).]
  Revalidate epoch and witness. Before possible gateway acceptance, move
  \(X_o\) from \(\mathsf R(o)\) to \(\mathsf Q(o)\). A stale, provably
  unexposed reservation is returned only through
  \(\ReturnAct(o,\pi_{\mathrm{local}})\), which installs the absorbing local
  outbox marker.

  \item[\(\mathsf{GatewayAccept}(o)\).]
  If the per-key register already contains
  \(\mathsf{Accepted}[\kappa_o]=\beta_o\), return that decision and stutter
  regardless of the operation's later terminal bucket. A different stored
  binding is rejected. Only when the key is unseen may first acceptance
  linearize: require the matching durable dispatchable marker and \(X_o\)
  wholly in \(\mathsf Q(o)\) or absorbing \(\mathsf F(o)\), then recompute and
  atomically accept the complete immutable digest \(\beta_o\) against the
  receiver's key, current witness, and the fence selected by the authenticated
  slot \((g,\chi_c)\). An unexposed reservation or a new old-epoch key is
  rejected.

  \item[\(\mathsf{EffectOccurs}(o,\iota)\).]
  Require a prior accepted \((\kappa_o,\beta_o)\) binding and
  \(\iota.\beta=\beta_o\); append the fresh instance \(\iota\) to the
  gateway's monotone observation log. Under the one-effect profile, a second
  instance for the same \(\kappa_o\) is disabled; transport attempts may
  duplicate without appending another instance.

  \item[\(\Settle(o,r^+)\).]
  Require \(X_o\) wholly in \(\mathsf Q(o)\), a verified success receipt,
  and \(\iota_{r^+}\in\mathsf{Obs}(s)\). Derive
  \(D_o=\mathsf{Take}_{\nu_{\mathsf{sel}}}
  (X_o,\mathbf d_{r^+})\) and persist the exact committed and remainder
  commitments. Move \(D_o\) to \(\mathsf C(o)\). If
  \(X_o\setminus D_o\ne\varnothing\), require a decision-bound \(c_r\)
  satisfying Equation~\eqref{eq:return-target} and move the remainder there.
  Settlement never widens the source label, changes the root, jumps to an
  ancestor, or resurrects a stale generation.
  If an identical receipt was already applied, stutter; any conflicting
  receipt is rejected.

  \item[\(\ReturnAct(o,\zeta)\).]
  Require exactly one type-matched branch. The local branch has
  \[
    \zeta=\pi_{\mathrm{local}},\qquad
    X_o\text{ wholly in }\mathsf R(o),\qquad
    \Verify_{\mathrm{local}}(\pi_{\mathrm{local}},o,s,s').
  \]
  The external branch has \(\zeta=\pi_{\neg e}\), \(X_o\) wholly in
  \(\mathsf Q(o)\), and an authoritative fenced
  \(\Verify_{\neg e}(\pi_{\neg e},o)\) proof that the exact binding was never
  gateway-accepted and did not and cannot execute. Also require a
  decision-bound \(c_r\) satisfying Equation~\eqref{eq:return-target}. Move
  all \(X_o\) from the branch-specific source to \(c_r\) and mark the
  operation \mcode{verified-aborted}; the local branch atomically installs
  the absorbing outbox marker.

  \item[\(\Retire(o)\).]
  Move an unresolved \(X_o\) from \(\mathsf Q(o)\) to \(\mathsf F(o)\).
  The bucket is absorbing: atoms never become available or committed. Later
  evidence annotates the operation's outcome and actual charge without moving
  ownership, so forfeiture remains a conservative maximum charge.

  \item[\(\Fence(c,\eta')\).]
  For a first installation, require
  \(\Current_s(c,\eta_c)\) and
  \(\eta'>\mathsf{AuthoritativeEpoch}_s(g,\chi_c)=\eta_c\). Atomically advance
  the slot's authoritative epoch, mark \(c\) \mcode{closing}, clear \(cid\) as
  its active generation, and install the same monotone fence at all protected
  gateways before permitting new-generation use. An exact already-installed retry stutters. A
  not-yet-accepted old request is rejected; a request accepted before the
  fence may complete later and remains charged/queryable.

  \item[\(\mathsf{Reissue}(c,\eta,\eta^+,c^+)\).]
  Require Equations~\eqref{eq:closed-generation}, \eqref{eq:reusable}, and
  \eqref{eq:reissue-target}, require
  \(\mathsf{SafeA}_s(c,\eta)=\mathsf{Avail}_s(c,\eta)\), and require the
  write-once reissue decision for
  the normalized predecessor key
  \((g,\chi_c,cid,\eta_c,\mcode{reissue})\) to select
  \((\eta^+,c^+)\), with \(\eta=\eta_c\). In one transaction, persist that
  decision, mark \(c\) \mcode{closed}, update current-generation metadata, and
  move exactly the uniquely available reusable footprint into fresh generation
  \((c^+,\eta^+)\) of the same logical slot. Exact replay
  stutters and a changed target is rejected. Old committed, forfeited,
  transferred, or otherwise obligated atoms do not move.

  \item[\(\RolloverAct(o,\omega,c,\eta,\eta^+,c^+)\).]
  Apply the fused transition in
  Equations~\eqref{eq:rollover-ready}--\eqref{eq:rollover-action} only when
  \(o\) is the final nonterminal operation of the fenced predecessor. In one
  commit, apply its type-matched terminal-outcome witness, move the terminal
  remainder together with the predecessor's already safe available footprint
  into the one normalized reissue target, close the predecessor, and install
  the same-slot successor. This is one abstract action, not an implicit
  relaxation of \(\Settle\) or \(\mathsf{Reissue}\).

  \item[\(\JoinAct(c_1,\ldots,c_m)\).]
  Require Definition~\ref{def:join}, no active operations on the input
  portions, and one atomic authority. Consume inputs and create one output
  over their disjoint union with a fresh logical slot. The same commit closes
  every fully consumed input and clears its active-generation register, then
  installs the output as \mcode{active} with matching authoritative-epoch and
  active-generation registers.
\end{description}

\subsection{Safety Invariants}

For all reachable states, we require:

\begin{enumerate}[label=\textbf{I\arabic*.}]
  \item \textbf{Ownership partition.}
        \(|\mathsf{Own}_s(u)|=1\) for every \(u\in U_g\), including uniqueness
        of the tag inside its raw relation.
  \item \textbf{Non-negative vectors.} All bucket vectors are in \(\NatVec\).
  \item \textbf{Semantic attenuation.} Every descendant label restricts every
        lineage ancestor label applicable to its footprint.
  \item \textbf{Current generation.} A split, delegate, reserve, join, or
        dispatch uses the current owner generation and epoch.
  \item \textbf{Operation binding.} Each \(oid\) derives exactly one canonical
  \(\kappa\). Within its receiver namespace, an accepted \(\kappa\)
        maps to exactly one \(\beta_o\), which immutably binds \(oid\), root,
        logical slot, footprint and lineage commitments, epoch, effect digest, receiver,
        maximum vector, selector schema version, source and operation labels,
        and the joint witness.
  \item \textbf{Settlement uniqueness.} At most one non-conflicting
        authoritative terminal-evidence variant is applied to an operation.
  \item \textbf{Instance injectivity.} The monotone gateway log contains at
        most one effect instance per canonical key, and one operation covers
        at most one instance in the one-effect profile.
  \item \textbf{Dispatch coverage.} Every protected dispatch without a
        terminal verified-no-effect decision has a durable operation whose
        atoms are wholly in \(\mathsf Q\) or \(\mathsf F\), or whose complete
        committed receipt accounts for its actual-charge subset in
        \(\mathsf C\). A full no-effect return is backed either by durable
        never-dispatchable evidence or by authoritative proof that past and
        future occurrence is impossible. A partial-success remainder is
        backed by one complete terminal receipt that proves the exact actual
        charge and disables any further instance or charge under
        \((\kappa_o,\beta_o)\).
  \item \textbf{No stale resurrection.} Advancing an epoch does not make a
        prior generation current again.
\end{enumerate}

\subsection{Effect Instances, Coverage, and External Exposure}

The state variable \(\mathsf{Obs}(s)\) is the gateway's monotone log of
instances appended only by \(\mathsf{EffectOccurs}\). It is a set of
\(\iota\) from Definition~\ref{def:effect-instance}, not a set of effect
descriptors. Thus two separate payments with identical payloads remain two
instances.

The relation \(\Covered(\iota,o,s)\) is non-recursive. It holds exactly when:

\begin{enumerate}
  \item \(\iota.\kappa=\kappa_o\), \(\iota.\beta=\beta_o\),
        \(\iota.h_e=h_{e_o}\), and the gateway and receiver agree; by
        Equation~\eqref{eq:accept-binding}, equality of \(\beta\) also binds
        root, logical slot, lineage, epoch, footprint, labels, maximum vector,
        and witness;
  \item \(o\)'s label and commit-time witness were valid at gateway acceptance;
        and
  \item either \(X_o\) remains wholly in \(\mathsf Q(o)\) or
        \(\mathsf F(o)\), or a complete committed receipt for \(\iota\)
        identifies
        \(D_o=\mathsf{Take}_{\nu_{\mathsf{sel}}}
        (X_o,\mathsf{ActualCharge}(\iota))\) in \(\mathsf C(o)\).
\end{enumerate}

Coverage and instance injectivity are:

\begin{align}
  &\forall \iota\in\mathsf{Obs}(s),\quad
    \exists!\,o\in\Operations:\Covered(\iota,o,s),
  \label{eq:coverage}\\
  &\Covered(\iota_1,o,s)\land\Covered(\iota_2,o,s)
    \Longrightarrow \iota_1=\iota_2.
  \label{eq:injective-coverage}
\end{align}

Let \(\mathsf{Settled}_g(s)\) contain exactly the instances whose unique
covering operation satisfies \(g_o=g\) and whose complete receipt has placed
its actual-charge footprint in \(\mathsf C\). Let
\(\mathsf{Risk}_g(s)=\{o\mid g_o=g\land
X_o\text{ is wholly in }\mathsf Q(o)\text{ or }\mathsf F(o)\}\),
irrespective of any later evidence annotation. The charged footprints
associated with \(\mathsf{Settled}_g\) and the full footprints associated with
\(\mathsf{Risk}_g\) are pairwise disjoint by ownership location. In particular,
a success receipt arriving after forfeiture is recorded for audit but does
not move \(\mathsf F\), does not enter \(\mathsf{Settled}\), and remains
charged at the conservative maximum. The external exposure bound is:

\begin{equation}
 \sum_{\iota\in\mathsf{Settled}_g(s)}
   \mathsf{ActualCharge}(\iota)
 +\sum_{o\in\mathsf{Risk}_g(s)}\mathbf u_o
 \preceq
 \mathbf C_g(s)+\mathbf Q_g(s)+\mathbf F_g(s)
 \preceq \mathbf B_g.
\label{eq:external-bound}
\end{equation}

This is stronger than an internally consistent counter and weaker than a claim
about unmediated physical reality. It additionally depends on the ownership
assumptions used by Theorem~\ref{thm:exact-conservation} and on
Assumptions~\ref{ass:effect}, \ref{ass:gateway}, \ref{ass:charge},
\ref{ass:receipt}, and \ref{ass:mediation}.

\subsection{Descendant Capacity}

For an execution prefix \(\rho\) ending in \(s\), a single-root sub-DAG \(W\),
and a cut separating it from its ancestors, let
\(\mathsf{In}_{\rho}(W)\subseteq U_g\) be the union of exact atom footprints in
authenticated ownership events that crossed from a responsible lineage outside
\(W\) to one inside \(W\) along \(\rho\). This includes cross-cut split,
accepted transfer, return, successor reissue, and join inputs, and is
deduplicated by atom identity on re-entry. Let \(\mathsf{Held}_s(W)\) include
atoms currently available, reserved, indeterminate, committed, or forfeited in
descendants of \(W\). Descendant non-amplification is:

\begin{equation}
  \mathsf{Held}_s(W)\subseteq\mathsf{In}_{\rho}(W),
  \qquad
  \mathsf{vec}(\mathsf{Held}_s(W))
  \preceq\mathsf{vec}(\mathsf{In}_{\rho}(W)).
\label{eq:descendant}
\end{equation}

Transfers out of \(W\) leave the left-hand side. A task dependency edge
without a credit transfer contributes no capacity.

\section{Fault-Tolerant Delegation and Settlement Protocol}
\label{sec:protocol}

\subsection{Design Principles}

The protocol follows six rules:

\begin{enumerate}
  \item \textbf{Move quantities and attenuate labels.} A child receives both
        an exclusive footprint and a narrower-or-equal authorization label.
  \item \textbf{Persist before exposure.} Reservation, obligation, outbox, and
        audit state commit before a protected dispatch can escape.
  \item \textbf{Bind the real effect.} The reservation uses a normalized
        effect digest and current dependency/authority witness, not a tool name
        alone.
  \item \textbf{Treat ambiguity as charged.} Response loss produces
        \mcode{indeterminate}, never automatic refund.
  \item \textbf{Fence before reuse.} Reclaim from an exposed obligation
        requires a gateway-visible fence \emph{and} authoritative evidence that
        the exact binding was never accepted and did not occur; a fence alone
        never releases quarantined rights.
  \item \textbf{Make replay a stutter.} Stable transfer, operation, and
        terminal-evidence IDs return the prior decision without repeating
        state movement.
\end{enumerate}

\subsection{Joint Authorization and Budget Witness}

The prior research stack supplies complementary predicates over one canonical
operation:
\begin{equation}
\begin{split}
\mathsf{ValidWitness}(o)\iff{}&
\mathsf{StaticPolicy}(o)
\land\mathsf{TaskIntent}(o)
\land\mathsf{ProvenanceCeiling}(o)\\
&\land\mathsf{SessionLease}(o)
\land\mathsf{SafeRoute}(o)
\land\mathsf{CurrentClosure}(o)\\
&\land\mathsf{NormalizedEffect}(o)
\land\mathsf{CurrentCredit}(o).
\end{split}
\label{eq:joint-witness}
\end{equation}

OpenPort supplies the server-side action boundary
\cite{zhu2026openport}. Public work on intent-governed authorization and
server-verified action claims supplies concrete examples of the intent and
execution predicates \cite{zhu2026igac,zhu2026ebte}; versioned
dependency-closure binding supplies the current recursive skill closure and
normalized effect \cite{zhu2026closurebound}. The remaining predicates in
Equation~\eqref{eq:joint-witness} are explicit interface assumptions defined
by this paper. They are inputs at the witness boundary, not imported
conservation results: the exclusive distributed-credit and settlement
properties are established here.

\subsection{Credit Certificates}

A transfer certificate contains:

\begin{lstlisting}
{
  "transfer_id": "...",
  "root_grant_id": "...",
  "source_credit_id": "...",
  "source_epoch": 17,
  "destination": "agent://worker-b",
  "resource_vector": {"usd_cents": 4000, "api_calls": 6},
  "lineage_commitment": "...",
  "rights_commitment": "...",
  "authorization_label_digest": "...",
  "expires_at": "...",
  "issuer": "...",
  "signature": "..."
}
\end{lstlisting}

The destination cannot enlarge the vector or label. The certificate identifies
a transfer operation, not a bearer balance that may be imported repeatedly.
Acceptance writes a uniqueness key over
\((\mathit{root},\mathit{transfer\_id})\).

\subsection{Idempotent Delegation}

Algorithm~\ref{alg:delegate} performs a sealed cross-store transfer. A local
split is its single-authority, no-transit specialization. The source
transition is authoritative for removing spendability. Message send occurs
after the durable seal.

\begin{algorithm}[t]
\caption{\textsc{DelegateEscrow}$(c,\mathit{dst},\mathbf q,\lambda',xid)$}
\label{alg:delegate}
\begin{algorithmic}[1]
\Require current credit \(c\), destination, vector, attenuated label, stable transfer ID
\State begin durable source transaction
\If{a decision for \(xid\) exists}
  \State \Return the recorded certificate or rejection
\EndIf
\State require \(\Current_{s_{\mathrm{txn}}}(c,\eta_c)\),
  \(\mathbf q\preceq\mathbf a_c\),
  and \(\lambda'\restrict\lambda_c\)
\State choose footprint \(X\subseteq X_c\) with \(\mathsf{vec}(X)=\mathbf q\)
\State \(X_c\gets X_c\setminus X\)
\State persist sealed transfer \(x=(xid,c,\mathit{dst},X,\lambda',\eta_c)\)
\State initialize \(\delta[xid]=\mcode{pending}\); persist outbox and audit; commit
\State send certificate from outbox, retrying with the same \(xid\)
\State destination verifies signature, epoch, label, and uniqueness
\State destination CASes
  \(\delta[xid]:\mcode{pending}\rightarrow\mcode{accepted}(dst,cid')\)
\If{the decision is \mcode{accepted}}
  \State atomically create or recover the uniquely named child \(cid'\) over \(X\)
  \State source records the terminal decision; it never restores an accepted \(X\)
  \State \Return accepted child certificate
\ElsIf{the decision is \mcode{cancelled}}
  \State reject import; source may restore \(X\) exactly once
  \State source records the terminal decision
  \State \Return durable cancellation or rejection status
\EndIf
\end{algorithmic}
\end{algorithm}

If the destination is unreachable, \(X\) remains in \(\mathbf T_g\). Safe
cancel is a protocol transition requiring destination fencing or a durable
decision that the transfer cannot later be accepted. The implementation may
instead leave it in transit until reconciliation.

\subsection{Reserve Before Dispatch}

Algorithm~\ref{alg:reserve} converts available credit to a durable obligation.
The maximum vector \(\mathbf u\) must be known conservatively. If a provider
can exceed \(\mathbf u\) without rejection, that provider cannot support the
hard-bound profile.

\begin{algorithm}[t]
\caption{\textsc{ReserveAndDispatch}$(c,\mathit{proposal},\mathbf u,oid)$}
\label{alg:reserve}
\begin{algorithmic}[1]
\Require current credit, proposed protected effect, maximum vector, stable
  operation ID, supported selector schema \(\nu_{\mathsf{sel}}\)
\State \(e\gets\Normalize(\mathit{proposal},\mathit{currentDomainState})\)
\State \(h_e\gets H(\mcode{effect}\parallel e)\)
\State compile current joint witness \(\Phi_o\) satisfying Equation~\eqref{eq:joint-witness}
\State require \(\mathsf{RequiredCap}(e,\chi)\preceq\mathbf u\)
\State derive key \(\kappa\gets H(\mcode{op}\parallel issuer\parallel g
  \parallel cid\parallel oid\parallel receiver\parallel h_e)\)
\State begin durable branch transaction
\If{operation \(oid\) exists}
  \State require all immutable bindings equal; \Return recorded state
\EndIf
\State require \(\Current_{s_{\mathrm{txn}}}(c,\eta_c)\),
  \(\mathbf u\preceq\mathbf a_c\),
  \(\Phi_o\) valid, and \(e\) compatible with \(\lambda_c\)
\State choose \(X_o\subseteq X_c\) with \(\mathsf{vec}(X_o)=\mathbf u\)
\State derive \(\beta_o\) from the selected \(X_o\) and every field in
  Equation~\eqref{eq:accept-binding}
\State construct the canonical signed operation envelope containing every
  preimage field of \(\kappa\) and \(\beta_o\), plus both digests
\State move \(X_o:\mathsf A(c)\rightarrow\mathsf R(o)\)
\State atomically persist operation/obligation with \(\beta_o\), the exact
  non-dispatchable envelope in the outbox, and audit; commit
\State begin dispatch transaction; reload the persisted envelope
\State recompute and require its \(\kappa\), \(\beta_o\), signature, and
  immutable operation binding; recheck
  \(\Current_{s_{\mathrm{txn}}}(c,\eta_c)\) and the persisted \(\Phi_o\)
\If{either check is stale while outbox is provably unexposed}
  \State construct the type-tagged candidate \(\pi_{\mathrm{local}}\) from the
    persisted operation and outbox row
  \State bind a decision-fixed target descriptor \(c_{\mathrm{return}}\);
    require this transaction's before/after states satisfy
    Equation~\eqref{eq:return-target}
  \State atomically require
    \(\Verify_{\mathrm{local}}(\pi_{\mathrm{local}},o,s_{\mathrm{txn}},
      s'_{\mathrm{txn}})\), return \(X_o\) to \(c_{\mathrm{return}}\), write
    the canonical terminal-decision/audit records, and mark verified-aborted
  \State commit; \Return the recorded stale-generation rejection
\EndIf
\State construct a branch-signed dispatch permit over \(\beta_o\)
\State atomically move \(X_o:\mathsf R(o)\rightarrow\mathsf Q(o)\), persist
  the permit, and mark the outbox dispatchable; commit before gateway exposure
\State send only the persisted envelope and permit through the fenced gateway
\State before first acceptance, gateway recomputes \(\kappa\) and \(\beta_o\),
  verifies the permit, and revalidates the bound precondition and current
  commit witness; an exact accepted replay returns the stored decision
\State \Return synchronous success receipt if present; otherwise
  \mcode{indeterminate}
\end{algorithmic}
\end{algorithm}

The explicit \(\mathsf R\rightarrow\mathsf Q\) cut is conservative. A crash
after the move but before network send can strand budget until the gateway
proves non-execution. Moving after send would create an uncovered effect
window. Safety chooses possible stranding.

\subsection{Terminal-Evidence Settlement}

Algorithm~\ref{alg:settle} is idempotent and never re-executes the effect.
Success consumes the actual vector and returns only the authenticated
remainder. A no-effect certificate returns the full footprint. Pending or
unknown status is not terminal evidence and leaves it in \(\mathsf Q\).

\begin{algorithm}[t]
\caption{\textsc{SettleOutcome}$(o,\epsilon)$}
\label{alg:settle}
\begin{algorithmic}[1]
\Require durable operation \(o\), candidate terminal evidence
  \(\epsilon\in\{r^+,\pi_{\neg e}\}\)
\State begin durable branch transaction
\If{\(X_o\) is wholly in absorbing \(\mathsf F(o)\)}
  \State require \(\Verify_{\mathrm{term}}(\epsilon,o)\)
  \State append or return the idempotent late-evidence annotation
  \State commit; \Return the unchanged forfeited ownership decision
\EndIf
\If{evidence ID or terminal decision for \(oid_o\) was applied}
  \State require byte-equivalent binding; \Return the prior decision
\EndIf
\State require \(X_o\) is wholly in \(\mathsf Q(o)\)
  \If{\(\epsilon=r^+\)}
  \State require \(\Verify_+(r^+,o)\) and
    \(\iota_{r^+}\in\mathsf{Obs}(s)\)
  \State derive
    \(D_o\gets\mathsf{Take}_{\nu_{\mathsf{sel}}}
      (X_o,\mathbf d_{r^+})\)
    and bind \(H(D_o),H(X_o\setminus D_o)\)
  \If{\(X_o\setminus D_o\ne\varnothing\)}
    \State bind a decision-fixed target descriptor \(c_{\mathrm{return}}\);
      require this transaction's successor state satisfy
      Equation~\eqref{eq:return-target}
  \EndIf
  \State move \(D_o:\mathsf Q(o)\rightarrow\mathsf C(o)\)
  \If{\(X_o\setminus D_o\ne\varnothing\)}
    \State move \(X_o\setminus D_o:\mathsf Q(o)
      \rightarrow\mathsf A(c_{\mathrm{return}})\)
  \EndIf
  \State mark \(o\) committed and persist \(r^+\)/audit atomically
\Else
  \State require \(\epsilon=\pi_{\neg e}\) and
    \(\Verify_{\neg e}(\pi_{\neg e},o)\)
  \State bind a decision-fixed target descriptor \(c_{\mathrm{return}}\);
    require this transaction's successor state satisfy
    Equation~\eqref{eq:return-target}
  \State move \(X_o:\mathsf Q(o)\rightarrow\mathsf A(c_{\mathrm{return}})\)
  \State mark \(o\) verified-aborted and persist
    \(\pi_{\neg e}\)/audit atomically
\EndIf
\State commit; \Return durable decision
\end{algorithmic}
\end{algorithm}

\subsection{Recovery, Reconciliation, and Forfeiture}

Algorithm~\ref{alg:recover} scans operations in \mcode{reserved} or
\mcode{indeterminate}.
Reserved operations whose outbox was never made dispatchable can be aborted
inside the branch transaction. Once dispatchability was committed, the
reconciler queries the gateway by \(\kappa\). It applies an authenticated
success receipt, installs a fence and obtains a no-effect proof, or retains
uncertainty.

\begin{algorithm}[t]
\caption{\textsc{RecoverAndReconcile}$(o)$}
\label{alg:recover}
\begin{algorithmic}[1]
\Require durable nonterminal operation \(o\)
\If{\(o.phase=\mcode{reserved}\) and durable outbox proves never dispatchable}
  \State construct the type-tagged candidate \(\pi_{\mathrm{local}}\)
  \State bind a decision-fixed target descriptor \(c_{\mathrm{return}}\);
    require the atomic transition's before/after states satisfy
    Equation~\eqref{eq:return-target}
  \State atomically require
    \(\Verify_{\mathrm{local}}(\pi_{\mathrm{local}},o,s,s')\), return
    \(X_o\) to \(c_{\mathrm{return}}\), and mark verified-aborted
  \State \Return
\EndIf
\State query gateway by the exact binding
  \((g_o,\chi_o,\kappa_o,\beta_o,\eta_o)\)
\If{authenticated success receipt \(r^+\) is returned}
  \State \(\Call{SettleOutcome}{o,r^+}\)
\ElsIf{gateway atomically fences the old binding/epoch and returns
  \(\pi_{\neg e}\)}
  \State \(\Call{SettleOutcome}{o,\pi_{\neg e}}\)
\ElsIf{administrator chooses permanent safety over utilization}
  \State move \(X_o:\mathsf Q(o)\rightarrow\mathsf F(o)\); preserve obligation
\Else
  \State retain \mcode{indeterminate}; schedule bounded-backoff reconciliation
\EndIf
\end{algorithmic}
\end{algorithm}

Forfeiture is not a cleanup trick. It records that the system cannot safely
reuse the capacity. Later evidence may annotate a committed outcome and actual
charge, but \(\mathsf F\) remains the conservative maximum and no ownership
bucket changes.

\subsection{Epoch Fencing and Safe Reclaim}

A controller may advance a credit epoch after a worker is lost. The new epoch
does not itself reclaim outstanding rights. We first close the old generation:

For the remainder of this section,
\(\mathsf{Ops}(c,\eta)\) and \(\mathsf{Transfers}(c,\eta)\) are monotone
historical index sets: they contain every durably created operation or transfer
whose immutable source binding names credit \(c\) and authenticated epoch
\(\eta\), including terminal records. They are not filtered ``live'' views.
\(\mathsf{SettlementAuthority}(o)\) is the immutable branch authority bound at
\(\Reserve\) that may apply \(o\)'s terminal ownership decision.
\(\mathsf{FenceInstalled}_s(g,\chi,\eta^+)\) means every gateway protecting
that root/slot durably enforces a rejection floor of at least \(\eta^+\).
Finally, let
\[
 K_c=(g,\chi_c,cid,\eta_c,\mcode{reissue})
\]
and let \(\mathsf{Decision}_s[K_c]\) and
\(\mathsf{TermDecision}_s[oid]\) be partial write-once registers, with
\(\bot\) denoting an unassigned value.

\begin{equation}
\begin{split}
\mathsf{Closed}_s(c,\eta,\eta^+)\iff{}&
\eta=\eta_c
\land\eta^+=\mathsf{AuthoritativeEpoch}_s(g,\chi_c)
\land\eta^+>\eta\\
&\land\sigma_c\in\{\mcode{closing},\mcode{closed}\}
\land\mathsf{FenceInstalled}_s(g,\chi_c,\eta^+)\\
&\land\ \forall x\in\mathsf{Transfers}(c,\eta):
  \delta[x]\in\{\mcode{accepted},\mcode{cancelled}\}\\
&\land\ \forall o\in\mathsf{Ops}(c,\eta):
  \sigma_o\in\{\mcode{committed},
  \mcode{verified-aborted},\mcode{forfeited}\}.
\end{split}
\label{eq:closed-generation}
\end{equation}

Let \(\mathsf{SafeA}_s(c,\eta)\) be the atoms for which the reclaim authority
holds a durable proof that
\(\mathsf{Own}_s(u)=\{\mathsf A(c)\}\), together with retained provenance:
the atom was idle, returned by a verified abort or authenticated remainder,
or restored after a \mcode{cancelled} transfer. This raw-ownership condition
does not require the fenced old generation itself to remain
\(\Current_s\). An
\mcode{accepted} transfer, a committed atom, and a forfeited atom are excluded.
Let
\[
 \mathsf{Avail}_s(c,\eta)=
 \{u\in U_g\mid \eta=\eta_c
   \land\mathsf{Own}_s(u)=\{\mathsf A(c)\}\}.
\]
A successor-installing action is enabled only when
\(\mathsf{SafeA}_s(c,\eta)=\mathsf{Avail}_s(c,\eta)\); otherwise the predecessor
remains \mcode{closing} and the unproved capacity is stranded, not silently
reclassified as reusable. The exact reusable footprint is therefore

\begin{equation}
 \mathsf{Reusable}_s(c,\eta,\eta^+)=
 \begin{cases}
   \mathsf{SafeA}_s(c,\eta),&
     \mathsf{Closed}_s(c,\eta,\eta^+)
     \land\mathsf{SafeA}_s(c,\eta)=\mathsf{Avail}_s(c,\eta),\\
   \varnothing,&\text{otherwise}.
 \end{cases}
\label{eq:reusable}
\end{equation}

Let \(\mathsf{FreshAuthenticated}_s(c^+)\) mean that no credit with the signed
identity of \(c^+\) exists in \(s\) and that its credit-authority
authentication verifies. Then
\begin{equation}
\begin{split}
\mathsf{ReissueTarget}&(c,\eta,\eta^+,c^+,s,s')\\
\iff{}&
  \mathsf{FreshAuthenticated}_s(c^+)
  \land\eta=\eta_c\\
 &\land\eta^+=\mathsf{AuthoritativeEpoch}_s(g,\chi_c)
  \land g_{c^+}=g_c\\
 &\land\chi_{c^+}=\chi_c
  \land P_{c^+}=\{c\}\\
 &\land\mathsf{Authority}(c^+)=\mathsf{Authority}(c)
  \land\eta_{c^+}=\eta^+
  \land\lambda_{c^+}\restrict\lambda_c\\
 &\land\Current_{s'}(c^+,\eta^+)
  \land\neg\Current_{s'}(c,\eta).
\end{split}
\label{eq:reissue-target}
\end{equation}

The reference SQLite profile fuses a final terminal decision with successor
installation. To give that concrete transaction one exact abstract image,
define \(\mathsf{ValidTerminal}_s(o,\omega)\) to mean exactly one of:
(i) \(\omega=r^+\), \(o\) is quarantined with all \(X_o\) in
\(\mathsf Q(o)\), \(\Verify_+(r^+,o)\), and the receipt's
\(\iota_{r^+}\in\mathsf{Obs}(s)\);
(ii) \(\omega=\pi_{\neg e}\), \(o\) is quarantined with all \(X_o\) in
\(\mathsf Q(o)\), \(\Verify_{\neg e}(\pi_{\neg e},o)\), and the verified
certificate binds the installed higher fence for
\((g_o,\chi_o)\) and proves that the exact binding was never accepted and
cannot execute; or
(iii) \(\omega=\pi_{\mathrm{local}}\), \(o\) is reserved with all \(X_o\) in
\(\mathsf R(o)\), and
\(\mathsf{LocalNoDispatchPre}_s(\pi_{\mathrm{local}},o)\). Let the terminal
outcome determine two disjoint footprints:
\[
\begin{array}{c|cc}
\omega & \mathsf{CommitFoot}(o,\omega) &
            \mathsf{ReturnFoot}(o,\omega)\\
\hline
r^+ & D_o & X_o\setminus D_o\\
\pi_{\neg e}\text{ or }\pi_{\mathrm{local}} &
\varnothing & X_o .
\end{array}
\]
Define
\[
\mathsf{TermStatus}(\omega)=
\begin{cases}
\mcode{committed},&\omega=r^+,\\
\mcode{verified-aborted},&
  \omega\in\{\pi_{\neg e},\pi_{\mathrm{local}}\}.
\end{cases}
\]
The terminal record stores this tag, the evidence and acceptance-binding
digests, and the exact committed and return footprint commitments:
\[
\begin{aligned}
\mathsf{TermRecord}(o,\omega)=\langle&
 \mathsf{TermStatus}(\omega),H(\omega),H(\beta_o),\\
 &H(\mathsf{CommitFoot}(o,\omega)),
 H(\mathsf{ReturnFoot}(o,\omega))\rangle .
\end{aligned}
\]
Define
\begin{equation}
\begin{split}
\mathsf{RolloverReady}_s&(c,\eta,\eta^+,o,\omega)\\
\iff{}&
  \eta=\eta_c
  \land\eta^+>\eta\\
 &\land\eta^+=\mathsf{AuthoritativeEpoch}_s(g,\chi_c)\\
 &\land\mathsf{FenceInstalled}_s(g,\chi_c,\eta^+)
  \land\sigma_c=\mcode{closing}
  \\
 &\land\mathsf{ValidTerminal}_s(o,\omega)
  \land o\in\mathsf{Ops}(c,\eta)\\
 &\land\mathsf{Authority}(c)=\mathsf{SettlementAuthority}(o)\\
 &\land\mathsf{TermDecision}_s[oid]=\bot
  \land\mathsf{Decision}_s[K_c]=\bot\\
 &\land\mathsf{SafeA}_s(c,\eta)=\mathsf{Avail}_s(c,\eta)\\
 &\land\forall x\in\mathsf{Transfers}(c,\eta):
     \delta[x]\in\{\mcode{accepted},\mcode{cancelled}\}\\
 &\land\forall o'\in\mathsf{Ops}(c,\eta)\setminus\{o\}:\\
 &\qquad
  \sigma_{o'}\in\{\mcode{committed},
  \mcode{verified-aborted},\mcode{forfeited}\}.
\end{split}
\label{eq:rollover-ready}
\end{equation}
Then the fused action is defined, not merely implemented by convention:
\begin{equation}
\begin{split}
&\mathsf{TerminalRollover}(o,\omega,c,\eta,\eta^+,c^+,s,s')
\iff\\
&\quad \mathsf{RolloverReady}_s(c,\eta,\eta^+,o,\omega)
\land\mathsf{ReissueTarget}(c,\eta,\eta^+,c^+,s,s')\\
&\quad \land\mathsf{Decision}_{s'}[K_c]=(\eta^+,c^+)\\
&\quad \land\mathsf{TermDecision}_{s'}[oid]
     =\mathsf{TermRecord}(o,\omega)\\
&\quad \land\sigma'_o=\mathsf{TermStatus}(\omega)
  \land\sigma'_c=\mcode{closed}\\
&\quad \land
 \bigl(\omega=\pi_{\mathrm{local}}\Rightarrow
   \Verify_{\mathrm{local}}(\pi_{\mathrm{local}},o,s,s')\bigr)\\
&\quad \land\
 \begin{cases}
 \mathsf{SafeA}_s(c,\eta):
      \mathsf A(c)\rightarrow\mathsf A(c^+),\\
 \mathsf{CommitFoot}(o,\omega):
      \mathsf Q(o)\rightarrow\mathsf C(o),\\
 \mathsf{ReturnFoot}(o,\omega):
      \mathsf R(o)\text{ or }\mathsf Q(o)\rightarrow\mathsf A(c^+),
 \end{cases}\\
&\quad \land\ \mathsf{Closed}_{s'}(c,\eta,\eta^+),
\\[-1mm]
&\quad \land\mathsf{Frame}_{o,c,c^+}(s,s').
\end{split}
\label{eq:rollover-action}
\end{equation}
\(\mathsf{Frame}_{o,c,c^+}(s,s')\) keeps every state component unchanged except
the three displayed ownership moves, \(o\)'s terminal record/status, \(c\)'s
closure, creation/current installation of \(c^+\), and the two displayed
write-once decisions. Every branch sets the named outbox's outcome field
exactly to the value fixed by \(\mathsf{TermRecord}\) and appends exactly its
content-addressed audit node. The \(\pi_{\mathrm{local}}\) branch, and no other
branch, additionally replaces the outbox dispatch gate by its absorbing
\mcode{local-abort} marker.
The three moved sets are pairwise disjoint:
\(\mathsf{SafeA}_s(c,\eta)\) is available before the action, whereas
\(\mathsf{CommitFoot}\uplus\mathsf{ReturnFoot}=X_o\) is reserved or
quarantined. The commit also persists the canonical \(D_o\)/remainder
commitments when \(\omega=r^+\), marks \(o\) terminal, closes \(c\), and
installs \(c^+\) as the active generation for the same slot. Exact replay
stutters; a changed witness, partition, epoch, or successor is rejected.
The authority conjunct in Equation~\eqref{eq:rollover-ready}, together with
\(\mathsf{ReissueTarget}\), supplies the settlement-authority clause of
\(\mathsf{ReturnTarget}\) for every returned atom without defining either
predicate circularly.

Reuse is a new ownership transition: under the state-transition predicate in
Equation~\eqref{eq:reissue-target}, it atomically moves exactly
\(\mathsf{Reusable}_s(c,\eta,\eta^+)\) from its proved available records into a fresh
same-root, same-authority, non-widening lineage successor
\((c^+,\eta^+)\). It never changes an old
\(\mathsf C\) or \(\mathsf F\) record and never reuses an accepted transfer.
The gateway rejects a new old-epoch acceptance after the fence while
idempotently revealing an outcome that crossed before it. An epoch check only
at the budget store is insufficient if a stale worker holds direct receiver
credentials.

\subsection{Fan-In and Return}

A task join has three safe choices:

\begin{enumerate}
  \item retain separate branch credits and select one for each new effect;
  \item atomically join disjoint, operation-free footprints under one
        authoritative store; or
  \item transfer each branch's remainder through sealed handoffs to a join
        store, then join after acknowledgment.
\end{enumerate}

It may not add reported numeric balances from branches. Each return uses a
stable transfer ID and consumes the child footprint. Duplicate completion or
return messages stutter. A canceled branch with an indeterminate operation
can return only its unaffected available footprint; the obligation remains
charged to the root.

\subsection{Protocol Adapters}

An adapter maps concrete events to abstract transitions:
\[
\begin{aligned}
 \alpha:\mathsf{ConcreteEvent}\rightarrow
 \{&\Issue,\Split,\Delegate,\mathsf{AcceptTransfer},
      \mathsf{CancelTransfer},\\
   &\Reserve,\Dispatch,\mathsf{GatewayAccept},\mathsf{EffectOccurs},\\
   &\Settle,\ReturnAct,\Retire,\Fence,\mathsf{Reissue},\RolloverAct,
     \JoinAct,\mathsf{stutter}\}.
\end{aligned}
\]
For MCP, the operation binding travels with the tool call or trusted proxy
context. For A2A, a task delegation carries the child credit certificate and
the completion record carries the stable operation/transfer IDs. For HTTP, the
gateway maps \(\kappa\) to an idempotency or conditional-write primitive.
Protocol authentication can use AIP/AAT or conventional workload identity;
the budget protocol does not define a new bearer-token format.

An adapter is conforming only if it preserves root, logical slot, credit
generation, lineage and footprint commitments, vector and maximum charge, selector schema
version, label digests, epoch, canonical operation key, complete acceptance
digest \(\beta_o\), normalized effect digest, receiver, joint witness, terminal-
evidence binding, and terminal committed/remainder footprint commitments.
Unknown fields that affect effect semantics fail closed. A textual Agent
message never substitutes for a trusted adapter event.

\subsection{Stable Decisions and Reason Codes}

Every rejection or uncertainty is explicit:

\begin{table*}[t]
\centering
\small
\begin{tabularx}{\textwidth}{L{0.25\textwidth}YY}
\toprule
\textbf{Reason} & \textbf{Meaning} & \textbf{Required response} \\
\midrule
\mcode{INSUFFICIENT\_}\allowbreak\mcode{LOCAL\_ESCROW} &
Branch lacks exclusive rights &
Coordinate transfer or refuse; never read a stale global total \\
\mcode{LINEAGE\_ALIAS} & Join inputs overlap & Reject and audit \\
\mcode{STALE\_CREDIT\_EPOCH} & Message targets an old generation &
Return prior outcome or reject; do not execute \\
\mcode{OPERATION\_}\allowbreak\mcode{BINDING\_MISMATCH} &
Replay changed effect/vector/receiver &
Reject as conflicting replay \\
\mcode{TERMINAL\_EVIDENCE\_}\allowbreak\mcode{BINDING\_MISMATCH} &
Success receipt or no-effect certificate does not prove this operation &
Retain reservation and investigate \\
\mcode{EFFECT\_INDETERMINATE} & Dispatch may have crossed boundary &
Keep quarantined; query or forfeit \\
\mcode{NO\_SAFE\_RECLAIM\_PROOF} & No authoritative non-acceptance/no-effect
proof under the installed fence &
Retain \(\mathsf Q\) or \(\mathsf F\); settle, or fence and query for proof \\
\mcode{ADAPTER\_}\allowbreak\mcode{REFINEMENT\_FAILURE} &
Concrete event lacks abstract binding &
Fail closed at boundary \\
\bottomrule
\end{tabularx}
\caption{Stable decisions make conservative failure observable. None of the
uncertain cases is silently converted into available budget.}
\label{tab:reasons}
\end{table*}

\section{Safety Results: Conservation and Recovery}
\label{sec:safety}

This section proves safety for the abstract machine in
Section~\ref{sec:formal} and states the additional obligations needed to lift
the results to a concrete implementation.  The proofs are prefix properties:
they hold after every transition, including prefixes that end at a crash,
partition, lost response, or permanently indeterminate effect.  They do not
assume eventual message delivery.

\subsection{Proof Setup}

Write \(s\xrightarrow{a}s'\) when abstract action \(a\) takes state \(s\) to
\(s'\), and let \(\mathsf{Reach}\) be the least set containing the initial
state and closed under the transition relation.  For a root \(g\), resource
dimension \(k\), and atom \(u\in U_g\), define
\[
  \chi_g(u)_k =
  \begin{cases}
    1,&u=(g,k,j)\text{ for some }j,\\
    0,&\text{otherwise}.
  \end{cases}
\]
For a location class \(Z\in\{\mathsf A,\mathsf T,\mathsf R,\mathsf Q,
\mathsf C,\mathsf F\}\), write \(\mathcal Z_s\) for its raw tagged
ownership relation and define
\[
  \mathbf Z_g(s)=
  \sum_{u\in\mathsf{supp}(\mathcal Z_s)}\chi_g(u).
\]
This definition does not presume that the six relations are disjoint or
total. Distinct competing tags within one relation or a cross-relation
overlap remain visible in the candidate multiset \(\mathsf{Own}_s(u)\) from
Section~\ref{sec:formal}; byte-identical storage replicas abstract to one
logical relation member.

The transition relation treats each stable operation identifier as a
single-assignment decision key.  Repeating an already accepted request with
identical immutable fields is a stutter step; repeating it with different
fields is rejected and is also a stutter at the abstract level.  A crash
before an atomic durable commit is a stutter.  A crash after that commit
observes the committed successor state.  These conventions make crash cuts
explicit without assuming that a process executes a multi-record update
exactly once.

\begin{definition}[Atom-preserving transition]
\label{def:atom-preserving}
A transition \(s\xrightarrow{a}s'\) is atom preserving for root \(g\) if
there is a set \(Y\subseteq U_g\) such that
\[
 \forall u\notin Y:\mathsf{Own}_{s'}(u)=\mathsf{Own}_s(u),
\]
and, for each \(u\in Y\), the action guard requires
\(\mathsf{Own}_s(u)=\{\ell_{\mathrm{src}}\}\) and the authoritative commit
atomically removes that source pair and inserts exactly one permitted
destination pair. It neither duplicates nor deletes an input atom.
\end{definition}

Every non-issuance action in the abstract machine is atom preserving.
\(\Split\), \(\JoinAct\), ordinary successor-generation issuance, and
\(\RolloverAct\) change the credit tag inside \(\mathcal A\);
\(\Delegate\), acceptance, and cancellation move atoms
\(\mathsf A\rightarrow\mathsf T\rightarrow\mathsf A\);
\(\Reserve\), \(\Dispatch\), \(\Settle\), \(\ReturnAct\), and \(\Retire\) move
atoms among the six bucket classes. \(\RolloverAct\) additionally applies one
terminal partition while installing the successor, but its moved footprints
are disjoint. \(\Fence\) changes an epoch but not an ownership relation. Later
evidence for a forfeited operation annotates its outcome while
\(\mathsf F(o)\) remains absorbing.

\subsection{Unique Ownership and Exact Conservation}

\begin{lemma}[Ownership-partition preservation]
\label{lem:ownership-partition}
Under Assumptions~\ref{ass:auth}--\ref{ass:escrow}, every reachable state
\(s\in\mathsf{Reach}\) has exactly one raw ownership tag for each
\(u\in U_g\):
\[
  \forall g\;\forall u\in U_g,\qquad
  |\mathsf{Own}_s(u)|=1 .
\]
\end{lemma}

\begin{proof}
We induct on the length of the execution prefix.  At issuance, authenticated
root uniqueness gives one \(U_g\), and \(\Issue\) assigns every atom in it to
the root credit.  Thus the property holds initially.

Assume it holds in \(s\).  Consider \(s\xrightarrow{a}s'\).  A rejected,
duplicate, or pre-commit-crash action stutters, so the induction hypothesis
applies directly.  For a committed action, Assumption~\ref{ass:durable}
excludes a visible state containing only a proper subset of its local moves,
while the protocol guard checks that every input has a singleton source tag.
Stable local decision keys and the transfer decision register in
Assumption~\ref{ass:escrow} prevent replay or competing terminal decisions
from applying a second move. Each transition is atom preserving:

\begin{itemize}
  \item \(\Split\) removes pairwise-disjoint \(X_i\) from one parent before
        assigning each \(X_i\) to one child.  No atom is in two \(X_i\).
  \item A delegation seal removes each source pair and inserts one
        \(\mathsf T(x)\) pair. Acceptance and cancellation compete on the
        monotone register in Equation~\eqref{eq:transfer-decision}. An accepted
        decision completes the uniquely named destination import; a cancelled
        decision restores the source. Neither terminal value enables the
        other, and a \mcode{pending} transfer remains solely in \(\mathsf T\).
  \item \(\Reserve\) and \(\Dispatch\) respectively relabel
        \(\mathsf A(c)\) as \(\mathsf R(o)\) and \(\mathsf R(o)\) as
        \(\mathsf Q(o)\).
  \item Successful \(\Settle\) partitions \(X_o\) into the disjoint,
        version-determined footprints
        \(D_o=\mathsf{Take}_{\nu_{\mathsf{sel}}}
        (X_o,\mathbf d_{r^+})\) and \(X_o\setminus D_o\), assigning the former to
        \(\mathsf C(o)\) and the latter, when nonempty, to the one
        decision-bound available credit satisfying
        Equation~\eqref{eq:return-target}.
        Verified abort and retirement each select exactly one successor for
        every input atom; outcome annotation after retirement changes no
        ownership relation.
  \item \(\RolloverAct\) partitions the final operation into the disjoint
        \(\mathsf{CommitFoot}(o,\omega)\) and
        \(\mathsf{ReturnFoot}(o,\omega)\).  These are disjoint from
        \(\mathsf{SafeA}_s(c,\eta)\), whose atoms are available before the
        action.  Equation~\eqref{eq:rollover-action} sends every atom in the
        three sets to exactly one of \(\mathsf C(o)\) or
        \(\mathsf A(c^+)\) in the same authoritative commit.
  \item An admissible \(\JoinAct\) consumes all pairwise-disjoint input
        footprints in one authoritative action, then assigns their union to
        one output credit.
  \item \(\Fence\), gateway acceptance, effect occurrence, and evidence
        annotation do not change an ownership relation. Successor-generation
        issuance moves only the singleton available atoms in
        \(\mathsf{Reusable}_s(c,\eta,\eta^+)\) to one fresh available tag
        satisfying
        Equation~\eqref{eq:reissue-target}.
\end{itemize}

Hence each changed atom loses exactly one source tag and gains exactly one
destination tag in the same abstract commit, while all other candidate
multisets are unchanged. The property therefore holds in \(s'\). Only now is
\(\mathsf{loc}_s(u)\) well-defined as the unique element of
\(\mathsf{Own}_s(u)\); uniqueness was not supplied by a function type.
\end{proof}

\begin{theorem}[Exact per-root vector conservation]
\label{thm:exact-conservation}
Under Assumptions~\ref{ass:auth}--\ref{ass:escrow}, every reachable state and
every root grant satisfy, componentwise,
\[
 \mathbf A_g(s)+\mathbf T_g(s)+\mathbf R_g(s)+\mathbf Q_g(s)
 +\mathbf C_g(s)+\mathbf F_g(s)=\mathbf B_g.
\]
All six vectors remain in \(\NatVec\).
\end{theorem}

\begin{proof}
Fix \(g\) and \(k\).  By Lemma~\ref{lem:ownership-partition}, every atom
\((g,k,j)\in U_g\) contributes exactly one to exactly one of the six bucket
coordinates.  There are \(B_{g,k}\) such atoms by
Definition~\ref{def:universe}.  Summing their indicators gives
\[
 A_{g,k}+T_{g,k}+R_{g,k}+Q_{g,k}+C_{g,k}+F_{g,k}=B_{g,k}.
\]
The argument holds independently for every \(k\).  Each coordinate is a
cardinality and is therefore a non-negative integer.  Notice that the proof
does not require a terminal response: an unresolved dispatch contributes to
\(\mathbf Q_g\), and permanently stranded authority contributes to
\(\mathbf F_g\).
\end{proof}

\begin{corollary}[No minting by retry, recovery, or join]
\label{cor:no-mint}
Under Assumptions~\ref{ass:auth}--\ref{ass:escrow}, no finite sequence of
duplicate messages, process restarts, reconciliations, or lineage-disjoint
joins can increase the sum of spendable and obligated capacity for a root
beyond \(\mathbf B_g\).
\end{corollary}

\begin{proof}
Retries and recovery actions are either stutters or atom-preserving moves.
Joins union disjoint footprints after consuming their inputs.  Theorem
\ref{thm:exact-conservation} therefore applies to every resulting prefix.
\end{proof}

\subsection{Lineage and Descendant Non-Amplification}

For the next results, define the \emph{responsible lineage} of an atom to be
the credit lineage that currently holds it or the lineage embedded in its
transfer or operation binding.  Terminal locations retain that binding; a
committed or forfeited atom does not lose provenance merely because it is no
longer spendable.

\begin{lemma}[Lineage provenance]
\label{lem:lineage-provenance}
Under Assumptions~\ref{ass:auth}--\ref{ass:escrow}, for every reachable
state, the responsible lineage of each \(u\in U_g\)
contains a path from the unique root credit to its current credit, transfer,
or operation.  Along that path the atom is moved, never copied, and every
descendant label applicable to \(u\) restricts its ancestor labels.
\end{lemma}

\begin{proof}
The root path is trivial after \(\Issue\).  Induct on transitions.
\(\Split\) adds one parent edge for each moved atom; pairwise disjointness
selects one child.  \(\Delegate\) records the source generation in its sealed
certificate, and destination acceptance creates one child from that
certificate.  \(\Reserve\) copies the credit identifier, epoch, and label
binding into the operation while moving the atom out of availability.
  \(\Dispatch\), \(\Retire\), and reconciliation preserve that operation
  binding.  \(\Settle\) and \(\ReturnAct\) can place a returned atom only in a
  same-root successor that is current at the atomic commit and satisfies
  Equation~\eqref{eq:return-target}; cross-store or ancestor return must use a
  later sealed transfer. Ordinary \(\mathsf{Reissue}\) creates only the
  authenticated same-root, same-authority successor in
  Equation~\eqref{eq:reissue-target}, with
  \(\lambda_{c^+}\restrict\lambda_c\).  \(\RolloverAct\) records that same
  authenticated reissue edge atomically with the final terminal decision:
  committed atoms end at \(o\), while every returned atom and every atom in
  the predecessor's disjoint \(\mathsf{SafeA}\) footprint has the unique
  predecessor \(c\) and ends at \(c^+\). An admissible \(\JoinAct\) requires all
  inputs and its output to have the same root and records all input generations
  but, because their footprints are disjoint, each individual atom has exactly
  one input predecessor.  Assumption~\ref{ass:auth} prevents a
  forged predecessor; Lemma~\ref{lem:ownership-partition} and the singleton
  source guard prevent a second owner. Equations~\eqref{eq:dual-attenuation}
  and \eqref{eq:return-target} supply label restriction at every
  lineage-creating or return step. Transitivity of \(\restrict\) gives
  restriction against all applicable ancestors.
\end{proof}

Fix a finite execution prefix \(\rho\) ending in state \(s\) and a single-root
credit sub-DAG \(W\). Call \(W\) \emph{credit-cut closed along \(\rho\)} when
every authoritative ownership event that moves responsibility for an atom from
a lineage outside \(W\) to one inside \(W\) appends the exact moved footprint
to a monotone authenticated ingress history. Such events include cross-cut
split outputs, accepted sealed transfers, returned remainders and verified
aborts, successor reissues, and every outside input consumed by a join. A
task-dependency edge or balance report alone is not an ingress. Let
\(\mathsf{In}_{\rho}(W)\) be the set union of those event footprints,
deduplicated by atom identity if the same right later re-enters, and let
\(\mathsf{Held}_s(W)\) be the atoms whose current available, reserved,
indeterminate, committed, or forfeited responsibility lies in \(W\); atoms
sealed for transfer out of \(W\) are excluded.

\begin{theorem}[Descendant non-amplification]
\label{thm:descendant}
Under Assumptions~\ref{ass:auth}--\ref{ass:escrow}, for every finite execution
prefix \(\rho\) ending in reachable state \(s\) and every sub-DAG \(W\) that is
credit-cut closed along \(\rho\),
\[
 \mathsf{Held}_s(W)\subseteq\mathsf{In}_{\rho}(W),
 \qquad
 \mathsf{vec}(\mathsf{Held}_s(W))
 \preceq \mathsf{vec}(\mathsf{In}_{\rho}(W)).
\]
If \(W\) contains the root credit, take
\(\mathsf{In}_{\rho}(W)=U_g\).
\end{theorem}

\begin{proof}
Induct on the transitions of \(\rho\), reasoning per atom. Initially a
non-root \(W\) holds no atom, whereas a root-containing \(W\) holds a subset of
\(U_g\). A transition whose source and destination responsible lineages are
both inside \(W\) merely relabels the atom. A transition to a lineage outside
\(W\), including an egress seal, removes it from
\(\mathsf{Held}(W)\). A transition from outside to inside can add only an atom
in its exact authenticated footprint; credit-cut closure appends that footprint
to the history defining \(\mathsf{In}_{\rho}(W)\) in the same logical move.
This case covers cross-cut split, transfer, return, ordinary reissue,
\(\RolloverAct\), and join rather than assuming any of them is internal.
Transitions on the outside and stable
decision replays do not change either side; the latter stutter because their
decision keys are unique. Lemma~\ref{lem:ownership-partition} excludes a second
simultaneous owner in every case. Hence the set inclusion is preserved.
Applying \(\mathsf{vec}\), which is monotone under set inclusion, gives the
componentwise inequality.
\end{proof}

The theorem is deliberately about the credit-lineage DAG, not the task DAG.
Adding a dependency edge, reference, model message, or shared specialist does
not change \(\mathsf{In}_{\rho}(W)\) and therefore grants no capacity.

\subsection{Stable Settlement and Continuous Effect Coverage}

\begin{lemma}[Terminal-decision stability]
\label{lem:terminal-stability}
Under Assumptions~\ref{ass:durable}, \ref{ass:escrow}, and
\ref{ass:receipt}, for any operation \(o\), at most one terminal settlement
decision can change ownership of \(X_o\).  After that decision, an identical
evidence item is a stutter and conflicting terminal evidence is rejected.
\end{lemma}

\begin{proof}
The operation row and its terminal-decision key are updated in the same local
transaction as the ownership move.  By Assumption~\ref{ass:durable}, two
concurrent settlement attempts are serialized: the first committed attempt
either records a terminal decision or leaves the old state unchanged.  Every
later attempt observes that decision. An identical witness returns the
recorded result; any byte-distinct witness conflicts with the stored
\(\mathsf{TermRecord}\) digest and is rejected before an ownership move.
Before a first decision can be created, an external variant must pass its
type-matched \(\Verify_{\mathrm{term}}\) verifier and a local variant must pass
\(\Verify_{\mathrm{local}}(\pi_{\mathrm{local}},o,s,s')\); the closed sum
admits no unverified fourth case.
\(\RolloverAct\) uses this same operation terminal-decision key while
atomically adding the independent normalized reissue decision; it is not a
second settlement namespace. A crash
before the first commit leaves no ownership move;
a crash after it leaves the terminal key and ownership move together.
Consequently there is no reachable prefix in which two terminal settlement
transitions both move \(X_o\).
\end{proof}

\begin{theorem}[At-most-once authoritative settlement]
\label{thm:at-most-once-settlement}
Under Assumptions~\ref{ass:auth}, \ref{ass:durable}, \ref{ass:escrow},
\ref{ass:gateway}, and \ref{ass:receipt}, each canonical operation key
\(\kappa_o\) has at most one non-conflicting authoritative outcome, and its
charged footprint is applied to \(\mathbf C_g\) at most once.
\end{theorem}

\begin{proof}
Assumption~\ref{ass:gateway} maps all gateway retries under \(\kappa_o\) to
one logical receiver outcome. Assumption~\ref{ass:receipt} binds either
admissible evidence variant to \(oid\), \(\kappa_o\), the complete acceptance
digest \(\beta_o\), root, logical slot, lineage, epoch, effect digest, and
receiver. A
success receipt additionally binds the instance and charge; a no-effect
certificate instead binds the higher installed fence for the same logical
slot and is mutually
exclusive with success. Lemma~\ref{lem:terminal-stability} allows at most one
variant to move ownership. Thus the canonical footprint
\(D_o=\mathsf{Take}_{\nu_{\mathsf{sel}}}
(X_o,\mathbf d_{r^+})\) enters \(\mathsf C(o)\) at most once. Network messages
and responses may be
delivered more than once; the theorem concerns authoritative settlement, not
exactly-once transport.
\end{proof}

\begin{lemma}[Ambiguity remains charged]
\label{lem:ambiguity-charged}
Suppose \(\Dispatch(o)\) has committed. Every atom of \(X_o\) remains in
\(\mathsf Q(o)\) until \(\Settle\), \(\ReturnAct\), \(\Retire\), or the
corresponding terminal branch of \(\RolloverAct\). Before authenticated
successful or verified-no-effect terminal evidence, it is either retained in
\(\mathsf Q(o)\) or conservatively moved to absorbing \(\mathsf F(o)\). A
timeout, process death, lease expiry, partition, missing response, or epoch
increment alone cannot move any atom of \(X_o\) to \(\mathsf A\).
\end{lemma}

\begin{proof}
After the durable dispatch cut, the abstract location is
\(\mathsf Q(o)\). Inspecting the transition relation, only \(\Settle\),
\(\ReturnAct\), \(\Retire\), and \(\RolloverAct\) can remove an atom from that
location. Authenticated successful settlement partitions it into
\(\mathsf C\) and an authenticated remainder in \(\mathsf A\); a fenced
verified-no-effect decision returns it to \(\mathsf A\); retirement moves it
to \(\mathsf F\). The corresponding \(\RolloverAct\) branches apply exactly
the same terminal partitions while atomically installing a successor.
Nonterminal reconciliation observations stutter.
Timeout, failure suspicion, and lease expiry are not ownership transitions.
\(\Fence\) changes only the accepted epoch.  Therefore ambiguity cannot
create available capacity.
\end{proof}

\begin{theorem}[Continuous injective coverage of mediated effect instances]
\label{thm:effect-coverage}
Under Assumptions~\ref{ass:auth}, \ref{ass:durable}, \ref{ass:escrow},
\ref{ass:effect}, \ref{ass:gateway}, \ref{ass:charge}, \ref{ass:receipt},
and \ref{ass:mediation}, every reachable state satisfies
\[
  \forall \iota\in\mathsf{Obs}(s),\qquad
  \exists!\,o\in\Operations:\Covered(\iota,o,s),
\]
and no one-effect-profile operation covers two distinct instances.
\end{theorem}

\begin{proof}
Induct on abstract transitions, treating \(\mathsf{Obs}\) as a monotone state
variable. Initially it is empty. Only \(\mathsf{EffectOccurs}(o,\iota)\) can
append an instance. Its guard requires a prior gateway acceptance of the
immutable digest \(\beta_o\) in Equation~\eqref{eq:accept-binding}, together
with \(\kappa_o\). On first acceptance, the action guard requires a matching
durable dispatchable marker and the full footprint in \(\mathsf Q(o)\), or
conservatively in \(\mathsf F(o)\) after retirement. An exact replay after a
later terminal move reads the existing accepted register and stutters; it
cannot create a new acceptance. The accepted witness also establishes
\(\mathsf{RequiredCap}(e,\chi)\preceq\mathbf u_o\), and the
\((g,\chi_c)\) committed inside \(\beta_o\) selects the corresponding
current-generation register and gateway fence. Thus the three
non-recursive clauses of \(\Covered(\iota,o,s)\) hold when \(\iota\) is
appended.

Assumption~\ref{ass:gateway} allocates a fresh \(iid\) and acceptance sequence
to at most one instance for the receiver-scoped key. The canonical
domain-separated derivation of \(\kappa\), stable operation decision, and
immutable binding prevent a different operation from acquiring that key
within the stated collision-resistance assumption; the one-effect guard
prevents the same operation from appending a second instance.
This proves both directions of injectivity. Two separate operations may still
request byte-identical descriptors; their keys and \(iid\) values make them
distinct budgeted instances.

Subsequent ownership actions preserve coverage. Before terminal evidence,
Lemma~\ref{lem:ambiguity-charged} retains all of \(X_o\) in \(\mathsf Q\), or
\(\Retire\) moves it wholly to absorbing \(\mathsf F\). A complete committed
receipt identifies the unique \(\iota\), moves the version-determined footprint
\[
 D_o=\mathsf{Take}_{\nu_{\mathsf{sel}}}
      (X_o,\mathsf{ActualCharge}(\iota))
\]
to \(\mathsf C\), and returns only the complement to a target satisfying
Equation~\eqref{eq:return-target}. A verified-no-effect return is disabled
after gateway acceptance linearizes by terminal-evidence non-equivocation and the
non-acceptance premise of Assumption~\ref{ass:receipt}. If
\(\RolloverAct\) is applied to an operation covering \(\iota\), its
\(\mathsf{ValidTerminal}\) evidence can therefore only be that same complete
success receipt: the no-effect branch contradicts prior acceptance and the
local-abort branch contradicts durable dispatch. Equation~\eqref{eq:rollover-action}
moves the same \(D_o\) to \(\mathsf C\), moves only its complement to the
successor, and moves the predecessor's disjoint \(\mathsf{SafeA}\) footprint
without changing \(o\)'s coverage. Evidence arriving
after retirement annotates the
record without moving \(\mathsf F\); the operation remains in
\(\mathsf{Risk}_g\) at its maximum rather than entering
\(\mathsf{Settled}_g\). All other transitions preserve \(\beta_o\) and
\(\mathsf{Obs}\). Hence coverage and injectivity hold for every prefix.
\end{proof}

\begin{theorem}[Root-bounded mediated external exposure]
\label{thm:external-exposure}
Under the assumptions of Theorem~\ref{thm:effect-coverage}, every reachable
state satisfies Equation~\eqref{eq:external-bound}:
\[
 \sum_{\iota\in\mathsf{Settled}_g(s)}
   \mathsf{ActualCharge}(\iota)
 +\sum_{o\in\mathsf{Risk}_g(s)}\mathbf u_o
 \preceq \mathbf C_g(s)+\mathbf Q_g(s)+\mathbf F_g(s)
 \preceq\mathbf B_g .
\]
\end{theorem}

\begin{proof}
For every settled instance, success-receipt completeness gives the distinct
charged footprint
\(D_o=\mathsf{Take}_{\nu_{\mathsf{sel}}}
(X_o,\mathsf{ActualCharge}(\iota))\) in \(\mathcal C_s\). For every risk
operation, irrespective of a
later evidence annotation, its entire maximum footprint \(X_o\) is present in
\(\mathcal Q_s\) or \(\mathcal F_s\), so
\(\mathbf u_o=\mathsf{vec}(X_o)\).  By definition an
\(\mathsf F\)-resident late success is counted only in this second,
conservative term. The coverage
injectivity theorem and ownership partition make all of these charged and
at-risk footprints pairwise disjoint; no atom contributes to both sums.
Their vector sum is therefore bounded by the raw supports of
\(\mathcal C_s,\mathcal Q_s,\mathcal F_s\). Exact conservation and
non-negativity then give
\(\mathbf C_g+\mathbf Q_g+\mathbf F_g\preceq\mathbf B_g\).
The first sum reports actual settled charge; the second deliberately charges
uncertainty and forfeiture at their enforced maxima.
\end{proof}

\subsection{Epoch Fencing, Late Completion, and Safe Reclaim}

\begin{lemma}[No stale resurrection]
\label{lem:no-stale-resurrection}
Under Assumptions~\ref{ass:auth}, \ref{ass:durable}, and
\ref{ass:gateway}, if the authoritative epoch for credit \(c\) advances from
\(\eta\) to \(\eta'>\eta\), no later transition can make \(\eta\) current
again, and no
new protected effect under \((c,\eta)\) is accepted after the corresponding
gateway fence is installed.
\end{lemma}

\begin{proof}
Epoch updates are strictly increasing by the \(\Fence\) precondition, so the
credit authority cannot restore \(\eta\).  Assumption~\ref{ass:gateway}
requires the protected gateway to compare an incoming generation against the
installed fence.  It either rejects a previously unseen stale request or
returns the outcome already recorded for its stable key.  Neither behavior
creates a new old-epoch effect.  Reordered or replayed messages do not alter
the monotone fence.
\end{proof}

\begin{theorem}[Late-completion safety]
\label{thm:late-completion}
Assume \ref{ass:auth}, \ref{ass:durable}, \ref{ass:escrow},
\ref{ass:effect}, \ref{ass:gateway}, \ref{ass:receipt}, and
\ref{ass:mediation}.  Let a mediated operation \(o\) be dispatched under epoch
\(\eta\), let its caller time out, and let a controller later install an
authoritative successor epoch \(\eta^+>\eta\). If
\(\mathsf{GatewayAccept}(o)\) linearized before the fence, all
of \(X_o\) remains charged in \(\mathsf Q(o)\) or \(\mathsf F(o)\) until
complete terminal evidence. A pre-retirement success moves its exact
actual-charge footprint \(D_o\) to \(\mathsf C(o)\); only its authenticated
complement may return to \(\mathsf A\), after the same terminal decision
excludes any further instance or charge under the old binding. A
post-retirement result leaves all of \(X_o\) in absorbing \(\mathsf F(o)\).
If no acceptance linearized before the fence, the gateway cannot newly accept
the old epoch afterward.
Consequently neither ordinary generation closure and successor issuance under
Equations~\eqref{eq:closed-generation}--\eqref{eq:reissue-target} nor fused
terminal rollover under
Equations~\eqref{eq:rollover-ready}--\eqref{eq:rollover-action} can make one
atom simultaneously fund a late old-epoch effect and a new-generation effect.
\end{theorem}

\begin{proof}
There are two cases.

\emph{Acceptance linearized before the fence.} Its reservation was already in
\(\mathsf Q(o)\) by the durable-before-exposure order. Timeout and epoch
advance do not refund it by Lemma~\ref{lem:ambiguity-charged}. A complete
pre-retirement success moves
\(D_o=\mathsf{Take}_{\nu_{\mathsf{sel}}}
(X_o,\mathbf d_{r^+})\) permanently to \(\mathsf C(o)\) and may return
\(X_o\setminus D_o\) only in the same write-once terminal decision. Receipt
completeness and the one-effect gateway rule then prohibit a later instance or
additional charge under \((\kappa_o,\beta_o)\), so no returned complement atom
funds the old effect. If retirement commits first, \(\mathsf F\) remains
absorbing and late evidence only annotates it; otherwise the unresolved
footprint stays in \(\mathsf Q(o)\). Assumption~\ref{ass:receipt} disables a
verified-no-effect return after acceptance has linearized, and a fence does not
cancel that acceptance.

\emph{No acceptance linearized before the fence.} Lemma
\ref{lem:no-stale-resurrection} disables a new old-epoch acceptance, so no
protected effect can arise from that binding. If a retry instead observes a
previously recorded acceptance or result, the premise of this case is false
and the execution belongs to the first case.

Equation~\eqref{eq:closed-generation} closes the old generation only after the
fence is installed, every transfer decision is terminal, and every old
operation is committed, verified-aborted, or forfeited.
Equation~\eqref{eq:reusable} then selects only uniquely available,
retained-provenance atoms. Committed and forfeited atoms, and atoms carried by
accepted transfers, are excluded. For the dispatched operation in this
theorem, a verified-aborted atom is selectable only with an authoritative
no-effect certificate protected by the same fence; an authenticated success
remainder is selectable only after its exact charged complement and complete
single-instance terminal result are fixed. An ordinary successor-generation
action moves, rather than copies, that exact set under the normalized
write-once reissue decision.  When the last terminal decision and successor
installation are fused, Equations~\eqref{eq:rollover-ready}--
\eqref{eq:rollover-action} establish the same facts atomically: the canonical
committed and return footprints partition \(X_o\), both are disjoint from the
predecessor's \(\mathsf{SafeA}\) footprint, and exactly the return and
\(\mathsf{SafeA}\) atoms move to the normalized same-slot successor.
Therefore a reclaimed atom can neither be in \(D_o\) nor later acquire charge
under the old binding, and no atom can fund effects in both generations.
\end{proof}

\begin{corollary}[Failure suspicion is never a refund proof]
\label{cor:no-timeout-refund}
Replacing the terminal-evidence premises in
Equations~\eqref{eq:closed-generation}--\eqref{eq:reusable} or
\eqref{eq:rollover-ready}--\eqref{eq:rollover-action} with a timeout,
worker lease expiry, or controller failure detector invalidates
Theorem~\ref{thm:late-completion}.
\end{corollary}

\begin{proof}
In an asynchronous system, an observer cannot distinguish an unexecuted
request from an executed request whose response is delayed.  Refund in the
latter execution makes the same atoms available while the effect exists.
The gateway fence and authoritative outcome are therefore proof obligations,
not implementation conveniences.
\end{proof}

\subsection{Partition Confinement and the Coordination Boundary}

\begin{theorem}[Partition confinement]
\label{thm:partition-confinement}
Under Assumptions~\ref{ass:auth}--\ref{ass:escrow}, consider a network
partition that separates credit authorities into components
\(P_1,\ldots,P_m\) at state \(s_c\). For each fixed root \(g\), define the
component-held footprint
\[
 \mathsf{Held}^g_s(P_i)=
 \{u\in U_g\mid\text{\(u\)'s current \(\mathsf A/\mathsf R/\mathsf Q/
 \mathsf C/\mathsf F\) responsible authority is in \(P_i\)}\}.
\]
Pending \(\mathsf T\) atoms are excluded from this set and assigned separately
to their unique source component below. Define
\[
 E_i^g=\mathsf{Held}^g_{s_c}(P_i)\ \cup\
 \bigcup_{\substack{x:\delta_{s_c}[x]=\mcode{pending}\\
                           \mathsf{src}(x)\in P_i,\ X_x\subseteq U_g}} X_x ,
\]
with set union by atom identity. An already accepted transfer is owned by its
destination and a cancelled transfer by its source; only a pending transfer is
charged to its source component for this cut accounting. If no cross-cut
transfer is accepted during the partition, then for every prefix during that
partition,
\[
  \mathsf{Held}^g_s(P_i)\subseteq E_i^g
  \quad\text{and}\quad
  \sum_i\mathsf{vec}(\mathsf{Held}^g_s(P_i))
  \preceq\mathbf B_g .
\]
In particular, a component can progress locally from its own available
escrow but cannot spend a right assigned to another component.
\end{theorem}

\begin{proof}
The \(E_i^g\) are pairwise disjoint by ownership partition and the unique
pending-transfer source. Induct over partition-time transitions. Local split,
reserve, dispatch, settlement, retirement, return, ordinary reissue,
\(\RolloverAct\), or same-root join
only relabels root-\(g\) atoms already in \(E_i^g\). A pending local transfer can accept
within the same component; a pending cross-cut transfer either stays in
\(\mathsf T\) or cancels back to its source \(E_i^g\), because cross-cut
acceptance is excluded by premise. Stable replays stutter, and a message
reporting a remote balance is not a credit transition. Thus
\(\mathsf{Held}^g_s(P_i)\subseteq E_i^g\) throughout. Disjointness and
\(\bigcup_i E_i^g\subseteq U_g\) yield the aggregate inequality.  A component
that exhausts its local \(\mathsf A\) must wait, reconcile, or refuse; safety
does not imply continued availability.
\end{proof}

\begin{proposition}[Unallocated-unit impossibility]
\label{prop:shared-unit}
In an asynchronous system with a possible partition, no protocol can
simultaneously guarantee both:

\begin{enumerate}
  \item \emph{local availability}: either of two isolated branches must
        approve a request to spend the same currently unallocated unit using
        only its local state; and
  \item \emph{global conservation}: across all executions, at most one
        protected effect may consume that unit.
\end{enumerate}

This holds even with crash-free branches and reliable local storage.
\end{proposition}

\begin{proof}
Let branches \(p\) and \(q\) have identical knowledge that one shared unit
remains, then partition them before either receives another message.
Consider execution \(E_p\), in which only \(p\) receives a spend request.
Local availability requires \(p\) to approve.  Symmetrically, in execution
\(E_q\), in which only \(q\) receives a request, \(q\) must approve.

Now construct \(E_{pq}\), in which both receive their requests during the
same partition.  Up to its decision, \(p\)'s local history in \(E_{pq}\) is
indistinguishable from its history in \(E_p\), so it approves.  Likewise
\(q\)'s history is indistinguishable from \(E_q\), so it approves.  Two
effects consume one unit, contradicting global conservation.  Randomization
cannot provide a universal safety guarantee: any coin outcomes that permit
the respective one-sided approvals form an execution in which both approve.

Exclusive preallocation avoids the contradiction by falsifying the first
premise: at most one partition has the right to approve that unit.  It
provides availability for preallocated escrow, not for a shared unallocated
remainder.
\end{proof}

\subsection{Concrete Refinement Obligations}

Let \(\Sigma\) be the concrete state of branch databases, outboxes, gateway
records, terminal-evidence stores, and epoch authorities.  Let
\(\mathsf{Rep}(\Sigma,s)\) be a representation relation requiring:

\begin{enumerate}
  \item each durable concrete credit, transfer, and operation maps to the
        corresponding abstract object with equal root, logical slot, lineage,
        label, epoch, vector, and immutable identifiers;
  \item each sealed transfer has exactly one logical owner location; its
        monotone accept/cancel decision and crash finalization map to the same
        abstract transfer decision rather than a second import or restoration;
  \item each committed database transaction maps to one abstract transition,
        including a fused terminal-decision-plus-successor transaction mapping
        to \(\RolloverAct\), and each uncommitted or duplicate event maps to
        stutter;
  \item each concrete operation phase and write-once terminal-decision key
        agree with its abstract \(\mathsf R/\mathsf Q/\mathsf C/\mathsf F\)
        ownership; its selector schema and exact committed/remainder footprint
        commitments agree with the abstract deterministic partition and
        disable every conflicting terminal move;
  \item gateway observations and terminal evidence map to the normalized
        effect and operation named by their acceptance digest \(\beta_o\), and
        concrete acceptance/fence records preserve the bound logical slot and
        their authoritative linearization order;
  \item every concrete remainder or no-effect return maps to a same-root,
        current-at-commit, same-authority, label-non-widening target satisfying
        Equation~\eqref{eq:return-target};
  \item every ordinary successor-generation issuance carries a concrete
        \(\mathsf{Closed}\) proof and moves exactly the concrete image of
        \(\mathsf{Reusable}\), excluding committed, forfeited, accepted, and
        unresolved rights, into the sole target selected by its write-once
        reissue decision and satisfying Equation~\eqref{eq:reissue-target};
        the decision, current-generation metadata, and ownership move commit
        atomically;
  \item every fused final terminal decision and successor installation maps
        to \(\RolloverAct\): it satisfies
        Equations~\eqref{eq:rollover-ready}--\eqref{eq:rollover-action},
        persists the canonical terminal partition and normalized reissue
        decision, and moves only the disjoint terminal-return and
        \(\mathsf{SafeA}\) footprints to the same-slot successor; and
  \item compressed footprints map injectively to logical rights, preserve
        disjointness, and have concrete sums equal to the abstract
        \(\mathsf{vec}\) of their atom sets.
\end{enumerate}

For paired finite prefixes
\(\widehat\rho=(\Sigma_0,\ldots,\Sigma_n)\) and
\(\rho=(s_0,\ldots,s_n)\), define
\(\mathsf{Rep}^{*}(\widehat\rho,\rho)\) to require:
(i) \(\mathsf{Rep}(\Sigma_i,s_i)\) for every \(i\);
(ii) every concrete authenticated cross-cut ingress event maps one-for-one to
the corresponding abstract event with the same exact footprint and
atom-identity deduplication, while task edges and balance reports map to no
ingress; and (iii), for any partition-confinement claim, the paired prefixes
preserve the cut state and interval, authority-to-component assignment,
pending-transfer source component, and the predicate that no cross-cut
acceptance commits during the interval. Conditions (ii)--(iii) are the
\emph{prefix-faithfulness obligation} on \(\alpha\).

\begin{theorem}[Conditional trace-lifting criterion]
\label{thm:forward-simulation}
Under Assumptions~\ref{ass:durable} and \ref{ass:adapter}, a concrete runtime
inherits the abstract safety results if it discharges the initial-state,
stepwise, and prefix-faithfulness obligations:
\[
\mathsf{Rep}(\Sigma_0,s_0)
\]
and, for every concrete step \(\Sigma\rightarrow\Sigma'\) and every abstract
state \(s\),
\begin{equation}
\mathsf{Rep}(\Sigma,s)\Longrightarrow
\exists s'.\
 \Bigl(
   [s'=s\land\alpha(\Sigma,\Sigma')=\mathsf{stutter}]
   \lor
   [s\xrightarrow{\alpha(\Sigma,\Sigma')}s']
 \Bigr)
 \land\mathsf{Rep}(\Sigma',s').
\label{eq:stepwise-refinement}
\end{equation}
The disjunction denotes either one enabled abstract action with the same
immutable binding or a stutter step. If these obligations hold, then for every
finite concrete execution
\[
  \Sigma_0\rightarrow\Sigma_1\rightarrow\cdots\rightarrow\Sigma_n
\]
there exists an abstract execution
\[
  s_0\xrightarrow{\alpha_1}s_1\xrightarrow{\alpha_2}\cdots
  \xrightarrow{\alpha_n}s_n
\]
such that
\(\mathsf{Rep}^{*}((\Sigma_0,\ldots,\Sigma_n),(s_0,\ldots,s_n))\).
Here \(s\xrightarrow{\mathsf{stutter}}s\) denotes the explicit abstract
self-loop used for the first disjunct.
\end{theorem}

\begin{proof}
Induct on \(i\). The initial obligation supplies the base case. Applying
Equation~\eqref{eq:stepwise-refinement} to each concrete successor supplies
either a matching enabled abstract action or a stutter and preserves
\(\mathsf{Rep}\). Concatenating those witnesses yields the abstract execution
and pointwise clause (i) of \(\mathsf{Rep}^{*}\); prefix faithfulness supplies
clauses (ii)--(iii).
\end{proof}

The theorem states the exact criterion a concrete runtime must satisfy. A
production implementation establishes
Equation~\eqref{eq:stepwise-refinement} handler by handler, including transfer
accept/cancel, \(W_2/W_3/W_4\) crash cuts, gateway acceptance versus fencing,
effect occurrence, ordinary successor-generation issuance, and
\(\RolloverAct\). The artifact exercises bounded and concrete traces against
these obligations; finite trace agreement does not replace a universal
implementation refinement proof. The SQLite research profile implements and
tests the online first-acceptance guard through a branch-signed \(W_3\) permit,
including missing, forged, stale, crash-recovered, and proof-excluded cases.
This supports the evaluated profile's guard while leaving the universal
adapter obligation and conditional end-to-end theorem subject to a complete
handler-level refinement argument.

\begin{theorem}[Conditional end-to-end dual conservation]
\label{thm:end-to-end}
If Assumptions~\ref{ass:auth}--\ref{ass:mediation} hold and the concrete
runtime satisfies the trace-lifting obligations of
Theorem~\ref{thm:forward-simulation}, then every finite concrete execution
prefix \(\widehat\rho\) has a lifted abstract prefix \(\rho\) for which the
following properties hold and reflect back to \(\widehat\rho\) under the stated
theorem-specific premises:

\begin{enumerate}
  \item exact per-root vector conservation as in
        Equation~\eqref{eq:conservation};
  \item for every sub-DAG \(W\) that is credit-cut closed along the lifted
        prefix \(\rho\), descendant
        non-amplification as in Equation~\eqref{eq:descendant};
  \item at-most-once authoritative settlement per canonical operation;
  \item unique continuous coverage for every gateway-observed protected
        effect as in Equation~\eqref{eq:coverage};
  \item settled actual charge plus conservative outstanding exposure bounded
        by the issued vector as in Equation~\eqref{eq:external-bound};
  \item for every dispatched operation closed and reissued either ordinarily
        under Equations~\eqref{eq:closed-generation}--\eqref{eq:reusable} or
        atomically under
        Equations~\eqref{eq:rollover-ready}--\eqref{eq:rollover-action}, no
        old-epoch effect funded by capacity safely reclaimed into a new
        generation; and
  \item during every partition interval with no accepted cross-cut transfer,
        partition-local spending bounded by the exclusive escrow assigned at
        the cut.
\end{enumerate}
\end{theorem}

\begin{proof}
Theorem~\ref{thm:forward-simulation} maps the concrete prefix to a reachable
abstract prefix.  Apply Theorems~\ref{thm:exact-conservation},
\ref{thm:descendant}, \ref{thm:at-most-once-settlement},
\ref{thm:effect-coverage}, \ref{thm:external-exposure},
\ref{thm:late-completion}, and
\ref{thm:partition-confinement} to that abstract prefix whenever their
explicit sub-DAG, dispatched-operation, and partition-interval premises hold.
The representation relation preserves the identifiers, selector version,
footprint vectors and terminal partition commitments, epochs, effects, and
terminal-evidence bindings mentioned by those properties, so each property
transfers back to the concrete state.
\end{proof}

\subsection{Assumption--Conclusion Boundary}

Table~\ref{tab:safety-boundaries} records the proof surface.  ``Required''
means that the result as stated uses the assumption; it does not mean that
every deployment must implement that assumption in the same way.

\begin{table*}[t]
\centering
\scriptsize
\setlength{\tabcolsep}{3.5pt}
\renewcommand{\arraystretch}{1.12}
\begin{tabularx}{\textwidth}{L{0.165\textwidth}L{0.25\textwidth}YY}
\toprule
\textbf{Result} & \textbf{Required assumptions / premises} &
\textbf{Failure if omitted} & \textbf{Not concluded} \\
\midrule
Ownership partition and exact conservation &
Authenticated unique roots/credits
(Assumption~\ref{ass:auth}); atomic durable moves
(\ref{ass:durable}); singleton-source guards and stable decision keys
(\ref{ass:escrow}) &
Forged issuance, torn moves, or copied inputs can duplicate or lose atoms &
Fair allocation, low stranding, or timely completion \\

Semantic and descendant non-amplification &
Root/lineage authenticity; exclusive transfer; label attenuation and
credit-cut accounting &
A numeric balance may remain bounded while purpose, audience, or ancestry
widens &
Correctness of the policy classifier or optimal task decomposition \\

At-most-once settlement &
Atomic terminal-decision key; immutable operation binding; sound tagged
terminal evidence;
receiver-scoped gateway idempotency &
Duplicate or conflicting evidence can move the same reservation twice, or
one transport retry can create multiple logical outcomes &
Exactly-once delivery or honesty of an arbitrary third-party receiver \\

Effect coverage and external exposure bound &
Authenticated binding, atomic durable ownership, and stable decisions
(\ref{ass:auth}--\ref{ass:escrow}); persist-before-exposure; sound normalization
(\ref{ass:effect}); enforced maximum-charge contract (\ref{ass:charge});
fenced queryable gateway (\ref{ass:gateway}); complete non-equivocating
terminal evidence (\ref{ass:receipt}); complete mediation (\ref{ass:mediation}) &
An effect can bypass its obligation, alias another effect identity, or be
refunded on a false result &
Coverage of covert, direct-credential, or otherwise unmediated effects \\

Late-completion safety and reclaim &
Durable \(\mathsf R\!\rightarrow\!\mathsf Q\) cut; gateway-visible monotone
epoch; authoritative success/no-effect result &
Timeout refund or store-only fencing permits old and new generations to use
one right &
Rollback of an irreversible effect or eventual recovery of stranded rights \\

Partition confinement &
Exclusive preallocation before the cut; no unfenced cross-cut acceptance &
Disconnected branches may spend the same cached remainder &
Availability after local escrow is exhausted \\

Conditional concrete end-to-end result &
Adapter preservation (Assumption~\ref{ass:adapter}), a handler-by-handler
proof of Equation~\eqref{eq:stepwise-refinement}, plus the assumptions of each
selected abstract result &
Concrete fields, write order, or protocol keys may fail to implement the
abstract move proved safe &
Byzantine storage/gateway tolerance, production SLOs, or model-task quality \\
\bottomrule
\end{tabularx}
\caption{Assumption--conclusion boundary for the safety results.  Ledger
integrity alone does not establish external-effect conservation; the latter
requires normalization, complete mediation, and a terminal-evidence-producing fenced
gateway.}
\Description{A four-column table maps each theorem family to its required
assumptions, the failure enabled when an assumption is omitted, and claims
that the theorem explicitly does not make.}
\label{tab:safety-boundaries}
\end{table*}

\subsection{What the Theorems Permit}

The results intentionally permit conservative outcomes.  A right may remain
in \(\mathbf T_g\) because a transfer cannot be canceled safely, in
\(\mathbf Q_g\) because the receiver outcome is unknown, or in
\(\mathbf F_g\) because an administrator retires unresolved authority.  Such
states reduce utilization but preserve the exact equation and keep potential
effects covered.  The theorems also permit coordination during rebalancing,
fan-in, and reclaim.  Proposition~\ref{prop:shared-unit} shows
why removing that coordination while retaining universal partition
availability would change, rather than strengthen, the safety contract.

Conversely, a perfectly conserved internal ledger does not prove that a
provider reported the true cost, that two protocol encodings normalize to the
same effect, or that an agent lacked a bypass credential.  Those are precisely
the boundaries isolated by Assumptions~\ref{ass:effect},
\ref{ass:receipt}, and \ref{ass:mediation}. The artifact therefore tests
state-machine invariants and terminal-evidence contracts as separate evidence
layers. Adapter conformance remains a handler-level trace-lifting obligation;
passing either tested layer is not reported as evidence for that obligation.

\section{Reference Architecture and Implementation Boundary}
\label{sec:implementation}

The abstract machine is intentionally independent of a database, transport,
agent framework, or model provider. This section gives a reference profile
that is implementable without weakening the proof obligations. It is not a
claim that any ordinary message queue or API automatically satisfies those
obligations.

\subsection{Lifecycle Realization}

Figure~\ref{fig:lifecycle} realizes one logical right as a state machine. Solid
arrows show ordinary ownership progress, including both branches of an
authenticated success receipt. Dashed arrows specifically denote
exclusion/no-effect or reversal evidence; they are not a legend for every
authenticated record. There is deliberately no timeout edge from
\(\mathsf Q\) to \(\mathsf A\). The \(\mathsf F\) bucket is absorbing; later
evidence may annotate its outcome but cannot reclassify or make it spendable.

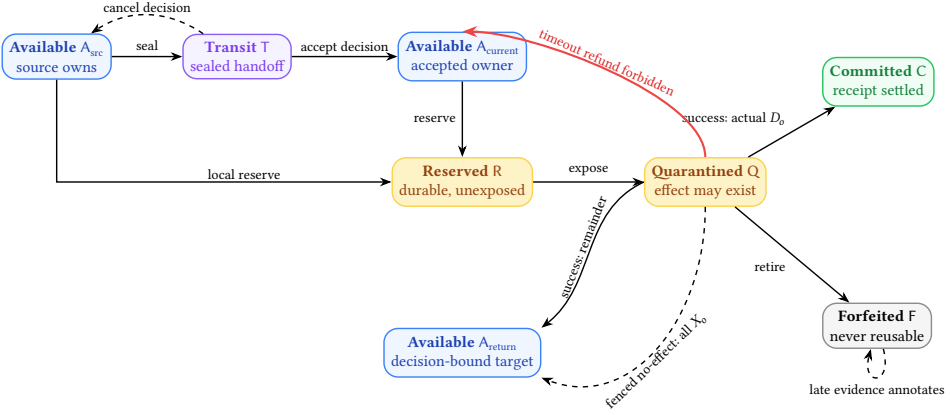
\begin{figure*}[t]
\centering
\begin{floatbox}
\resizebox{0.94\linewidth}{!}{%
\begin{tikzpicture}[
  >=Stealth,
  state/.style={draw, rounded corners=2.8mm, thick, align=center,
    minimum width=2.05cm, minimum height=0.82cm, inner sep=4pt, font=\small},
  flow/.style={->, thick},
  proofedge/.style={->, thick, dashed},
  forbidden/.style={->, very thick, draw=opRed}
]
\node[state, draw=opBlue, fill=opBlueFill, text=opBlueText] (asrc)
  {\textbf{Available} \(\mathsf A_{\rm src}\)\\source owns};
\node[state, draw=opPurple, fill=opPurpleFill, text=opPurpleText,
  right=14mm of asrc] (t)
  {\textbf{Transit} \(\mathsf T\)\\sealed handoff};
\node[state, draw=opBlue, fill=opBlueFill, text=opBlueText,
  right=21mm of t] (acur)
  {\textbf{Available} \(\mathsf A_{\rm current}\)\\accepted owner};
\node[state, draw=opAmber, fill=opAmberFill, text=opAmberText,
  below=15mm of acur] (r)
  {\textbf{Reserved} \(\mathsf R\)\\durable, unexposed};
\node[state, draw=opAmber, fill=opAmberFill, text=opAmberText,
  right=22mm of r] (q)
  {\textbf{Quarantined} \(\mathsf Q\)\\effect may exist};
\node[state, draw=opGreen, fill=opGreenFill, text=opGreenText,
  above right=10mm and 11mm of q] (c)
  {\textbf{Committed} \(\mathsf C\)\\receipt settled};
\node[state, draw=opBlue, fill=opBlueFill, text=opBlueText,
  below=24mm of r] (aret)
  {\textbf{Available} \(\mathsf A_{\rm return}\)\\decision-bound target};
\node[state, draw=black!45, fill=black!4,
  below right=19mm and 11mm of q] (f)
  {\textbf{Forfeited} \(\mathsf F\)\\never reusable};

\draw[flow] (asrc) -- node[above, font=\footnotesize]{seal} (t);
\draw[proofedge, bend right=34] (t) to node[above, font=\footnotesize,
  inner sep=1pt]
  {cancel decision} (asrc);
\draw[flow] (t) -- node[above, font=\footnotesize, inner sep=1pt]
  {accept decision} (acur);
\draw[flow] (asrc.south) |- node[pos=0.78, above, font=\footnotesize,
  inner sep=1pt]{local reserve} (r.west);
\draw[flow] (acur) -- node[left, font=\footnotesize]{reserve} (r);
\draw[flow] (r) -- node[above, font=\footnotesize]{expose} (q);
\draw[flow] (q) -- node[above left, font=\footnotesize]
  {success: actual \(D_o\)} (c);
\draw[flow] (q.west) to[out=205,in=38]
  node[above, sloped, font=\footnotesize, inner sep=1pt]
  {success: remainder}
  (aret.north east);
\draw[proofedge] (q.south) .. controls +(0,-22mm) and +(18mm,-7mm) ..
  node[pos=0.54, below, sloped, font=\footnotesize, inner sep=1pt]
  {fenced no-effect: all \(X_o\)}
  (aret.south east);
\draw[flow] (q) -- node[below left, font=\footnotesize]{retire} (f);
\draw[proofedge, loop below] (f) to node[below, font=\footnotesize]
  {late evidence annotates} (f);
\draw[forbidden] (q.north) .. controls +(0,12mm) and +(14mm,0) ..
  node[pos=0.52, above, sloped, font=\footnotesize, inner sep=1pt,
  text=opRedText]{timeout refund forbidden} (acur.north);
\end{tikzpicture}%
}
\end{floatbox}
\caption{Per-right lifecycle. A timeout can trigger reconciliation but never
changes ownership. Only an authenticated settlement or a fenced proof of
non-execution releases a quarantined right. A success receipt canonically
partitions \(X_o\): \(D_o\) commits, while \(X_o\setminus D_o\) moves to the
separate decision-bound \(\mathsf{ReturnTarget}\) box. A fenced no-effect proof
moves all \(X_o\) to that box along a distinct dashed path. Separate source,
current-owner, and return-target boxes prevent the three ownership roles from
being visually conflated.}
\Description{Eight colored boxes distinguish source Available, Transit,
current-owner Available, Reserved, Quarantined, return-target Available,
Committed, and Forfeited states. Transit either returns to the source after a
cancel decision or moves to the current owner after an accept decision. Solid
arrows show reserve, expose, success settlement, and retirement. The two
success arrows split the footprint: the actual subset moves to Committed and
the remainder moves to the return-target Available box. A separate dashed
curve sends the whole quarantined footprint to that box after a fenced
no-effect proof. A dashed self-loop on Forfeited says late evidence annotates
without changing ownership. A red curved arrow from Quarantined to
current-owner Available is labeled timeout refund forbidden.}
\label{fig:lifecycle}
\end{figure*}

\subsection{Durable Records}

The reference architecture requires an authoritative local store per branch.
The store may itself be replicated, provided that a successful local
transaction is linearizable and durable under Assumption~\ref{ass:durable}.
Table~\ref{tab:schema} gives the minimum logical records. Fields shown as
immutable participate in a canonical digest; mutable phase fields are updated
only by the transition that owns them.

\begin{table*}[t]
\centering
\scriptsize
\setlength{\tabcolsep}{4pt}
\begin{tabularx}{\textwidth}{L{0.18\textwidth}L{0.27\textwidth}YY}
\toprule
\textbf{Record} & \textbf{Immutable identity/binding} &
\textbf{Mutable state} & \textbf{Uniqueness or integrity rule} \\
\midrule
\mcode{root\_grant} & root, issuer, resource schema, issued vector, label
digest & status, current policy version, installed fence epochs & issuer
signature; one issued vector per root generation \\
\mcode{credit} & root, logical slot/reissue-chain ID, credit ID, parent credit
ID, generation, epoch, issuer, credit authority/owner, source-label digest,
lineage and rights commitments & available ranges, lifecycle, current marker &
immutable predecessor identity; at most one current generation per logical
slot/reissue chain; disjoint split siblings may be current concurrently, but
live ranges never overlap \\
\mcode{transfer} & root, transfer ID, source generation, destination, ranges,
label & sealed / accepted / rejected / fenced & unique \((root,transfer\_id)\);
terminal decision immutable \\
\mcode{operation} & root, logical slot, operation ID, source
lineage/rights commitments, epoch, effect digest, receiver, key, \(\beta_o\),
maximum vector, source/operation labels, witness, selector schema version &
reserved/quarantined/terminal phase;
write-once return target and \(H(D_o),H(X_o\setminus D_o)\) inside the terminal
decision & unique \((root,operation\_id)\); immutable replay, partition, and
return binding \\
\mcode{gateway\_}\allowbreak\mcode{effect} & \(iid\), key, \(\beta_o\),
receiver, effect digest,
acceptance sequence & accepted/effect/terminal outcome & one immutable
acceptance binding and at most one instance per key \\
\mcode{terminal\_}\allowbreak\mcode{evidence} & evidence ID, tagged variant,
\(\beta_o\), root/slot/operation binding, authority/signature; success adds
\(iid\), acceptance sequence, and actual vector; no-effect adds fence epoch &
verification result and application time & unique evidence ID and one
compatible terminal variant per operation \\
\mcode{outbox} & event ID, aggregate ID, payload digest, terminal-outcome
digest; dispatch-permit authority, permit ID, permit digest, and signature & dispatch gate:
non-dispatchable/dispatchable/local-abort; delivery:
ready/sent/observed & inserted in the same transaction as the state
transition; permit fields are complete iff dispatchable; local-abort is
absorbing \\
\mcode{audit\_}\allowbreak\mcode{node} & previous hash, transition digest,
authority, causal
parents & publication/checkpoint state & append-only hash link; no authority
to move rights \\
\bottomrule
\end{tabularx}
\caption{Logical durable records. A cryptographic audit chain detects
post-hoc modification but does not replace atomic ownership transitions.}
\label{tab:schema}
\end{table*}

\statementheading{SQLite executable binding}
Table~\ref{tab:schema} is a logical schema, not a claim that every field is a
literal standalone SQL column. In the compact profile, the representation map
encodes the absorbing local gate as
\(\mcode{operations.state}=\mcode{aborted}\),
\(\mcode{abort\_mode}=\mcode{local\_unexposed}\), and
\(\mcode{outbox.dispatchable}=0\). The only dispatch CAS requires the
operation still be \mcode{reserved}, so this product state cannot later become
dispatchable. The \(W_3\) CAS simultaneously records the permit authority,
domain-separated permit identifier, digest, and Ed25519 signature; no
dispatchable row can have a partial permit. The logical terminal-outcome digest is reconstructed from the
terminal operation/evidence row by the independent checker rather than stored
as a separate outbox column.

The prototype persists a 16-field executable-beta row.  Its digest preimage is
the following 15-field projection:
\begingroup
\footnotesize
\[
\begin{aligned}
\beta_o^{\mathrm{SQL}}=(&
\mcode{issuer\_id},\mcode{root\_id},\mcode{credit\_id},\mcode{op\_id},
\mcode{epoch},\mcode{op\_key},\\
&
\mcode{lineage\_commitment},\mcode{rights\_commitment},
\mcode{effect\_digest},\mcode{receiver},\\
&
\mcode{max\_usd},\mcode{max\_calls},\mcode{source\_label\_digest},
\mcode{operation\_label\_digest},\mcode{witness\_digest}),
\end{aligned}
\]
\endgroup
and the sixteenth field is
\(\mcode{beta\_digest}=H(\mcode{FTBC\_EXECUTABLE\_BETA\_V1}
\parallel\mathsf{encode}(\beta_o^{\mathrm{SQL}}))\).
In this deliberately single-chain profile, define the authenticated
representation
\[
  \chi_{\mathrm{SQL}}(g)
  =H(\mcode{sqlite\mbox{-}single\mbox{-}slot}\parallel g).
\]
Because exactly one logical slot exists per root in the profile and
\(\mcode{root\_id}\) is committed inside both the concrete key and
\(\beta_o^{\mathrm{SQL}}\), the gateway and checker recompute the same
\(\chi_{\mathrm{SQL}}(g)\); the root fence is the physical representation of
the abstract fence for \((g,\chi_{\mathrm{SQL}}(g))\). No standalone
\mcode{slot\_id} column is implied.
Thus ``complete executable-beta'' in Section~\ref{sec:results} means complete
with respect to this concrete prototype schema: every field and the digest are
persisted or independently recomputed across the exercised branch, gateway,
effect, and terminal-evidence rows.  It does not claim that the SQLite
prototype implements the abstract atom-selector schema, nor does it discharge
the universal refinement obligation or a multi-slot adapter. That executable
profile has one serial
reissue chain per root and materializes a single
\mcode{root.current\_credit\_id}; this is a projection of per-chain currentness,
not a restriction on the full DAG schema, which permits pairwise-disjoint
current split siblings.

The compact gateway implements the permit branch of the abstract
first-acceptance guard. In the same \(W_3\) transaction that persists
\(\mathsf R\rightarrow\mathsf Q\) and marks the outbox dispatchable, the
branch persists an operation-scoped Ed25519 permit over the complete
\(\beta_o^{\mathrm{SQL}}\). For an unseen key, the gateway verifies that permit
and the current root fence before atomically recording acceptance. An exact
accepted replay returns the stored decision without requiring a new permit;
a durable no-effect proof excludes later acceptance of a delayed permit. The
independent checker, outside the enforcement TCB, separately verifies the
permit signature, ordering, and branch/gateway correspondence. These finite
traces exercise the declared guard but do not prove universal refinement for
arbitrary adapters or the complete end-to-end theorem.

When the SQLite profile creates a fresh successor during its final local abort,
no-effect return, or success settlement, the transaction simultaneously moves
the predecessor's already safe available balance and the terminal return
footprint. Its abstract image is the explicit \(\RolloverAct\) action in
Equations~\eqref{eq:rollover-ready}--\eqref{eq:rollover-action}, not a
purported single \(\Settle\), \(\ReturnAct\), or \(\mathsf{Reissue}\) step.

\subsection{Write Ordering and Crash Cuts}

For a protected operation, the required order is:
\[
\begin{array}{rcl}
W_1&=&\text{validate current credit and joint witness},\\
W_2&=&\text{atomically persist }\mathsf A\!\rightarrow\!\mathsf R
       \text{, operation, obligation, outbox, and audit node},\\
W_3&=&\text{atomically persist }\mathsf R\!\rightarrow\!\mathsf Q
       \text{ and a branch-signed permit over }\beta_o,\\
   &&\text{then mark the same outbox row dispatchable in that transaction},\\
N_1&=&\text{send through the fenced gateway},\\
W_4&=&\text{atomically verify and settle an authoritative outcome}.
\end{array}
\]

A crash before \(W_2\) exposes no work and owns no reservation. A crash after
\(W_2\) but before \(W_3\) leaves an unexposed reservation that local recovery
can prove safe to abort. A crash after \(W_3\), including before \(N_1\),
leaves a quarantined obligation. This can sacrifice availability, but it
eliminates the opposite crash cut in which an effect exists without a charged
right. A crash after \(N_1\) and before \(W_4\) is reconciled by the canonical
operation key rather than by reissuing a semantically new request.

The outbox is therefore a durability protocol, not a statement that message
delivery is exactly once. Duplicate sends are expected. The gateway's
idempotency and receipt query turn duplicates into observation of one logical
operation.

\subsection{Reference Components}

The reference profile separates six logical components; the compact artifact
implements the exercised subset of these roles:

\begin{enumerate}
  \item \textbf{Issuer.} Validates the root authority and fixes resource units,
        quantization, settlement authority, maximum-charge policy, and initial
        vector.
  \item \textbf{Credit authority.} Executes split, sealed transfer, acceptance,
        disjoint join, return, and epoch transitions.
  \item \textbf{Branch ledger.} Executes reserve, quarantine, settlement, and
        retirement in local transactions.
  \item \textbf{Gateway emulator.} Enforces stable keys and epochs, verifies
        the branch-signed dispatch permit before first acceptance, durably
        binds acceptance to \(\beta_o\), records effect instances and
        acceptance sequences, and exposes queryable outcomes.
  \item \textbf{Reconciler.} Resolves nonterminal transfers and operations
        without interpreting timeout as non-execution.
  \item \textbf{Independent checker.} Independently recomputes the declared
        keys, bindings, ownership totals, conservation, and coverage predicates
        at every retained durable snapshot/prefix.
\end{enumerate}

The checker is intentionally outside the mutation path. Agreement between
runtime assertions and an independently implemented checker is stronger
evidence than either alone, but is neither a complete abstract-transition
replay nor a handler-by-handler refinement proof of the production database or
cryptography.

\subsection{Concurrency and Placement}

Credits are the unit of locality. Operations on disjoint credits need no
shared counter and can proceed concurrently. A branch serializes transitions
for one credit generation, for example with a row-version compare-and-swap or
a local transaction lock. Cross-branch transfers serialize only the source
seal and destination acceptance; the rights remain unavailable in
\(\mathsf T\) between them.

A hot root does not imply a hot root row during ordinary spending. Root totals
can be reconstructed from credit and operation records or maintained as
non-authoritative aggregates. Coordination is required to move unused rights
between partitions, join credits at one authority, change resource semantics,
or reclaim a lost generation. This is the intended ``just-right'' boundary:
spending preallocated rights is local; changing exclusive ownership is not.

\subsection{Complexity and Metadata}

Let \(d=|\Resources|\), \(p\) be the number of parent credits in a join,
\(\ell\) the compressed lineage witness length, and \(z\) the number of
nonterminal records scanned during recovery. With range-compressed rights,
ordinary vector operations are \(O(d)\); the implementation does not
materialize the logical atoms of Definition~\ref{def:universe}.

\begin{table}[t]
\centering
\small
\begin{tabularx}{\columnwidth}{L{0.22\columnwidth}YY}
\toprule
\textbf{Transition} & \textbf{Local work} & \textbf{Coordination} \\
\midrule
Reserve / settle & \(O(d)\) & one local durable transaction \\
Split & \(O(d+\ell)\) & one local durable transaction \\
Cross-store transfer & \(O(d+\ell)\) & source seal + destination accept \\
Join & \(O(pd+p\ell)\) & one authority over all consumed inputs \\
Reconcile & \(O(d)\) per operation & gateway query; reclaim additionally needs
fence + non-acceptance proof \\
Audit replay & \(O(d)\) per transition & none on enforcement path \\
\bottomrule
\end{tabularx}
\caption{Asymptotic costs for compressed footprints. Cryptographic verification
adds work proportional to the chosen signature and commitment formats.}
\label{tab:complexity}
\end{table}

Each operation stores identifiers and digests plus \(O(d)\) quantities. A
transfer additionally stores a lineage/rights commitment. Full ancestry need
not travel on every request if the verifier trusts a checkpointed Merkle or
accumulator proof; that optimization changes proof format, not the ownership
invariant.

\subsection{Trusted Computing Base and Profiles}

The \emph{ledger profile} trusts the issuer, credit/branch stores, and
cryptographic verification. It proves only internal right conservation. The
\emph{mediated-effect profile} additionally trusts effect normalization,
complete mediation, the fenced gateway, and terminal-evidence authority, yielding the
coverage theorem. The \emph{accounting profile} may rely on a provider's usage
receipt. If the provider can omit hidden tokens, canceled-stream usage, taxes,
or asynchronous charges, the theorem covers the reported metric rather than
the eventual invoice.

An endpoint can advertise one of three capabilities:

\[
\mathsf{profile}(r)\in
\{\mathsf{FencedQueryable},\mathsf{ReceiptOnly},\mathsf{SelfReported}\}.
\]

The deployment must not silently upgrade these profiles. In particular,
\(\mathsf{ReceiptOnly}\) can preserve no-amplification by quarantining ambiguity
but cannot promise eventual release, and \(\mathsf{SelfReported}\) cannot
support an effect-coverage claim about receiver state.

\section{Evaluation Methodology}
\label{sec:methodology}

The evaluation is claim-driven and intentionally performed after the model,
protocol, theorem assumptions, and failure semantics are fixed. This order
reduces the risk of fitting definitions to favorable measurements. The
artifact uses no hosted model or provider. Agent decisions are deterministic
event generators because the evaluated claims concern distributed ownership
and settlement, not task quality.

We use frozen cases, explicit oracles, and machine-readable denial reasons.
The present oracle is a distributed state invariant, and each paired
variation is a crash, message, permit, or epoch schedule rather than a natural
language authorization context.

MAS-FIRE provides a complementary taxonomy and injection framework for
semantic, cognitive, and inter-agent coordination faults in LLM-based
multi-agent systems \cite{jia2026masfire}. Our evaluation targets a different
layer: durable crash cuts, duplicated or reordered protocol events, epoch
fencing, and preservation of ownership and effect-accounting invariants. It
does not measure reasoning robustness.

\subsection{Research Questions}

\begin{description}[leftmargin=2.4em,style=nextline]
  \item[\textbf{RQ1: Model safety.}]
  In the declared finite scopes, do the settlement/generation and
  delegation-lineage machines preserve their stated ownership, conservation,
  settlement, and fencing invariants?

  \item[\textbf{RQ2: Necessity.}]
  Does timeout refund yield a short mechanism-specific counterexample, and
  does a compound fan-in certificate-validation mutant that substitutes
  historical inputs and omits disjointness produce the predicted ownership
  amplification?

  \item[\textbf{RQ3: Crash-fault runtime.}]
  Across the declared \(W_2\), gateway-acceptance, and \(W_4\) pre/post-commit
  cuts, duplicate receipt, stale epoch, response loss, and
  pre-fence-accept/late-completion schedule, does an independent checker find
  any prefix violation in the safe SQLite profile?

  \item[\textbf{RQ4: Mechanism cost sanity.}]
  Do all injected process deaths recover durable SQLite state, and what local
  RPC and restart timings characterize this prototype execution? These
  observations are health checks, not a comparative performance claim.
\end{description}

\subsection{Evidence Strata}

No single experiment establishes the end-to-end theorem. We therefore use
four deliberately separated evidence strata:

\begin{enumerate}
  \item two complementary TLA+ specifications explored by TLC: one for
        settlement/generation recovery and one for lineage/fan-in
        \cite{lamport2002specifying};
  \item a dependency-free JavaScript encoding that exhausts the same declared
        settlement scope and runs named schedules;
  \item two SQLite service processes, driven through separate durable files,
        with real \mcode{SIGKILL} cuts and an independent newline-delimited
        JSON (JSONL) trace checker;
        and
  \item lightweight local RPC/restart timing captured while those correctness
        schedules execute.
\end{enumerate}

The TLA+ and JavaScript artifacts are separate executable encodings, not
independent conceptual specifications. The SQLite checker imports neither
encoding, protocol implementation, SQL, scenario, nor oracle module; it
recomputes properties from durable snapshots. This reduces, but cannot
eliminate, common-specification error. None of the strata proves an arbitrary
deployment, endpoint, database, signature library, or the universal
trace-lifting obligation in Equation~\eqref{eq:stepwise-refinement}.

The settlement TLA+ and JavaScript encodings deliberately type each logical
unit with one location. Their \mcode{ExactConservation} check is therefore a
transition-accounting and regression sanity check, not independent evidence
for the manuscript's raw-relation ownership-partition lemma. The hand proof
establishes that lemma without a location-function type; the complementary
lineage model can represent multiple tagged occurrences and its mutant
exercises that multiplicity. We do not report bounded state counts as a
machine-checked proof of arbitrary raw-relation ownership.

\subsection{Analytical Comparison Points}

The following mechanisms structure the threat and novelty comparison; they
are not presented as measured reimplementations:

\begin{enumerate}
  \item \textbf{Read--check--decrement.} Replicated remaining balances with
        last-writer-wins merge; this isolates the invariant-convergence error.
  \item \textbf{Central linearizable counter.} One Raft-style authoritative
        decrement and settlement record \cite{ongaro2014raft}; safe when
        reachable, unavailable to isolated branches.
  \item \textbf{Static escrow / Bounded Counter profile.} Rights are
        preallocated and never dynamically reclaimed during faults
        \cite{balegas2015bounded}; this is the strongest numeric baseline.
  \item \textbf{Affine local budget.} A single-process move-only credit modeled
        after Token Budgets \cite{khan2026tokenbudgets}; cross-process crash
        state is outside its type envelope.
  \item \textbf{Allocation-contract monitor.} Child vectors must sum within
        the parent as in Agent Contracts \cite{ye2026agentcontracts}, but
        external execution has no durable obligation.
  \item \textbf{Shared replay registry.} AIP/AAT-style attenuated tokens plus a
        linearizable operation-ID registry; this prevents some replay but
        lacks exclusive consumable transfer semantics.
  \item \textbf{Fail-closed coordinator.} MPAC-style centralized current state
        rejects mutation while the coordinator is unavailable
        \cite{qian2026mpac}.
  \item \textbf{Proposed profile.} Exclusive credit lineage plus
        reserve/quarantine/terminal-evidence settlement and gateway fencing.
\end{enumerate}

RIFL, Beldi, and ExoFlow are architectural reference points for durable
completion records and exactly-once workflows
\cite{lee2015rifl,zhang2020beldi,zhuang2023exoflow}. The reference artifact
evaluates this paper's conservation and crash-cut obligations; these systems
remain architectural comparison points rather than reimplemented performance
baselines.

\subsection{Finite Scopes and Concrete Schedules}

The settlement/generation model contains two logical atoms, two actors,
epochs and operation generations \(0..1\), transfer seal/accept/verified
cancel, reserve, dispatch, timeout, gateway effect, settlement, fencing,
verified abort, successor-generation reuse, forfeiture, crash, and restart.
The JavaScript encoding explores the complete reachable graph of the same
finite abstraction and executes named traces.
Those settlement models sequence terminal resolution and ordinary reissue;
they do not contain a separate fused \(\RolloverAct\) transition. The fused
action is covered by the inductive proof and the selected SQLite transactions,
not by a claimed TLC/JavaScript bisimulation.

The complementary lineage model contains two dimensions (\mcode{usd} and
\mcode{tokens}), two rights per dimension, root/left/right/join credits,
generations \(0..2\), epochs \(0..1\), an atomic disjoint split, a legal
current-input fan-in, and sequential full transfers and returns. The safe
join consumes the two current, active, pairwise-disjoint branches atomically;
a retained reachability witness confirms that this legal join is not
vacuously disabled. The sequential workflow separately produces two
overlapping \emph{historical} return certificates without duplicating live
ownership. The safe model never treats those stale certificates as join
inputs. A deliberately compound validation mutant both substitutes them for
current active inputs and omits their disjointness check; its counterexample
must therefore not be interpreted as isolating either check independently.

The concrete SQLite profile issues vector \((100,2)\) in dimensions
\mcode{usd} and \mcode{calls}. It runs ten safe cases and one
timeout-refund mutant. Each case creates fresh branch and gateway databases.
The safe schedules cover pre/post-commit death at \(W_2\), pre/post-commit
death at \(W_3\), pre/post-commit death at gateway acceptance,
effect persistence with reply loss,
pre/post-commit death at \(W_4\), receipt replay, stale-epoch rejection, and a
request accepted before a fence that completes after it. Table~\ref{tab:fault-matrix}
states the broader protocol cuts and identifies which are exercised by which
artifact layer.

\begin{table*}[t]
\centering
\small
\begin{tabularx}{\textwidth}{L{0.21\textwidth}L{0.25\textwidth}YY}
\toprule
\textbf{Fault injection point} & \textbf{Ambiguity created} &
\textbf{Required safe state} & \textbf{Bug exposed if weakened} \\
\midrule
After reserve commit, before dispatchable & No exposure possible by durable
evidence & \(\mathsf R\), then verified abort & leaked reservation or unsafe
send without outbox \\
After quarantine commit, before send & Send may not have happened &
\(\mathsf Q\) until gateway proof & utilization loss if safe; double spend if
blindly refunded \\
After gateway effect, before reply & Effect exists, response absent &
\(\mathsf Q\), then receipt settlement & timeout-refund/late-effect overspend \\
After receipt arrival, before ledger commit & Receipt can replay &
same settlement applied once & duplicate charge or duplicate remainder \\
During sealed transfer & Source debited, destination unknown &
\(\mathsf T\) until accept/exclusion & both-side ownership after retry \\
During epoch advance & Old worker may retain credentials &
old gateway requests fenced & stale-generation spend after reclaim \\
During fan-in & Two references may share ancestry or be stale &
require current active inputs, reject overlap, and consume atomically &
join amplification \\
During partial-vector settlement & Actual charge smaller than maximum &
commit exact subset; return authenticated remainder & invented remainder or
cross-dimension substitution \\
\bottomrule
\end{tabularx}
\caption{Required crash and concurrency cuts. Each outcome is checked after
every durable prefix, not only after recovery completes.}
\label{tab:fault-matrix}
\end{table*}

The SQLite layer uses real process death at the \(W_2\) reservation
transaction, the \(W_3\) quarantine/permit transaction, the separate
gateway-acceptance transaction, effect/reply-loss cut, and \(W_4\) settlement
transaction; it also exercises stale-epoch rejection and pre-fence acceptance
followed by late completion. Transfer
cancellation and partial charge are likewise exercised in the executable
settlement traces; the lineage model exercises a legal current fan-in plus an
invalid historical-certificate path.
These scopes are reported separately so model-level coverage is not mistaken
for a distributed runtime implementation of transfer or DAG join.

\subsection{Metrics}

The primary safety metrics are counts of:
\[
\begin{split}
V_{\mathrm{own}}&=\#\text{ownership-partition violations},\\
V_{\mathrm{sum}}&=\#\text{bucket-conservation violations},\\
V_{\mathrm{cover}}&=\#\text{uncovered or multiply covered effects},\\
V_{\mathrm{bound}}&=\#\text{external-exposure bound violations},\\
V_{\mathrm{settle}}&=\#\text{conflicting/double settlements},\\
V_{\mathrm{epoch}}&=\#\text{accepted stale-generation effects}.
\end{split}
\]

All must be zero for the proposed safe profile. Mutation success is the
opposite: each intentionally weakened rule must produce at least one
independently validated counterexample in a schedule that exercises the
weakened assumption.

Capacity state can be reported as:
\[
  s_T=\frac{\|\mathbf T\|_1}{\|\mathbf B\|_1},\quad
  s_Q=\frac{\|\mathbf Q\|_1}{\|\mathbf B\|_1},\quad
  s_F=\frac{\|\mathbf F\|_1}{\|\mathbf B\|_1}.
\]
For the concrete artifact, restart time is measured from process death to a
ready service against the same database. We report only RPC sample count,
per-case median/p95, and aggregate restart median captured by the correctness
run. We do not infer throughput, tail SLOs, capacity, or comparative overhead
from those samples. Safety is never combined with timing into one score.

\subsection{Decision Rules and Statistical Treatment}

Bounded model checking is reported as explored distinct states, generated
states, depth/diameter when available, completed search status, and invariant
violations. A zero count means no counterexample exists in that declared
finite scope; it is not reported as an unbounded proof.

Deterministic transition and fault cases are pass/fail with an exact case
manifest. The local timings are descriptive observations over the RPCs and
restarts that actually occur in those cases; there is no warm-up,
random-sampling model, confidence interval, or hypothesis test. Runtime,
hardware, SQLite version, topology, and the fact that CLI process startup is
included are recorded with the results.

\subsection{Predeclared Interpretation Boundaries}

The following outcomes would narrow or falsify claims:

\begin{itemize}
  \item any safe-model or proposed-runtime prefix with
        \(V_{\mathrm{own}}+V_{\mathrm{sum}}>0\) falsifies ledger conservation;
  \item any mediated gateway effect without exactly one bound
        \(\mathsf Q/\mathsf C/\mathsf F\) obligation falsifies effect coverage;
  \item a mutant not detected within the exercised scope weakens evidence for
        necessity and must be reported rather than reclassified;
  \item a timeout-only reclaim that appears safe merely indicates the schedule
        omitted the delayed-effect twin;
  \item provider self-reported usage cannot validate actual invoice bounds;
  \item fault-free microbenchmark speed cannot establish fault tolerance; and
  \item a bounded TLC completion cannot establish correctness outside its
        constants or assumptions.
\end{itemize}

\subsection{Reproducibility Protocol}

The artifact entry point runs formal exploration, mutation, named transition,
fault injection, aggregation, and internal consistency checks in that order.
Every result file records the command, source digest, seed or schedule ID,
runtime version, start/end time, and declared scope. Generated LaTeX tables are
derived from JSON results; hand-entered benchmark numbers are forbidden.
Appendix~\ref{app:artifact} gives the exact commands and expected artifacts.

% AUTO-GENERATED by bench/aggregate_results.mjs; do not edit.
% aggregate-sha256: 5c9d361bbe2560a84ed44cb4d7ed550e2d10227e17b4e43f4bf75e2816f34916
\section{Evaluation Results}
\label{sec:results}

\subsection{Bounded Model-Checking Results}

All three safe finite graphs completed with an empty work queue and no
invariant violation.  The settlement TLC model explored
54,477 generated and
16,592 distinct states; the
JavaScript encoding reached the same count of
16,592 distinct states.  Equality of
counts does not establish equality of state sets, a bisimulation, or independent
specification evidence; it is only a coarse cross-encoding consistency check.

\begin{table*}[t]
\centering
\scriptsize
\setlength{\tabcolsep}{3.5pt}
\begin{tabularx}{\textwidth}{@{}L{0.14\textwidth}XL{0.12\textwidth}L{0.11\textwidth}L{0.08\textwidth}L{0.17\textwidth}@{}}
\toprule
\textbf{Evidence} & \textbf{Declared finite scope} &
\textbf{Generated or edges} & \textbf{Distinct states} &
\textbf{Depth} & \textbf{Safe outcome} \\
\midrule
Settlement TLC & 2 atoms,
2 actors, epochs/generations
0..1 &
54,477 &
16,592 &
29 & exhausted; no violation \\
Lineage TLC & 2 dimensions,
4 rights,
4 credits &
23 &
21 &
8 & exhausted; no violation \\
JavaScript graph & 2 atoms,
2 actors, complete reachable graph &
54,476 &
16,592 &
28 & exhausted; no violation \\
\bottomrule
\end{tabularx}
\caption{Completed bounded safety checks. The generated-or-edges column denotes
TLC generated states for the TLA+ rows and explored transitions for JavaScript.
Completion applies only to the constants shown.}
\label{tab:formal-results}
\end{table*}

The separate expected-failure reachability profile retained a
3-state legal current-input join
witness.  Both inputs were active, their live footprints were pairwise
disjoint and covered all 4
rights, and the output retained multiplicity one.  This establishes that the
safe join action was reachable; it is not a safety counterexample.

The deterministic executable suite then ran 19
named schedules (132 scheduled steps).
All 18 safe schedules passed with zero
final invariant violations, including gateway acceptance before a fence
followed by completion after the fence, and the corresponding forfeiture
case. These schedules test particular executions; they do not enlarge the
exhaustive finite scope.

\subsection{Targeted Necessity Checks}

Both predeclared weakened mechanisms produced short, mechanism-specific
counterexamples.  The lineage mutant is compound: it substitutes returned
historical certificates for current active inputs and bypasses pairwise
disjointness validation.  It must not be interpreted as an overlap-only
ablation.

\begin{table*}[t]
\centering
\scriptsize
\setlength{\tabcolsep}{3.5pt}
\begin{tabularx}{\textwidth}{@{}L{0.23\textwidth}L{0.19\textwidth}L{0.18\textwidth}X@{}}
\toprule
\textbf{Weakening} & \textbf{Detection} & \textbf{Shortest retained witness} &
\textbf{Observed failure} \\
\midrule
Timeout refund after exposure & TLC and JavaScript &
TLC depth 4;
JS depths 3
and 5 &
unsafe reclaim; then loss of effect coverage after the accepted request
completes late \\
Historical-certificate validation (compound) & lineage TLC &
6 states &
ownership uniqueness; 8
live occurrences for 4
rights (maximum multiplicity
2) \\
\bottomrule
\end{tabularx}
\caption{Targeted mutant results. Witness lengths are breadth-first
counterexamples in the declared finite models, not field-failure
probabilities. The historical-certificate row weakens both currentness and
disjointness validation.}
\label{tab:mutation-results}
\end{table*}

\subsection{Multi-Process Crash-Fault Results}

The concrete artifact executed 10/10
safe SQLite cases without a prefix violation under the independent checker.
It injected 9 actual \mcode{SIGKILL} process deaths and
retained 313 JSONL records.  The checker also detected the
timeout-refund mutant through \mcode{dispatch\_coverage},
\mcode{effect\_coverage}, and \mcode{effect\_budget\_bound}.
It recomputed 25 canonical keys
and 59 complete executable-beta rows,
verified 11 signed success
receipts and 1 signed no-effect
proof, and checked 14 terminal
return decisions including
3 fresh successors.  All
12 negative handler assertions were
observed without durable mutation.

The compact gateway enforces the evaluated first-acceptance guard online: the
branch persists and signs an operation-scoped dispatch permit in the same
\(W_3\) transaction that moves the reservation to quarantine, and the gateway
verifies and records that permit before accepting an unseen key.  Exact replay
returns the stored outcome, while a later authoritative no-effect proof excludes
acceptance of a delayed permit.  The independent checker separately recomputes
permit signatures and cross-store dispatch coverage from the trace.  These
finite schedules validate the reference profile; they do not prove universal
refinement for arbitrary runtime adapters.

\begin{table}[t]
\centering
\scriptsize
\setlength{\tabcolsep}{3pt}
\begin{tabularx}{\columnwidth}{@{}XrrL{0.12\columnwidth}@{}}
\toprule
\textbf{Safe schedule} & \textbf{Kills} & \textbf{Records} &
\textbf{Checker} \\
\midrule
W2 pre-commit rollback & 1 & 29 & pass \\
W2 post-commit recovery & 1 & 23 & pass \\
gateway-accept pre-commit rollback/retry & 1 & 33 & pass \\
gateway-accept post-commit dedupe & 1 & 24 & pass \\
effect/reply loss; W4 pre-commit & 2 & 29 & pass \\
W4 post-commit dedupe & 1 & 27 & pass \\
stale epoch rejection & 0 & 39 & pass \\
pre-fence accept; late completion & 0 & 30 & pass \\
W3 pre-commit rollback & 1 & 25 & pass \\
W3 post-commit permit recovery & 1 & 25 & pass \\
\bottomrule
\end{tabularx}
\caption{Independent checks over ten concrete safe traces. Each case uses
fresh branch-ledger and gateway databases.}
\label{tab:sqlite-results}
\end{table}

All 22/22
retained SQLite databases returned \mcode{ok} from
\mcode{PRAGMA integrity\_check}.  This is evidence that the captured crash
cuts left readable database images; it is not a proof of SQLite internals.
The safe late-completion case binds the gateway-assigned effect instance and
acceptance sequence through settlement and moves its partial remainder to a
fresh, same-lineage successor rather than the stale source.  The stale-epoch
case separately exercises local unexposed abort and gateway-proved no-effect
return; both create fresh successors only after predecessor closure and exact
reusable-balance computation.

The correctness run collected 111 local RPC samples.
Per-case medians ranged from
9.702 to
12.007 ms, the largest
per-case p95 was 28.111 ms, and
the 9 restart observations had median
30.76 ms.  These measurements
include local process and SQLite CLI startup and serve only as mechanism
health checks; they establish neither throughput nor a production latency
SLO.

\subsection{Interpretation and Residual Limits}

The result is bounded but falsifiable: zero safe violations in the declared
finite models, zero violations across all 10 retained
safe runtime traces,
and detection of both targeted weakenings.  It does not establish liveness,
Byzantine tolerance, production cryptographic assurance or receipt-authority
integrity, arbitrary runtime refinement, cross-store credit transfer, or DAG
join in the SQLite prototype.  In particular, the SQLite profile exercises
the implemented first-acceptance guard but does not furnish a universal
handler-level refinement proof for \(\mathsf{GatewayAccept}\) or the
conditional end-to-end theorem.  Its fused terminal-settlement/successor transactions have the
explicit analytical image \(\RolloverAct\), but no universal handler
simulation proof is claimed.  The Ed25519 fixture exercises authentication
and negative verification paths but provides no production key isolation.
External-effect claims remain conditional on a cooperating, idempotent,
fenced, and queryable gateway with complete mediation.

\section{Security Analysis}
\label{sec:security}

\subsection{Safety Against a Malicious or Confused Agent}

The agent is not trusted to report a balance, choose a credit generation,
forge terminal evidence, or decide that a timeout is a refund. It may emit
arbitrarily many proposals, retries, delegations, cancellations, and joins.
The branch store accepts a proposal only if the current credit exclusively
owns the selected footprint, the label and joint witness authorize the
normalized effect, and the immutable operation binding either does not yet
exist or matches exactly. Consequently, prompt injection can influence the
proposed action but cannot by itself duplicate a right.

A compromised worker can still waste its own authorized escrow by reserving
unproductive work, creating indeterminate obligations, or intentionally
withholding completion. Conservation is a safety property, not a defense
against budget denial. Deployments need rate limits, per-child reservation
caps, bounded delegation depth, admission control, and human escalation for
large or long-lived quarantines.

\subsection{Replay, Equivocation, and Identifier Substitution}

Every transfer, operation, and terminal-evidence item has a stable identifier,
but identifier uniqueness alone is insufficient. The branch/gateway
\emph{enforcement verifier}, which is inside the safety TCB, recomputes the
single canonical operation binding in
Equation~\eqref{eq:accept-binding}, including the selector schema version
\(\nu_{\mathsf{sel}}\), and rejects any field mismatch. Transfers use their
separate domain-separated certificate and decision binding; they are not
silently interpreted through the operation formula. The independent
\emph{audit verifier} in Figure~\ref{fig:architecture} is outside the TCB and
can detect divergence but cannot make an unsafe request safe. An exact replay
is a stutter. Reuse of an identifier with a different tuple is an equivocation
and is rejected. Domain separation prevents an operation ID from being
reinterpreted as a transfer or receipt ID. Receiver and effect class are
included so two services cannot accidentally share an idempotency namespace.

Settlement also binds where unused rights may return.  A target must satisfy
Equation~\eqref{eq:return-target}: it has the same root, is a
successor of the source lineage that is current in the committed successor
state under the settlement
authority, and has a label no broader than the source credit.  An agent
therefore cannot use a valid partial receipt to redirect the remainder into
an unrelated, broader, or stale credit.

Terminal-evidence equivocation is detected when two authenticated terminal outcomes
bind the same operation. Under Assumption~\ref{ass:receipt}, at most one can be
valid. If a deployed terminal-evidence authority can sign inconsistent outcomes, the
main effect theorem no longer applies; the ledger can quarantine the
conflict, preserve no-amplification, and expose evidence, but cannot establish
receiver truth.

\subsection{Stale Generations and Split Brain}

Epoch fencing protects against a lost controller or worker that later resumes.
The critical check is at the last trusted boundary before the effect, not only
at the branch database. A gateway accepts a new request for
\((cid,\eta)\) only while that generation is current. After an authenticated
\(\eta^+>\eta\) is installed, it may return the already recorded result for an old
idempotency key but must reject a new old-generation effect.

If two authorities can independently declare themselves current for the same
credit generation, exclusive ownership is already violated. Consensus,
single-writer placement, or another correct ownership service is required
within that authority. The protocol confines this requirement to ownership
changes; it does not solve consensus and does not claim Byzantine split-brain
tolerance.

\subsection{External-Effect Ambiguity}

The most consequential attack surface is a receiver that completes work and
drops or falsifies the response. The safe response is not to guess. Under the
fenced-queryable profile, reconciliation obtains the recorded result. Under
the receipt-only profile, missing evidence leaves the right in \(\mathsf Q\)
or eventually \(\mathsf F\). Under self-report, only the internal ledger
claim survives.

This tiering follows the same fundamental requirement as systems that provide
exactly-once RPC or workflow execution: a completion record must be durably
coupled to the mutation or returned by a cooperating transactional endpoint
\cite{lee2015rifl,zhang2020beldi,zhuang2023exoflow}. A client-side retry
library cannot manufacture that coupling for an arbitrary third party.

\subsection{Underreporting and Variable Charges}

A maximum reservation \(\mathbf u\) is sound only if the gateway can prevent
an actual charge above it. For a paid API, that can mean preauthorization,
provider-side hard caps, chunked metering with renewal, or reserving the
maximum billable unit. If the provider reveals usage only afterward and may
exceed \(\mathbf u\), no local algorithm can prove a hard bound on the real
charge. The appropriate choices are to refuse the hard-bound profile, reserve
a conservative worst case, or state a softer monitoring claim.

Partial settlement is safe when the receipt authenticates
\(\mathbf d\preceq\mathbf u\). The remainder returns only after that receipt.
For streaming calls, the gateway can use successive sub-operations with
separate reserved chunks; canceling a stream is not itself a zero-charge
receipt.

\subsection{Normalization and Path Aliasing}

Two syntactically different requests may cause the same effect, while two
similar strings may address different recipients or preconditions.
\(\Normalize\) is therefore part of the trusted policy boundary. Its test
obligations include canonical encoding, default parameters, redirects,
indirect identifiers, locale and unit conversion, recipient resolution,
protocol-specific retries, and relevant domain preconditions. Unknown
semantics fail closed.

Complete mediation is equally important
\cite{saltzer1975protection}. An agent with direct cloud
credentials can bypass an MCP proxy; a browser can activate an effect through
a UI path; a tool can invoke a second tool internally. The theorem covers only
credentials and routes that the deployment actually interposes. Network
egress policy, workload identity, receiver-side enforcement, and credential
isolation are deployment controls, not consequences of the budget data
structure.

\subsection{Audit Integrity and Rollback}

A hash-linked log reveals deletion or reordering relative to a trusted
checkpoint, but a compromised store could roll back both state and its local
log. Production profiles should bind store generation, credit epoch, and
checkpoint sequence to a monotonic external witness, replicated log, hardware
counter, or consensus state. This strengthens forensic detection; exclusive
ownership still depends on the authoritative transition service.

Audit records contain sensitive purpose, recipient, spend, and graph
metadata. The default evidence plane should use digests and stable opaque
identifiers, encrypt detailed receipts at rest, separate operational and audit
roles, and enforce retention and deletion policies. Public verifiability does
not justify publishing task content.

\begin{table*}[t]
\centering
\small
\begin{tabularx}{\textwidth}{L{0.20\textwidth}L{0.23\textwidth}YY}
\toprule
\textbf{Attack or failure} & \textbf{Prevented or detected by} &
\textbf{Residual consequence} & \textbf{Outside the theorem} \\
\midrule
Duplicate child allocation & atomic exclusive split + stable ID &
duplicate request stutters & corrupt authoritative store \\
Lineage alias at fan-in & disjoint rights commitment + input consumption &
join rejected & unsound commitment implementation \\
Dispatch replay & operation binding + gateway idempotency & prior outcome
returned & non-cooperating receiver \\
Timeout refund race & \(\mathsf Q\) retention + reconciliation & stranded
capacity & guaranteed eventual receiver response \\
Stale worker & gateway-visible epoch fence & old request rejected or observed &
direct credentials bypassing gateway \\
Forged/changed receipt & signature and full binding verification &
quarantine and investigation & compromised terminal-evidence authority \\
Provider underreporting & authenticated hard-cap profile & weaker
reported-usage claim & correctness of an arbitrary invoice \\
Audit rollback & external checkpoint / monotonic epoch & detect or confine
rollback & Byzantine consensus if not deployed \\
Reservation flooding & per-child caps and admission policy & availability
loss within authorized escrow & agent usefulness or fair allocation \\
\bottomrule
\end{tabularx}
\caption{Security coverage and residual risk. The protocol prefers explicit
loss of utility to an unsupported conservation claim.}
\label{tab:security}
\end{table*}

\section{Related Work Beyond the Closest Mechanisms}
\label{sec:related}

\subsection{Agent Economies and Governance Receipts}

EpochX treats tasks as bounties, supports recursive budgeted delegation, and
uses lock and settle operations \cite{wang2026epochx}. It is therefore direct
prior art for hierarchical economic allocation; our focus is the failure
semantics around duplicate delivery, persistent ownership, quarantine,
fencing, and late settlement that are not established by EpochX's reported
model and evaluation. AgentBound centralizes pre-execution
governance and emits signed governance receipts
\cite{kaul2026agentbound}. Those receipts attest to a policy decision, while
our effect receipt attests to a receiver outcome and actual charge. Human
Delegation Provenance (HDP), a lightweight cryptographic protocol, provides
human-delegation provenance,
while Aegon proposes ledger-bound tokens
with attested receipts \cite{dalugoda2026hdp,baskaran2026aegon}. We do not
claim the first delegation chain, ledger token, or signed receipt.

Agent Contracts and Token Budgets are the nearest resource-governance works
\cite{ye2026agentcontracts,khan2026tokenbudgets}. Their allocation law and
affine ownership discipline motivate our model. The contribution here begins
where one process or one allocation snapshot stops being authoritative:
ownership must survive messages, durable crash cuts, joins, and uncertain
effects.

\subsection{Authorization Delegation}

OAuth token exchange and Rich Authorization Requests provide standardized
building blocks for audience-bound exchange and structured authorization
\cite{rfc8693,rfc9396}. AIP and the current AAT Internet-Draft extend these
ideas to agent identity and attenuating chains
\cite{prakash2026aip,niyikiza2026aat}; Overlaying
Governance formalizes policy composition \cite{ibrahim2026overlay}. These
answer \emph{which actions remain permitted}. Our credits add linear
consumption: two equally attenuated children must not both own the same last
unit.

Resource-bounded temporal logics have long represented strategies constrained
by consumable or producible resources \cite{alechina2010resource}. Our problem
is operational rather than strategic: it specifies how authenticated rights
move and settle under crash and partition. The logical connection is useful,
but a model-checking semantics does not supply durable message protocols by
itself.

\subsection{Escrow, CRDTs, and Mobile Rights}

Escrow transactions reserve rights so concurrent transactions preserve a
numeric invariant \cite{oneil1986escrow}. Mobile escrow distributes rights to
occasionally disconnected clients \cite{krishnakumar1997mobileescrow}.
Bounded Counter CRDTs provide a particularly strong construction for
partition-local decrement and rights transfer
\cite{balegas2015bounded}. Just-Right Consistency combines such mechanisms
with analysis that coordinates only invariant-sensitive operations
\cite{shapiro2018justright}.

We adopt, rather than rename, this lineage. The extension is that an agent
credit carries an authorization label and dynamic delegation lineage, the
task graph can join without merging aliased rights, and a spent unit remains
coupled to an external-effect obligation through uncertainty. A numeric
counter that converges to a valid value can still miss a duplicated payment
or disclosure.

\subsection{Transactions, Workflows, and Exactly-Once Semantics}

Sagas decompose long-lived transactions into steps and compensations
\cite{garciamolina1987sagas}; Atomix specializes transactional semantics to
agent tool use \cite{mohammadi2026atomix}. RIFL couples request IDs with
durable completion records \cite{lee2015rifl}. Beldi supplies fault-tolerant
transactional serverless workflows, and ExoFlow targets exactly-once DAG
execution \cite{zhang2020beldi,zhuang2023exoflow}. These systems demonstrate
that replay safety is a storage and endpoint property, not a property of a
prompt.

Our theorem is narrower in one dimension and broader in another. It does not
claim a universal exactly-once workflow. It proves at-most-once settlement and
continuous charged coverage for consumable delegated authority, including the
safe but unavailable outcome in which an arbitrary effect cannot be
authoritatively resolved.

\subsection{Coordination and Impossibility}

CAP rules out simultaneous strong consistency and availability for every
request during arbitrary partitions \cite{gilbert2002cap}; FLP limits
deterministic consensus progress in an asynchronous system with a crash
\cite{fischer1985flp}. Invariant confluence and the Coordination Criterion
more directly characterize when application correctness permits
coordination-free execution
\cite{bailis2014coordination,hellerstein2026coordination}. Our
Proposition~\ref{prop:shared-unit} is an application-specific consequence:
exclusive preallocation permits local work, while creating or reallocating a
shared spendable unit requires coordination or blocking.

\section{Discussion}
\label{sec:discussion}

\subsection{What Is Actually Conserved?}

The object conserved is authenticated authorization capacity, not money or
physical energy. A dollar dimension means ``authority to cause up to one
dollar of charges under a named settlement rule.'' If taxes, exchange rates,
tips, later adjustments, or hidden provider usage are outside that rule, they
are outside the invariant. A rigorous deployment must define each dimension's
unit, rounding, maximum charge, authority, and reconciliation semantics before
issuing it.

The exact equality in Equation~\eqref{eq:conservation} is stronger and more
diagnostic than merely checking \(\mathbf C\preceq\mathbf B\). It distinguishes
spendable, moving, not-yet-exposed, uncertain, consumed, and sacrificed
capacity. Operators can therefore see whether apparent under-spending is true
availability or merely hidden stranding.

\subsection{Why Keep Both
\texorpdfstring{\(\mathsf R\) and \(\mathsf Q\)}{R and Q}?}

Combining reserved and indeterminate rights preserves the numeric equation,
but loses a crucial recovery fact. A durable \(\mathsf R\) record whose outbox
provably never became dispatchable can be locally returned. A \(\mathsf Q\)
record may correspond to a real effect and needs external evidence. The
separation turns a crash cut into an auditable transition and prevents a
generic ``release pending reservations'' job from creating authority.

\subsection{Safety, Liveness, and Economic Utility}

The protocol's safety is prefix-closed; its liveness is conditional. A branch
can progress during partition while it has suitable local rights and a
reachable gateway. Once exhausted, it blocks even if another partition has
idle capacity. A quarantined right can remain unavailable forever. These are
not incidental implementation shortcomings but consequences of refusing to
guess about exclusive ownership or external effects.

Allocation policy can reduce the cost: assign escrow proportional to expected
demand, keep small emergency reserves in multiple zones, rebalance before
known disconnections, and choose queryable endpoints for large effects.
Dynamic prediction may improve utilization but must not enter the safety
theorem.

\subsection{Vectors and Conversion}

Vectors prevent a worker from exchanging, for example, ten low-risk reads for
one high-risk write without authorization. Some applications legitimately
convert resources: a dollar cap may purchase token calls, or a global risk
budget may be divided among effect classes. Such conversion is a new protected
effect. It needs a signed rate/rounding contract, consumes source rights,
creates target rights once, and preserves a declared value invariant. Without
that extra contract, conversion is undefined rather than inferred.

\subsection{DAG Tasks, Forest-Owned Rights}

Task dependencies and budget ownership should remain different graphs. A
synthesis agent can read results from many parents without owning their
remainders. When it needs budget, each input branch may transfer a disjoint
credit or one authority may perform an atomic join. This design prevents the
common error of summing balances associated with incoming task edges.

At large scale, full atom-set proofs are impractical. Range certificates,
bounded-counter rights matrices, or cryptographic commitments can compress
the footprint. The implementation must nevertheless expose a disjointness
witness and consume each input generation exactly once.

\subsection{Standards Integration}

The protocol does not require a new wire protocol. AIP/AAT or OAuth can carry
identity and attenuated scope; MCP or A2A can carry opaque credit and operation
bindings; HTTP idempotency or conditional writes can implement a gateway
primitive. Standardization value lies in a common semantic profile:

\begin{itemize}
  \item authenticated resource schema and maximum-charge semantics;
  \item stable root, logical-slot, transfer, credit-generation, operation, and
        terminal-evidence IDs;
  \item immutable binding digest and the per-slot current epoch/fence;
  \item endpoint capability declaration;
  \item terminal outcomes including verified no-effect and forfeiture; and
  \item machine-readable reason codes for conservative refusal.
\end{itemize}

Interoperability testing should verify refinement of these meanings, not only
JSON shape compatibility.

\subsection{Deployment Migration}

A practical rollout can proceed in increasing assurance:

\begin{enumerate}
  \item instrument existing agents with stable operation IDs and observed
        usage, making no hard-bound claim;
  \item move counters into a durable branch ledger and reserve before effects;
  \item route high-impact effects through idempotent, queryable gateways;
  \item introduce exclusive child credits and disable copyable numeric
        ``remaining'' fields as authorities;
  \item add epoch fencing, disjoint joins, reconciliation, and an independent
        trace verifier; and
  \item publish only claims supported by the endpoint profile actually
        deployed.
\end{enumerate}

Legacy direct paths remain outside complete mediation until credentials and
egress controls make bypass impossible.

\section{Limitations and Validity}
\label{sec:limitations}

\subsection{Model Limitations}

The base model assumes crash faults, authenticated records, one correct local
linearization point per branch, and a finite natural vector. It does not model
Byzantine stores, signature compromise, side channels, probabilistic
settlement, debt, negative balances, interest, speculative credit, or
cross-root netting. Two independently authenticated roots are distinct
authorizations; an issuer that accidentally signs overlapping grants can
overspend a real-world cap without violating either per-root equation.
Issuer-level escrow is required for an aggregate cap. The logical-atom
construction is a proof device; a faulty compression or commitment scheme can
violate the implementation mapping.

Recent work on distributed agent governance under a compromised provider
studies Byzantine replication, monitoring, auditing, and hybrid defenses
\cite{laws2026byzantinegovernance}. Its adversary is outside our base theorem:
we assume correct authoritative stores, gateways, and evidence services and
establish safety only for the declared crash, retry, and partition model. The
two guarantees are complementary rather than ordered by strength.

The end-to-end effect theorem depends on sound normalization, complete
mediation, an enforceable maximum-charge contract, gateway
fencing/idempotency, and complete non-equivocating terminal evidence. These are
substantial assumptions. The ledger theorem remains useful when some fail,
but it must not be presented as conservation of actual external cost.

\subsection{Formal-Method Limitations}

The inductive proofs establish the mathematical transition system. TLA+/TLC
and the independent explorer cover only their declared finite constants,
fault schedules, and modeled transitions. State-space completion increases
confidence in the specification and detects regressions; it is not a proof of
database durability, transport authentication, cryptographic libraries, or
unmodeled adapters.

In particular, the settlement TLA+ and JavaScript encodings assign each
logical unit one location by type. Their exact-sum invariant checks transition
bookkeeping but cannot represent the duplicate-tag raw states excluded by the
ownership-partition proof. The lineage encoding admits origin-tag
multiplicity and detects its compound fan-in mutant, but it is still a bounded
model rather than a mechanized proof of the full raw-relation semantics.

The trace-lifting result defines the criterion that every runtime handler must
satisfy. A concrete event such as a receiver redirect or background retry
outside the gateway must appear in the mapping; otherwise the deployment does
not inherit the theorem. The SQLite profile discharges the exercised
first-acceptance permit and crash-cut obligations. Production deployments must
add a handler-by-handler discharge of
Equation~\eqref{eq:stepwise-refinement}, supported where appropriate by
verified storage procedures, refinement types, or proof-carrying adapters.

\subsection{Experimental Limitations}

The provider-free artifact isolates distributed-control behavior and supports
deterministic internal verification. Its measurements characterize the stated
formal scopes, named crash windows, and one local implementation and host.
Model quality, real-provider accounting, Internet-scale latency, multi-region
database failover, operator response, production failure frequency, and
capacity planning are separate evaluation targets.

The most important negative result is structural and does not depend on
benchmark scale: without endpoint cooperation, a lost response cannot
distinguish non-execution from completed execution. Experiments can
demonstrate the race but cannot remove it.

\subsection{Construct and External Validity}

``Budget'' spans tokens, currency, calls, disclosures, and risk units. These
dimensions differ in metering authority and reversibility. Results transfer
only when the deployment supplies a conservative maximum and an authoritative
settlement relation. Quantization can over-reserve; coarse units trade utility
for simple guarantees.

The 2026 agent papers discussed here are recent and several are preprints or
an Internet-Draft. Their reported systems and terminology may evolve. We use
them to set a conservative novelty boundary, not as proof of deployment
prevalence. Classical escrow, invariant-confluence, lease, and workflow
literature remains the stronger foundation.

\subsection{Internal Validity}

Independent implementations, mutation tests with expected counterexamples,
prefix checks, stable result manifests, and explicit claim/evidence mapping
reduce correlated specification risk across the formal model, runtime,
checker, and generated tables. Artifact results are regenerated from source;
every failed or skipped case is reported.

\section{Ethics and Responsible Deployment}
\label{sec:ethics}

Budget enforcement can reduce financial loss, excessive API use, unauthorized
disclosure, and repeated high-impact effects. It can also create false
confidence. A green ledger does not prove that a planner chose a beneficial
action, that a price was fair, that a recipient was appropriate, or that
complete mediation actually held. Interfaces should display endpoint profile,
quarantined capacity, unresolved effects, and assumption failures alongside a
remaining balance.

Fine-grained receipts and delegation graphs reveal organizational structure,
purposes, recipients, spending patterns, and potentially sensitive content.
Evidence minimization, separation of duties, access logging, encryption,
retention limits, and deletion procedures are necessary. Cryptographic
commitments can reduce disclosure but do not make metadata harmless.

Conservative quarantine can be abused to deny service or trap a competitor's
shared capacity. Issuers should cap the budget one child can reserve, define
escalation and forfeiture authority, record human overrides as new
authorizations, and avoid automatic punitive decisions based solely on an
ambiguous external outcome or unverified assertion. The mechanism allocates
authority; it does not decide whose goals deserve resources.

\section{Conclusion}
\label{sec:conclusion}

Distributed agent delegation turns a budget from a number into a fault-tolerant
authorization object. Parent/child allocation inequalities, affine local
types, attenuated scopes, idempotency keys, or convergent counters each solve
part of the problem. None alone prevents a refunded timeout from colliding
with a late effect, an aliased fan-in from reminting capacity, or a stale
worker from spending a reclaimed generation.

We introduced a protocol in which exclusive, labeled credit moves through a
delegation lineage; a protected effect is reserved before exposure; ambiguity
is quarantined; and one idempotent terminal decision applies authenticated
evidence at most once. A success receipt releases only its authenticated
remainder. Full return from an exposed obligation requires a
receiver-visible fence together with authoritative proof that the exact
binding was never accepted and never occurred. The
six-bucket equation keeps spendable, moving, reserved, uncertain, committed,
and forfeited rights explicit. Under stated assumptions, distinct effect
instances are injectively covered, and settled actual charge plus conservative
outstanding exposure is bounded by the issued root vector. The construction
also preserves ledger conservation, descendant non-amplification,
late-completion safety, and partition confinement.

The central systems result is precise: local availability comes
from rights allocated before a partition, while uncertainty must remain
charged until evidence resolves it. A system that hides this cost by refunding
on timeout does not conserve authority; it merely delays discovery of the
overspend.

\section*{Data, Code, and Materials Availability}

The public preprint package contains the manuscript source and bibliography.
The formal specification, bounded explorer, mutation suite, fault schedules,
SQLite traces, trace checker, result aggregator, and digest manifest are
maintained as a deterministic, provider-free research artifact using only
synthetic identifiers and effects. Appendix~\ref{app:artifact} records its
commands, declared scope, and expected outputs so the reported evidence can be
audited against the frozen internal release. The executable artifact is not
distributed with this preprint and no public archival URL is claimed.

\section*{Author Contributions}

Both authors contributed to problem formulation, system design, formalization,
analysis, artifact review, and manuscript preparation. The machine-readable
manifest separates generated evidence from source and records the verification
commands used for the reported results.

\section*{Competing Interests}

The authors are affiliated with Accentrust, which develops agent-governance
technology. The paper's claims are restricted to the stated formal assumptions
and accompanying artifact; no production deployment claim is made.

\appendix

\section{Notation and Semantic Domains}
\label{app:notation}

Table~\ref{tab:notation} collects symbols used by the abstract model. A bold
quantity is a resource vector; a set-valued footprint denotes logical rights.
All vector relations are componentwise unless an authenticated conversion
contract explicitly says otherwise.

\begin{table*}[t]
\centering
\scriptsize
\setlength{\tabcolsep}{4pt}
\begin{tabularx}{\textwidth}{L{0.18\textwidth}L{0.25\textwidth}Y}
\toprule
\textbf{Symbol} & \textbf{Domain} & \textbf{Meaning} \\
\midrule
\(\Resources\) & finite set & authenticated resource dimensions \\
\(\mathbf B_g\) & \(\NatVec\) & vector issued by root grant \(g\) \\
\(U_g\) & finite disjoint union & logical unit rights underlying
\(\mathbf B_g\) \\
\(X_c,X_o,X_x\) & subsets of \(U_g\) & credit, operation, and transfer
footprints \\
\(\mathsf{vec}(X)\) & \(\NatVec\) & per-dimension cardinality of footprint
\(X\) \\
\(\lambda\) & authorization-label domain & root, principal, purpose, audience,
effect class, expiry, and policy version \\
\(\lambda'\restrict\lambda\) & partial order & child label is no broader than
parent label \\
\(\Graph_T\) & DAG & task dependency graph; edges do not imply ownership \\
\(\Graph_C\) & DAG & credit lineage; edges represent consumed input and fresh
output \\
\(c,x,o\) & records & credit, sealed transfer, and operation \\
\(\epsilon,r^+,\pi_{\neg e}\) & externally authenticated terminal evidence &
success receipt or fence-backed no-effect certificate \\
\(\omega,\zeta,\pi_{\mathrm{local}}\) & terminal-outcome witnesses & closed
tagged sum of external evidence and a durable, authenticated local
never-dispatchable witness; \(\zeta\) excludes success \\
\(\eta\) & monotonic epoch & receiver-visible fencing epoch \\
\(\chi_c\) & logical slot identifier & stable identifier of one reissue chain;
split and join outputs receive fresh slots \\
\(\Current_s(c,\eta)\) & state predicate & authority records \(c\) as the
unique active generation for slot \((g,\chi_c)\) at epoch \(\eta\);
disjoint sibling slots may be current concurrently, and singleton input
ownership is a separate guard \\
\(\kappa\) & receiver-scoped key & idempotency/query key for one normalized
effect and source generation; the separate acceptance digest binds the slot \\
\(\beta_o\) & digest & complete gateway-acceptance binding from
Equation~\eqref{eq:accept-binding}, including \((g,\chi_c)\) \\
\(\mathsf{ReturnTarget}\) & predicate over \(o,c_r,s,s'\) & same-root,
current in the committed successor state, same-source-lineage, same-authority,
label-non-widening destination for unused rights \\
\(\mathsf{ReissueTarget}\) & predicate over
\(c,\eta,\eta^+,c^+,s,s'\) & authenticated same-root, same-authority,
non-widening successor generation \\
\(\mathsf{Avail},\mathsf{SafeA}\) & atom sets & all raw available atoms of a
generation, and the subset with retained reclaim provenance; successor
installation requires equality \\
\(\RolloverAct\) & abstract action & one atomic final terminal partition,
predecessor closure, and normalized same-slot successor installation \\
\(\Phi_o\) & witness & joint current authorization, route, closure, effect,
and credit evidence \\
\(\mathbf u_o\) & \(\NatVec\) & maximum amount reserved for operation \(o\) \\
\(\mathbf d_{r^+}\) & \(\NatVec\) & authenticated actual charge in success
receipt \(r^+\) \\
\(\nu_{\mathsf{sel}}\) & selector schema version & immutable public ordering
used to derive an exact charged footprint from a charge vector \\
\(\mathbf A,\mathbf T,\mathbf R,\mathbf Q,\mathbf C,\mathbf F\) &
\(\NatVec\) & available, transit, reserved, quarantined, committed, and
forfeited totals \\
\(\mathsf{Obs}(s)\) & effect-instance log & protected instances observed by the
gateway in state \(s\) \\
\(\iota\) & effect-instance record & distinct gateway occurrence with
\(iid\), key, acceptance binding, digest, receiver, and acceptance sequence \\
\(\Covered(\iota,o,s)\) & relation & instance \(\iota\) has one compatible
charged operation \(o\) in state \(s\) \\
\(\alpha\) & abstraction function & concrete event to abstract transition or
stutter \\
\bottomrule
\end{tabularx}
\caption{Notation. Logical unit sets make aliasing explicit; the runtime may
use compact balance and lineage certificates.}
\label{tab:notation}
\end{table*}

\section{Complete Abstract Transition Contract}
\label{app:transitions}

Tables~\ref{tab:transitions}--\ref{tab:transitions-recovery} make the
proof-relevant preconditions and postconditions explicit. ``Same atoms'' means
identity-preserving movement, not numeric subtraction followed by unrelated
issuance.

\begin{table*}[t]
\centering
\scriptsize
\setlength{\tabcolsep}{3.4pt}
\begin{tabularx}{0.99\textwidth}{L{0.14\textwidth}L{0.265\textwidth}L{0.265\textwidth}Y}
\toprule
\textbf{Action} & \textbf{Precondition} & \textbf{Postcondition} &
\textbf{Replay/failure behavior} \\
\midrule
\(\Issue\) & authenticated fresh root; vector and label schema valid &
all \(U_g\) in one root \(\mathsf A\) credit with a fresh logical slot; install
matching epoch and active-generation registers &
duplicate root ID rejected or returns identical issuance \\
\(\Split\) & current owner generation/epoch; pairwise-disjoint subsets;
attenuated labels &
parent loses exactly union; install every child under a pairwise-distinct fresh
logical slot; close and clear the parent only if its complete responsibility is
consumed & same request stutters; conflicting
replay rejected \\
\(\Delegate\) seal & current source; subset available; destination/label bound &
same atoms \(\mathsf A\rightarrow\mathsf T\); signed certificate/outbox
durable & send retries preserve transfer ID; timeout stays \(\mathsf T\) \\
\(\Delegate\) decision & valid current certificate or destination-exclusion
proof; decision pending & CAS to exactly one of accepted or cancelled and
logically moves the footprint to the unique destination child or back to the
source; acceptance installs the child under a fresh logical slot with matching
current registers & duplicate returns
prior decision; competing terminal decision rejected \\
\(\Delegate\) materialize & terminal transfer decision and matching footprint &
durable rows materialize the logical owner already selected by the decision
register; abstract stutter & crash recovery repeats the same materialization;
stale source or destination rows are never authoritative \\
\(\Reserve\) & current credit/epoch; available footprint; valid effect and
joint witness & same atoms \(\mathsf A\rightarrow\mathsf R\); operation,
obligation, outbox, audit atomically durable & duplicate identical operation
stutters; changed binding rejected \\
\(\Dispatch\) & durable unexposed reservation; outbox binding intact; current
epoch and witness rechecked & same
atoms \(\mathsf R\rightarrow\mathsf Q\) before possible effect & duplicate
gateway send observes same key; crash remains \(\mathsf Q\) \\
\(\mathsf{GatewayAccept}\) & if key exists, stored binding must equal
\(\beta_o\); if unseen, matching durable dispatchable marker,
\(X_o\) wholly in \(\mathsf Q/\mathsf F\), recomputed \(\beta_o\), and current
epoch/witness for the authenticated \((g,\chi_c)\) slot & existing exact
binding stutters; otherwise one per-key accepted
\(\beta_o\) is durable & exact replay remains readable after terminal bucket
moves; unexposed, stale, or conflicting first acceptance rejects \\
\(\mathsf{EffectOccurs}\) & prior accepted binding; no instance for key in the
one-effect profile & append one fresh \(\iota\) to \(\mathsf{Obs}\) & transport
retry cannot append a second instance \\
\bottomrule
\end{tabularx}
\caption{Abstract action contract for issuance, transfer, reservation,
dispatch, acceptance, and effect occurrence. Preconditions are part of the
mechanism, not optional validation advice.}
\label{tab:transitions}
\end{table*}

\begin{table*}[t]
\centering
\scriptsize
\setlength{\tabcolsep}{3.4pt}
\begin{tabularx}{0.99\textwidth}{L{0.14\textwidth}L{0.265\textwidth}L{0.265\textwidth}Y}
\toprule
\textbf{Action} & \textbf{Precondition} & \textbf{Postcondition} &
\textbf{Replay/failure behavior} \\
\midrule
\(\Settle\) success & full footprint in \(\mathsf Q(o)\); valid receipt for
\(\iota\in\mathsf{Obs}\), matching \(\beta_o\); when the remainder is
nonempty, a decision-bound return credit satisfying
Equation~\eqref{eq:return-target};
\(\zero\preceq\mathbf d\preceq\mathbf u\) & version-bound
\(\mathsf{Take}_{\nu_{\mathsf{sel}}}(X_o,\mathbf d)\)
\(\mathsf Q\rightarrow\mathsf C\); complement
\(\mathsf Q\rightarrow\mathsf A\); bind \(H(D_o)\) and
\(H(X_o\setminus D_o)\) & identical receipt stutters; conflicting receipt or
partition rejects \\
\(\ReturnAct\) & type-matched \(\zeta\): local no-dispatch precondition for
\(\pi_{\mathrm{local}}\), or gateway fence plus authoritative
\(\pi_{\neg e}\) proof that the exact binding was never accepted and cannot
occur; return target satisfies Equation~\eqref{eq:return-target} & all
operation atoms move from the branch-specific
\(\mathsf R/\mathsf Q\rightarrow\mathsf A\); write terminal decision and
audit; the local branch also persists \(\pi_{\mathrm{local}}\) and sets the
dispatch gate to absorbing \mcode{local-abort} & timeout or local lease expiry
alone has no transition \\
\(\Retire\) & unresolved quarantined operation and authorized administrative
decision & all operation atoms \(\mathsf Q\rightarrow\mathsf F\) &
availability never increases; \(\mathsf F\) is absorbing \\
\(\Fence\) & source is current; authenticated target
\(\eta^+>\mathsf{AuthoritativeEpoch}(g,\chi_c)=\eta_c\); gateway installation
protocol satisfied & install the gateway fence, mark the old credit closing,
clear its active generation, and advance the slot authority epoch without
moving balance or creating a successor; old
generation cannot initiate a new effect and recorded old outcomes remain
queryable & old messages reject or return prior result \\
\(\mathsf{Reissue}\) & old generation closed at installed \(\eta^+\); exact
\(\mathsf{Reusable}\) footprint uniquely available;
\(\mathsf{SafeA}=\mathsf{Avail}\); write-once key selects
\((\eta^+,c^+)\) for normalized predecessor
\((g,\chi_c,cid,\eta_c)\); target satisfies
Equation~\eqref{eq:reissue-target} & one
transaction moves only that footprint and installs \(c^+\) as current for the
same logical slot &
committed, forfeited, accepted, or unresolved atoms never move; exact replay
stutters and changed target rejects \\
\(\RolloverAct\) & fenced predecessor is closing; \(o\) is its final
nonterminal operation with a type-matched terminal-outcome witness; source and settlement
authority agree; \(\mathsf{SafeA}=\mathsf{Avail}\); both decision keys are
unassigned; normalized reissue target
satisfies Equations~\eqref{eq:rollover-ready}--\eqref{eq:rollover-action} &
one transaction fixes the terminal partition, moves only its return footprint
plus disjoint \(\mathsf{SafeA}\) to \(c^+\), closes \(c\), and installs
\(c^+\), writes both keys, sets the witness-specific operation/outbox outcome,
and appends its audit node; the local branch additionally sets the absorbing
\mcode{local-abort} dispatch gate &
exact replay stutters; changed evidence, partition, epoch, or
successor rejects \\
\(\JoinAct\) & current, operation-free, pairwise-disjoint same-root inputs
controlled by one transition authority & inputs consumed; same-root output
owns their exact union under lower-bound label and a fresh logical slot;
fully consumed inputs close and clear, and the output's current registers are
installed &
duplicate stutters; overlap or mixed root rejects \\
\bottomrule
\end{tabularx}
\caption{Abstract action contract for settlement, return, retirement, fencing,
reissue, fused rollover, and join (continued).}
\label{tab:transitions-recovery}
\end{table*}

\section{Concrete-to-Abstract Refinement Obligations}
\label{app:refinement}

A concrete implementation trace is a sequence of committed database
transactions, gateway observations, and authenticated terminal evidence. Volatile
steps that neither expose an effect nor change durable state map to stutter.
Tables~\ref{tab:refinement} and \ref{tab:refinement-recovery} list the minimum
simulation obligations.

\begin{table*}[t]
\centering
\footnotesize
\setlength{\tabcolsep}{4pt}
\begin{tabularx}{\textwidth}{L{0.20\textwidth}L{0.28\textwidth}YY}
\toprule
\textbf{Concrete event} & \textbf{Abstract image} &
\textbf{Required evidence} & \textbf{Forbidden mapping} \\
\midrule
Root row commit & \(\Issue\) & issuer signature, fresh generation, exact vector
& editing an existing issued vector \\
Source debit + transfer row commit & \(\Delegate\) seal & one transaction,
same compressed ranges & message send before source debit \\
Transfer-decision CAS & \(\mathsf{AcceptTransfer}\) or
\(\mathsf{CancelTransfer}\), including the one logical ownership switch &
valid certificate/exclusion proof and monotone decision register & accepted
and cancelled both becoming terminal \\
Destination/source finalization & \(\mathsf{stutter}\) (materialization only) &
same terminal decision and exact compressed ranges; logical ownership is read
from the decision register & treating a stale row as authoritative or moving
the abstract footprint a second time \\
Operation/outbox commit & \(\Reserve\) & atomic debit and immutable
slot/effect binding & reserving after tool execution \\
Dispatchable marker commit & \(\Dispatch\) & durable
\(\mathsf R\rightarrow\mathsf Q\) before gateway exposure & network send while
still only \(\mathsf R\) \\
Gateway acceptance row & \(\mathsf{GatewayAccept}\) & existing exact binding,
or, for first acceptance only, matching durable dispatchable marker and
\(\mathsf Q/\mathsf F\) ownership, authenticated logical slot, its current
epoch/fence, recomputed \(\beta_o\), witness, and maximum vector & reapplying
first-acceptance bucket guards to an
exact terminal replay, or accepting an unexposed, stale, or changed binding \\
Gateway effect log & \(\mathsf{EffectOccurs}\) & prior accepted binding,
fresh \(iid\), complete maximum-charge contract & an unbound direct receiver
call or a second instance for one key \\
\bottomrule
\end{tabularx}
\caption{Refinement obligations for issuance, transfer, dispatch, and gateway
effects.}
\label{tab:refinement}
\end{table*}

\begin{table*}[t]
\centering
\footnotesize
\setlength{\tabcolsep}{4pt}
\begin{tabularx}{\textwidth}{L{0.20\textwidth}L{0.28\textwidth}YY}
\toprule
\textbf{Concrete event} & \textbf{Abstract image} &
\textbf{Required evidence} & \textbf{Forbidden mapping} \\
\midrule
Terminal-outcome transaction & \(\Settle\), \(\ReturnAct\), or, when fused
with successor installation, \(\RolloverAct\) & a tagged external signature
or trusted durable local no-dispatch witness, \(\beta_o\) including selector
version, one terminal CAS, exact
committed/remainder footprint commitments, and
Equation~\eqref{eq:return-target} for every returned atom & treating a success
receipt as no-effect evidence, choosing a noncanonical partition, applying a
textual worker assertion, or widening/redirecting a remainder \\
Timeout / worker death & stutter + reconciliation request & no ownership
change & \(\mathsf Q\rightarrow\mathsf A\) \\
Epoch installation & \(\Fence\) & monotone successor visible at every
protected gateway & local-only epoch bump followed by reuse \\
Ordinary successor-generation transaction & \(\mathsf{Reissue}\) &
closed-generation
proof at the installed \(\eta^+\), exact uniquely available reusable footprint,
write-once unique reissue decision, and authenticated target satisfying
Equation~\eqref{eq:reissue-target}; decision, current metadata, and move are
atomic & copying old
\(\mathsf C/\mathsf F/\mathsf T/\mathsf Q\) atoms or widening/changing the
successor lineage, root, or authority \\
Fused final-outcome/successor transaction & \(\RolloverAct\) &
Equations~\eqref{eq:rollover-ready}--\eqref{eq:rollover-action}; canonical
terminal partition and status; complete \(\mathsf{SafeA}\); atomic terminal
and normalized reissue decisions &
mapping the fused commit to only \(\Settle\), \(\ReturnAct\), or
\(\mathsf{Reissue}\), or copying an input atom \\
Audit-only append & stutter & digest of an already authoritative decision &
moving rights solely because a log entry exists \\
\bottomrule
\end{tabularx}
\caption{Refinement obligations for terminal resolution, fencing, and
successor installation. A trace that cannot be mapped without violating
Tables~\ref{tab:refinement}--\ref{tab:refinement-recovery} is outside the
implementation theorem.}
\label{tab:refinement-recovery}
\end{table*}

\section{Adversarial Trace Catalogue}
\label{app:traces}

\subsection{Timeout Refund and Late Completion}

Let \(U_g=\{u\}\). The safe prefix is:
\[
\mathsf A(c)
\xrightarrow{\Reserve(o)}\mathsf R(o)
\xrightarrow{\Dispatch(o)}\mathsf Q(o).
\]
The controller sees no response. In execution \(H_0\), no gateway effect
occurred; in \(H_1\), the effect occurred and its response is delayed. The
controller's local view is identical. A timeout refund would place \(u\) in
\(\mathsf A\) in both histories. Spending it under a fresh operation in
\(H_1\) creates two covered effect instances for one issued atom. The safe
machine instead stutters until a query distinguishes the histories, or moves
the atom to \(\mathsf F\).

\subsection{Duplicate Transfer Acceptance}

The source seals \(X\) under \(xid\) and retries after losing the
acknowledgment. If the destination interprets the certificate as a bearer
balance, it can create \(c_1\) and \(c_2\), each with numeric value
\(\mathsf{vec}(X)\). The required uniqueness decision over
\((g,xid)\) returns the previously created \(c_1\); it does not re-import the
ranges. Source cancellation is equally dangerous unless the destination has
durably fenced \(xid\).

\subsection{Aliased Diamond Join}

A parent credit is referenced by task branches \(a\) and \(b\), but only one
actual transfer occurs. Both task results point to the same descendant credit
\(c\). A naive join computes \(\mathbf a_c+\mathbf a_c\). The set model instead
tests \(X_c\cap X_c=X_c\ne\varnothing\) and rejects. Task multiplicity does not
create credit multiplicity.

\subsection{Stale Epoch After Reclaim}

Worker \(w\) holds generation \(\eta\), partitions, and prepares operation
\(o\). The controller first installs the gateway fence and advances the
authority epoch to some \(\eta^+>\eta\). That step neither rewrites the old credit's
epoch, moves its balance, nor creates a successor. Reallocation occurs later
and only after \(\mathsf{Closed}\) and exact \(\mathsf{Reusable}\) hold: each
exposed obligation has complete terminal settlement or authoritative,
fence-backed proof that its exact binding was never accepted and never
occurred. When \(w\) reconnects, the gateway rejects a new key under \(\eta\).
If \(o\)'s key was already recorded before the
fence, the gateway returns that outcome for settlement; it does not execute a
second effect. An epoch checked only in the controller database cannot provide
this property.

\subsection{Partial Charge and Duplicate Receipt}

Operation \(o\) reserves vector \((1000\text{ cents},3\text{ calls})\). A
receipt authenticates \((730,1)\). One transaction commits a matching subset
and returns \((270,2)\). Replaying the receipt returns the same terminal
decision. A second receipt claiming \((820,1)\) is not a second partial update;
it conflicts with the immutable terminal decision and is quarantined for
investigation.

\section{Claim--Evidence Matrix}
\label{app:claims}

\begin{table*}[t]
\centering
\scriptsize
\setlength{\tabcolsep}{3.2pt}
\begin{tabularx}{\textwidth}{L{0.23\textwidth}L{0.25\textwidth}L{0.23\textwidth}Y}
\toprule
\textbf{Claim} & \textbf{Analytical support} &
\textbf{Executable support} & \textbf{Explicit exclusion} \\
\midrule
Ownership partition & transition induction over raw ownership relations &
lineage tagged-occurrence model, legal-join witness, and compound invalid
fan-in mutant; settlement model separately checks the single-location
obligations &
Byzantine/corrupt authoritative store \\
Exact six-bucket conservation & partition lemma + vector corollary &
finite transition-accounting checks and runtime prefix checks & real-world
value not represented by resource schema \\
Descendant non-amplification & cut/lineage induction & split, transfer, join,
and alias schedules & task quality and fair allocation \\
At-most-once settlement & immutable binding + terminal CAS & duplicate and
conflicting receipt cases & exactly-once network delivery \\
Effect-instance coverage and external exposure bound & injective gateway
binding + complete-charge theorem & effect-instance log joined to operation
and maximum-vector trace & unmediated paths or dishonest charge authority \\
Late-completion safety & \(\mathsf Q\) retention + fence argument &
effect-before-response-loss and delayed receipt cases & automatic eventual
refund \\
Partition confinement & exclusive preallocation theorem & exclusive-transfer
and coordinator-outage schedules in bounded models;
the theorem is the primary support & full availability after local escrow
exhaustion \\
Concrete trace lifting & conditional criterion under
Appendix~\ref{app:refinement} & independent checks of selected binding and
prefix obligations plus invalid-trace cases & a universal refinement proof
for unexamined handlers or deployments \\
Prototype timing characterization & no theorem & local durable-store
diagnostic timings &
Internet-scale or production SLO \\
\bottomrule
\end{tabularx}
\caption{Claims are linked to distinct analytical and executable evidence.
The exclusions are part of the claim, not deferred implementation details.}
\label{tab:claim-evidence}
\end{table*}

\section{Artifact and Reproduction Protocol}
\label{app:artifact}

\subsection{Directory Contract}

\begin{lstlisting}
formal/
  BudgetConservation.tla
  DelegationLineage.tla
  budget_model.mjs
  bounded_model.mjs
  configs/
    safe.cfg
    timeout_refund_mutant.cfg
    lineage_safe.cfg
    lineage_current_join_witness.cfg
    lineage_historical_certificate_validation_mutant.cfg
bench/
  run_tlc_model_checks.mjs
  run_lineage_tlc.mjs
  fault_injection_eval.mjs
  sqlite_process_service.mjs
  sqlite_multiprocess_eval.mjs
  sqlite_trace_checker.mjs
  run_all.mjs
  aggregate_results.mjs
  verify_artifact.mjs
  results_tla_model_check.json
  results_lineage_tla_model_check.json
  results_bounded_model.json
  results_fault_injection.json
  results_aggregate.json
  artifact-manifest.json
  sqlite_runtime_results/<run>/
    summary.json
    checker-summary.json
    traces/
    db/
generated-results.tex
\end{lstlisting}

The two TLA+ files are complementary: one covers settlement, acceptance,
effects, fencing, and generations; the other covers lineage and fan-in. The
JavaScript model is a second executable encoding of the settlement
abstraction, not an independent conceptual specification. Concrete SQLite
results retain one branch-ledger and one gateway database per case. The
compact branch persists a signed dispatch permit in its \(W_3\) transaction;
the gateway verifies and records the permit before accepting an unseen key.
Its independent checker separately validates the signature, temporal order,
and cross-store binding. This exercises the evaluated first-acceptance guard
without substituting finite traces for universal adapter or end-to-end
refinement.
\mcode{results\_aggregate.json} is the sole numeric source for
\mcode{generated-results.tex}; generated output is never treated as a source
model.

\subsection{Canonical Commands}

\begin{lstlisting}
# Verify retained TLC/SQLite evidence and rerun JS layers:
node bench/run_all.mjs --skip-tlc

# Rerun both pinned TLC models:
TLA2TOOLS_JAR=/absolute/path/tla2tools.jar \
JAVA_BIN=/absolute/path/java \
node bench/run_all.mjs

# Also create a fresh, non-overwriting SQLite run:
TLA2TOOLS_JAR=/absolute/path/tla2tools.jar \
JAVA_BIN=/absolute/path/java \
node bench/run_all.mjs --rerun-sqlite

# Regenerate/verify from the retained concrete run:
node bench/aggregate_results.mjs \
  --sqlite-dir \
  bench/sqlite_runtime_results/dispatch-permit-v1-20260925
node bench/verify_artifact.mjs

tectonic -X compile main.tex --outdir build \
  --keep-logs --keep-intermediates --only-cached
\end{lstlisting}

The \mcode{--rerun-sqlite} mode always allocates a new timestamped directory;
it neither removes nor overwrites a retained run. The verification command
fails if a source or result digest is stale, a required safe or mutant case is
absent, a safe-profile invariant count is nonzero, the aggregate and generated
LaTeX differ, a JSONL record count changes, or any of the 22 retained SQLite
databases fails \mcode{PRAGMA integrity\_check}. It also re-executes the
independent SQLite snapshot checker, the dependency-free finite explorer, and
the schedule suite, compares their outputs byte-for-byte, and directly checks
persisted terminal return targets, successor lifecycle rows, and transaction
guards.

\subsection{Environment Disclosure}

The artifact manifest records operating system, architecture, CPU, memory,
Node.js, Java and TLC versions, the pinned TLC JAR digest, and SHA-256 plus
byte length for every retained source and result. The concrete SQLite evidence
additionally records SQLite version, process topology, command, UTC
start/completion times, and source hashes. TLC uses one worker, a fixed
fingerprint polynomial, and a fixed seed where applicable; the named
JavaScript and SQLite schedules are deterministic. The provider-free artifact
requires no network access after the pinned TLC JAR is available.

\subsection{Result Reporting}

For each model:

\begin{itemize}
  \item report constants, distinct and generated states, search depth or
        graph depth, queue/completion status, and invariant result;
  \item retain the shortest available counterexample for every mutant;
  \item identify whether symmetry, state constraints, or liveness checking was
        used; and
  \item never translate ``no error in scope'' into ``proved for all sizes.''
\end{itemize}

For each concrete case, the artifact retains the full JSONL event trace,
durable branch/gateway snapshots, effect/terminal-evidence binding, expected outcome,
actual outcome, and independent pass/fail decision. It also retains per-case
RPC and restart summaries from the same correctness runs so every reported
descriptive timing can be regenerated. These timing samples are health checks,
not a throughput experiment or production SLO.

\section{Deployment Assurance Checklist}
\label{app:deployment}

\begin{enumerate}
  \item Define every resource dimension, unit, rounding rule, maximum-charge
        rule, and authoritative settlement source.
  \item Enumerate every route to each protected effect; remove direct
        credentials or mark the deployment outside complete mediation.
  \item Verify that reserve, obligation, outbox, and audit records share one
        durable local transaction.
  \item Verify that possible gateway exposure occurs only after the durable
        \(\mathsf R\rightarrow\mathsf Q\) transition.
  \item At first acceptance, require a signed single-use dispatch permit or an
        authoritative branch query proving the matching dispatchable marker
        and \(\mathsf Q/\mathsf F\) ownership; an offline audit is insufficient.
  \item Scope operation keys to receiver and normalized effect; make exact
        replay a stable decision and conflicting replay an error.
  \item Recompute \(\beta_o\) and test terminal-evidence signatures and every immutable
        binding, including root, logical slot, lineage, footprint, epoch,
        effect, receiver, maximum, and actual charge.
  \item Require Equation~\eqref{eq:return-target} for every partial remainder
        or no-effect return; use a separate sealed transfer for cross-store or
        ancestor return.
  \item Before reusing a lost generation, install fences at every protected
        gateway and obtain terminal settlement or authoritative proof that each
        returned binding was never accepted and never occurred; a fence or
        lease expiry alone is not effect-absence evidence.
  \item Reject numeric fan-in unless ownership inputs are pairwise disjoint and
        atomically consumed.
  \item Monitor \(\mathbf T,\mathbf Q,\mathbf F\) separately from committed
        spend and available capacity.
  \item Bound per-child reservations and define escalation, reconciliation,
        and forfeiture authorities.
  \item Run the conformance suite for each MCP, A2A, HTTP, browser, shell, or
        framework adapter.
  \item Publish the actual endpoint capability profile and preserve the
        assumptions beside every assurance claim.
\end{enumerate}

\section{Proof and Evidence Dependency Map}
\label{app:dependency-map}

The logical dependency chain is:
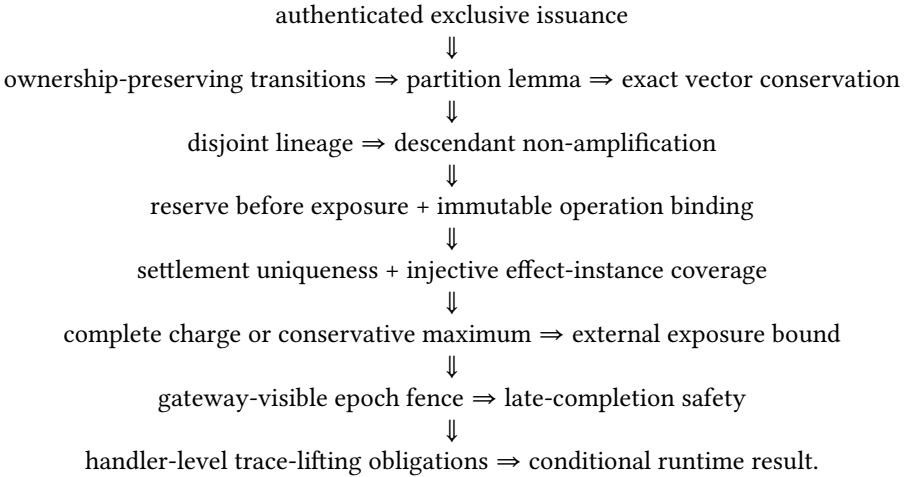
\[
\begin{array}{c}
\text{authenticated exclusive issuance}\\
\Downarrow\\
\text{ownership-preserving transitions}
\Rightarrow \text{partition lemma}
\Rightarrow \text{exact vector conservation}\\
\Downarrow\\
\text{disjoint lineage}
\Rightarrow \text{descendant non-amplification}\\
\Downarrow\\
\text{reserve before exposure + immutable operation binding}\\
\Downarrow\\
\text{settlement uniqueness + injective effect-instance coverage}\\
\Downarrow\\
\text{complete charge or conservative maximum}
\Rightarrow \text{external exposure bound}\\
\Downarrow\\
\text{gateway-visible epoch fence}
\Rightarrow \text{late-completion safety}\\
\Downarrow\\
\text{handler-level trace-lifting obligations}
\Rightarrow \text{conditional runtime result}.
\end{array}
\]

Formal exploration targets the first five arrows. Concrete trace checking
targets the final refinement arrow. Fault injection demonstrates that the
runtime reaches the declared recovery cuts. None substitutes for verifying
the environmental assumptions named above.

\bibliographystyle{ACM-Reference-Format}
\bibliography{references}

\end{document}